%% file: main.tex
\documentclass{article}

\usepackage{arxiv}

\usepackage[utf8]{inputenc} 
\usepackage[T1]{fontenc}    
\usepackage{hyperref}       
\usepackage{url}            
\usepackage{booktabs}       
\usepackage{amsfonts}       
\usepackage{nicefrac}       
\usepackage{microtype}      
\usepackage{xcolor}         
\usepackage{natbib}
\usepackage{graphicx} 
\usepackage{enumitem}
\usepackage{xspace}
\usepackage{wrapfig}
\usepackage{subcaption}
\usepackage{placeins}

\usepackage{amsthm}
\usepackage{mathtools}
\usepackage{todonotes}
\usepackage{amssymb,fge}
\usepackage{thm-restate}
\usepackage{wrapfig}
\usepackage{bbm}
\usepackage{mathrsfs}
\usepackage{soul}
\usepackage{array}
\usepackage{multirow}
\usepackage{algorithm}
\usepackage{cleveref}

\usepackage{algorithm}
\usepackage{algpseudocode}

\input{macros.tex}

\title{Provable Benefit of SignGD: A Minimal Model Under Heavy-Tailed Class Imbalance}

\author{%
  \begin{tabular}{@{}c@{\hspace{5em}}c@{}}
    Robin Yadav$^{1,3}$ &
    Shuo Xie$^{1}$\\
    {\normalfont\texttt{robiny12@student.ubc.ca}} &
    {\normalfont\texttt{shuox@ttic.edu}}\\[0.75em]
    Tianhao Wang$^{1,2}$ &
    Zhiyuan Li$^{1}$\\
    {\normalfont\texttt{tianhaowang@ucsd.edu}} &
    {\normalfont\texttt{zhiyuanli@ttic.edu}}
  \end{tabular}\\[3em]
  {\normalfont $^{1}$Toyota Technological Institute at Chicago; $^{2}$University of California, San Diego;}\\
  {\normalfont $^{3}$University of British Columbia}
}

\newcommand{\smoothness}[2]{L_{#2}(#1)}

\begin{document}

\maketitle

\input{sections/abstract}
\input{sections/introduction}

\input{sections/related-work}

\input{sections/minimal-language-model}
\input{sections/deterministic-softmax-unigram}
\input{sections/stochastic-softmax-unigram}

\input{sections/adam-extension-main}
\input{sections/conclusion}

\bibliographystyle{plainnat}
\bibliography{bibliography}

\newpage
\appendix

\tableofcontents

\input{sections/softmax-unigram-constants}

\input{sections/proofs-without-weight-decay}
\input{sections/deterministic-rate-proofs}
\input{sections/gd-lower-bound}

\input{sections/generic-stochastic-nsd-wd}
\input{sections/stochastic-softmax-proofs}
\input{sections/stochastic-momentum-softmax}

\input{sections/adam-extension}
\input{sections/experimental-details}

\end{document}

%% file: macros.tex
\newcommand{\norm}[1]{{\left\| #1 \right\|}}

\newcommand{\ip}[2]{\left\langle#1,#2\right\rangle} 

\newcommand{\floor}[1]{\left \lfloor#1\right\rfloor} 

\newcommand{\abs}[1]{\left\lvert #1\right\rvert} 
\newcommand{\softmax}{\operatorname{softmax}}

\newcommand{\cH}{\mathcal{H}}

\newcommand{\cL}{\mathcal{L}}

\newcommand{\cN}{\mathcal{N}}
\newcommand{\cO}{\mathcal{O}}
\newcommand{\cP}{\mathcal{P}}

\newcommand{\E}{\mathbb{E}}

\newcommand{\R}{\mathbb{R}}  

 \newcommand{\mA}{\mathbf{A}}

\newcommand{\del}[1]{}

\newcommand{\pr}[1][]{
  \ifthenelse { \equal{#1}{} }
  { \ensuremath{\mathrm{P}} }
  { \ensuremath{\mathrm{P}\left(#1\right)} }
}

\DeclareMathOperator{\Var}{Var}         
\DeclareMathOperator{\Unif}{Unif} 

\DeclareMathOperator{\tr}{Tr}           
\DeclareMathOperator{\diag}{diag}       
\DeclareMathOperator{\KL}{KL}

\DeclarePairedDelimiterX{\infdivx}[2]{(}{)}{%
  #1\;\delimsize\|\;#2%
}

\numberwithin{assumption}{section}
\newtheorem{theorem}{Theorem}[section]
\newtheorem{lemma}[theorem]{Lemma}

\newtheorem{corollary}[theorem]{Corollary}

 \numberwithin{exercise}{section}
\numberwithin{example}{section}
\newtheorem{definition}{Definition}\numberwithin{definition}{section}
\numberwithin{remark}{section}

\DeclareMathOperator*{\argmax}{arg\,max}
\DeclareMathOperator*{\argmin}{arg\,min}

\DeclareMathOperator*{\sign}{sign}

%% file: sections/abstract.tex
\begin{abstract}
\looseness=-1
Adaptive and non-Euclidean optimizers often outperform Euclidean methods such as stochastic gradient descent (SGD) in language modeling by a large margin.
Existing theory usually explains this gap by assuming favorable smoothness geometry or noise structure tailored to the specific optimizer.
We instead ask whether such geometry can be induced from a concrete learning setting. Starting from an optimizer gap that persists across realistic language-modeling experiments, we progressively remove sequence dependence, architectural complexity, and stochasticity.
We find that the gap exists in a minimal setting: the softmax unigram model with heavy-tailed data.
This model exposes a simple deterministic mechanism under heavy-tailed class imbalance. We prove that GD learns rare tokens slowly because the corresponding logits receive only tiny updates, while SignGD removes this magnitude dependence and moves rare and common coordinates on a more comparable scale.
We make this precise with upper and lower bounds for the convergence rate of GD and upper bounds for the convergence of SignGD. 
Our stochastic bounds contain additional noise-dependent terms that can obscure this advantage in the convergence guarantees and can be reduced by increasing the batch size.
\end{abstract}

%% file: sections/introduction.tex
\section{Introduction}\label{sec:intro}

Adaptive coordinate-wise methods are widely used for modern deep learning problems \citep{bernstein_old_2024}. In particular, the Adam optimizer \citep{kingma_adam_2017} and its variants \citep{loshchilov_decoupled_2019} are prevalent in LLM pretraining, where they significantly surpass the performance of conventional SGD methods \citep{xie_adam_2025, kunstner_heavy-tailed_2024}. Despite this remarkable empirical success, we still lack a complete theoretical understanding of why adaptive coordinate-wise methods converge faster than SGD for language-modeling tasks.


A growing theoretical literature explains the gap between adaptive or non-Euclidean methods and Euclidean SGD by postulating optimizer-friendly geometry. These works derive convergence results under assumptions that include coordinate-wise or anisotropic smoothness \citep{liu_adagrad_2024, bernstein_old_2024, jiang_provable_2025}, adaptive smoothness \citep{xie_adam_2025}, and favorable stochastic-gradient geometry \citep{yu_sign-based_2026, bernstein_signsgd_2018}. These results clarify when particular algorithms benefit from specific geometric assumptions, but they leave open a basic question: how do such optimizer-friendly geometries arise in concrete learning problems? In particular, it remains unclear how the structure of LLM pretraining objective gives rise to geometry that favors non-Euclidean methods.

\begin{figure}[!t]
    \centering

    \begin{minipage}[t]{0.24\textwidth}
        \centering
        {\scriptsize\textbf{(a)}}
        \includegraphics[width=\linewidth]{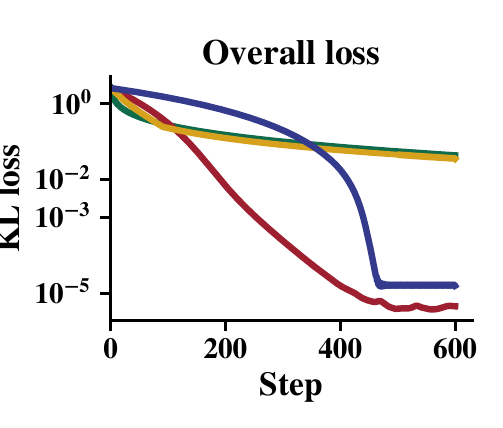}
    \end{minipage}%
    \hfill
    \begin{minipage}[t]{0.24\textwidth}
        \centering
        {\scriptsize\textbf{(b)}}
        \includegraphics[width=\linewidth]{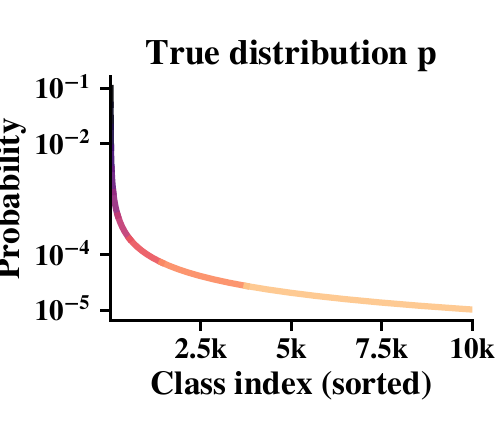}
    \end{minipage}%
    \hfill
    \begin{minipage}[t]{0.24\textwidth}
        \centering
        {\scriptsize\textbf{(c)}}
        \includegraphics[width=\linewidth]{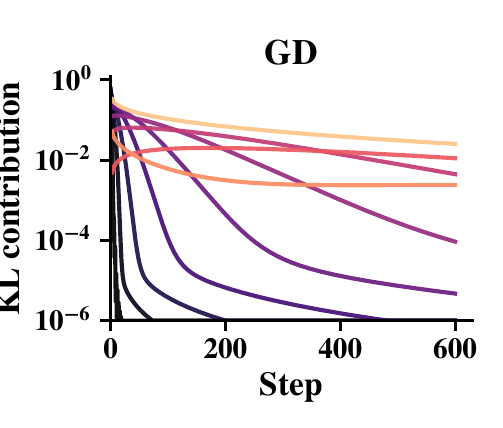}
    \end{minipage}%
    \hfill
    \begin{minipage}[t]{0.24\textwidth}
        \centering
        {\scriptsize\textbf{(d)}}
        \includegraphics[width=\linewidth]{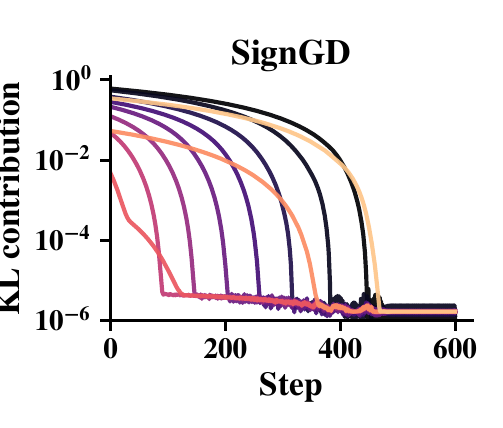}
    \end{minipage}

    \vspace{-2pt}
    \noindent
    \begin{minipage}[t]{0.50\textwidth}
        \centering
        \hspace*{0.10\textwidth}%
        \includegraphics[width=0.9\linewidth]{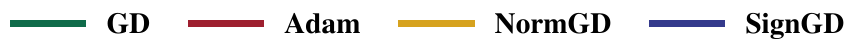}
    \end{minipage}%
    \hfill
    \begin{minipage}[t]{0.42\textwidth}
        \centering
        \hspace*{-0.04\textwidth}%
        \includegraphics[width=0.72\linewidth]{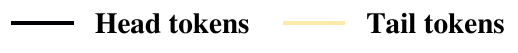}
    \end{minipage}

    \vspace{-4pt}
    \caption{
        \textbf{GD struggles to learn the tail token probabilities on a simple deterministic softmax unigram model with heavy-tailed class imbalance}. The softmax unigram model objective is defined as $f_{\mathrm{uni}}(\theta)=\KL(p\|\softmax(\theta))$ for logits $\theta$.
        \textbf{(a)} Adam and SignGD outperform GD and normalized GD by a significant margin.
        \textbf{(b)} True token distribution for the softmax unigram model; black indicates head tokens, and yellow indicates tail tokens.
        \textbf{(c)} and \textbf{(d)} KL contributions grouped by token rank for GD and SignGD, respectively, where each group has mass 0.1. The contribution of token $i$ is $p_i\log(p_i/q_i)- p_i+q_i$ for $q=\softmax(\theta)$. These non-negative coordinate-wise terms, which sum to the total KL, show that GD fails to reduce loss on tail tokens, while SignGD makes more uniform progress across ranks. Although tail contributions are small, the large number of tail tokens renders them non-negligible.
    }
    \label{main:figure_softmax}
\end{figure}

Our investigation is motivated by the recent finding that, under heavy-tailed token frequencies, adaptive or sign-based methods can outperform Euclidean methods even in simplified transformer models \citep{kunstner_heavy-tailed_2024, kunstner_noise_2023}. 
This gap persists even in the full-batch setting, suggesting that it is not merely an artifact of stochastic gradient noise. At the same time, these simplified settings still retain several architectural and modeling ingredients, making it difficult to isolate the source of the gap between different optimizers. This leads us to ask: 
\begin{quote}
    \emph{Can we develop a minimal language model to isolate and rigorously explain the advantage of sign-based methods under heavy-tailed class imbalance?}
\end{quote}

\looseness=-1
Our answer to this question is the \textit{softmax unigram model}. Inspired by the simple transformer setting in \citet{kunstner_heavy-tailed_2024}, we design a minimal language-modeling problem in which we can prove that non-Euclidean steepest descent methods such as sign gradient descent (SignGD), normalized steepest descent with respect to the $\ell_\infty$ norm, converge faster than GD under heavy-tailed class imbalance. Despite its simplicity, the parameters of the softmax unigram model are coupled through the softmax activation, unlike the simpler separable quadratic models that are often used in prior analyses \citep{kunstner_heavy-tailed_2024,kunstner_scaling_2025}.

\looseness=-1
In the deterministic setting, this model isolates and reveals a simple mechanism behind the optimizer gap, as shown in \Cref{main:figure_softmax}. Under heavy-tailed token frequencies, GD learns rare-token probabilities slowly because progress on tail coordinates is scaled by their small frequencies. In contrast, SignGD is less sensitive to these coordinate-wise magnitude imbalance, making progress on head and tail coordinates at a comparable rate. We formalize this contrast by deriving a lower bound for GD from its tail-learning bottleneck and an upper bound for SignGD strictly better in dimension dependence; together, these bounds establish a strict iteration-complexity separation on the softmax unigram model.

\looseness=-1
In the stochastic setting, our upper bound consists of a deterministic term plus an additional noise-dependent term. Increasing the batch size reduces the bound on the noise term. Specifically, for sufficiently large batch sizes, we obtain an iteration-complexity separation between SignSGD with weight decay and full-batch GD (even though GD has access to the exact gradient).




Together, these results show how a concrete language-modeling objective with heavy-tailed class imbalance gives rise to geometry that favors sign-based optimization. We summarize our contributions as follows:

\begin{enumerate}[leftmargin=*]
    \item In \Cref{sec:minimal_model}, we show that the advantage of sign-based methods over Euclidean methods persists on data with heavy-tailed class imbalance as we progressively reduce a transformer to the softmax unigram model. The resulting model provides a minimal setting in which the gap between SignGD and GD can be rigorously analyzed. Meanwhile, in \Cref{sec:adam}, we show why existing analyses cannot explain Adam's empirical advantage over GD in the same setting; understanding this advantage theoretically remains open.
    \item In \Cref{sec:deterministic}, we establish the deterministic separation on the softmax unigram model. Under the $\ell_2$ geometry, we prove an $\cO(d/T)$ convergence rate for GD (\Cref{main:gd_upper_bound}) and identify its tail-learning bottleneck, which gives an iteration lower bound of $\Omega_\epsilon(\eta^{-1}d^{1-A\epsilon}\log d)$ for any fixed sufficiently small $\epsilon$, with the absolute constant $A$ defined in \Cref{main:gd_lower_bound}; we also show that normalized GD achieves a $\widetilde{\cO}(d/T)$ convergence rate (\Cref{main:deterministic_normgd}). In contrast, under the $\ell_\infty$ geometry, we show that SignGD achieves a $\widetilde{\cO}((\log d)^2/T)$ convergence rate (\Cref{main:deterministic_sign}), yielding a strict separation from GD with a stepsize independent of $d$ and better dimensional dependence than normalized GD. We further show that this $\ell_\infty$-versus-$\ell_2$ dimensional gap persists when weight decay is used (\Cref{main:deterministic_sign_wd,main:deterministic_normgd_wd}).
    \item In \Cref{sec:stochastic}, we derive noise-dependent upper bounds for SignSGD and normalized SGD with weight decay (\Cref{main:stochastic_sign_fixed,main:stochastic_norm_fixed}). We derive fixed-budget batch-size choices and rates for SignSGD, normalized SGD, and SGD (\Cref{main:stochastic_sign_sample,main:stochastic_norm_sample,main:sgd_bound}). With a sufficiently large batch, SignSGD with weight decay retains an iteration-complexity separation from full-batch GD (\Cref{main:large_batch_signsgd_complexity}).
\end{enumerate}

%% file: sections/related-work.tex
\section{Related Work}\label{sec:related_work}
\textbf{Adam, SGD, and heavy-tailed language data.}
Several recent works seek explanations for the gap between Adam and SGD on language models. These include heavy-tailed gradient noise in transformer models \citep{zhang_why_2020}, directional sharpness \citep{pan_toward_2023}, and Hessian heterogeneity \citep{zhang_why_2024}. Recently, \citet{kunstner_heavy-tailed_2024} identified heavy-tailed class imbalance as a key property of language models that causes SGD to struggle compared to Adam. Tokens in language data follow Zipf's law; the $k$-th most frequent token has frequency proportional to $1/k$. Empirically, SGD makes slow progress on low-frequency tokens, which dominate the loss, leading to poor overall convergence while Adam does not suffer from this issue. Increasing the batch size or training in the full-batch setting only increases the gap between these optimizers \citep{kunstner_noise_2023, kunstner_heavy-tailed_2024}. More recently, \citet{sreckovic_batch_2025} show that SGD with carefully tuned hyperparameters can perform similarly to Adam in small-batch settings. \citet{glentis_revisiting_2026} attribute much of the Adam--SGD gap to SGD's inability to sustain large effective learning rates and derive frequency-dependent bounds on stochastic-gradient magnitudes. Their theory motivates clipping but does not characterize how low-frequency tokens are learned over the optimization trajectory. In contrast, we directly prove slow tail-token learning and establish an iteration-complexity separation between GD and SignGD.

\looseness=-1
\textbf{Geometry-aware Optimizers.} SignGD is often considered a proxy for Adam because coordinate-wise normalization makes both methods less sensitive to gradient magnitudes \citep{kunstner_noise_2023, bernstein_old_2024, xie_adam_2025, balles_dissecting_2018, kunstner_heavy-tailed_2024}. Analyses of adaptive and steepest descent methods typically posit favorable geometric assumptions. SignGD is $\ell_\infty$ steepest descent and is studied under $\ell_\infty$ smoothness \citep{bernstein_signsgd_2018, balles_dissecting_2018}. Adam and AdamW admit related $\ell_\infty$ interpretations and analyses under coordinate-wise smoothness \citep{xie_implicit_2024, xie_adam_2025}. AdaGrad-type methods exploit diagonal, anisotropic, or adaptive smoothness \citep{duchi_adaptive_nodate, ward_adagrad_2020, liu_adagrad_2024, jiang_provable_2025}, while Shampoo-style methods exploit matrix geometries \citep{gupta_shampoo_2018, xie_structured_2025}. Recent work further relates adaptive and standard norm-smoothness, introduces adaptive gradient variance \citep{xie_tale_2025}, and establishes unified and accelerated rates for AdaGrad-type preconditioning under anisotropic or matrix smoothness and noise \citep{kovalev_adaptive_preconditioning_2026}. Normalized steepest descent and trust-region methods also admit rates governed by norm-specific smoothness and the solution radius \citep{cutkosky_normalized_2020, kovalev_orthogonalization_2025}. Because normalized steepest descent directions are computed by a linear minimization oracle over the norm ball, these methods are also connected to classical, stochastic, and momentum Frank--Wolfe analyses \citep{jaggi_frank_wolfe_2013, hazan_variance_reduced_2016, li_heavy_ball_2021, pethick_norm_constrained_2025, nazykov_stochastic_2024}.

\begin{figure}[t]
    \centering

    \begin{minipage}[t]{0.27\textwidth}
        \centering
        {\scriptsize\textbf{(a)}}
        \includegraphics[width=\linewidth]{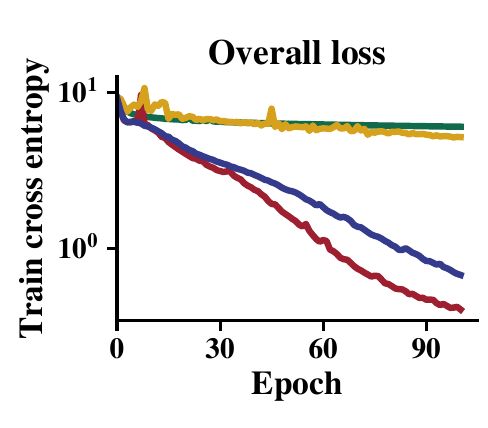}
    \end{minipage}%
    \hspace{2pt}
    \begin{minipage}[t]{0.27\textwidth}
        \centering
        {\scriptsize\textbf{(b)}}
        \includegraphics[width=\linewidth]{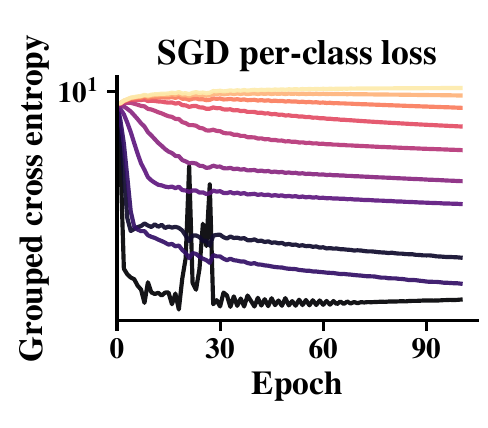}
    \end{minipage}%
    \hspace{2pt}
    \begin{minipage}[t]{0.27\textwidth}
        \centering
        {\scriptsize\textbf{(c)}}
        \includegraphics[width=\linewidth]{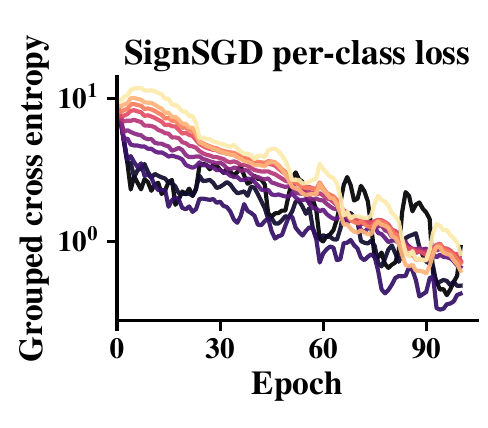}
    \end{minipage}

    \vspace{-1pt}

    \begin{minipage}[c]{0.28\textwidth}
        \centering
        \includegraphics[width=\linewidth]{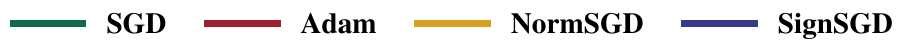}
    \end{minipage}%
    \hspace{2pt}
    \begin{minipage}[c]{0.32\textwidth}
        \centering
        \includegraphics[width=\linewidth]{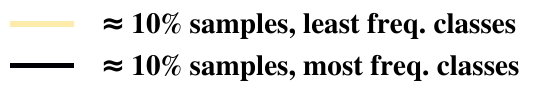}
    \end{minipage}

    \caption{
    \textbf{SignSGD makes more uniform progress across token frequencies than SGD when training a transformer model.}
    \textbf{(a)} Training loss for SGD, Adam, normalized SGD, and SignSGD when training all parameters of a one-layer transformer on the PTB dataset.
    \textbf{(b,c)} Training loss per token-frequency group for SGD and SignSGD, respectively. Tokens are sorted by frequency and partitioned into ten groups, each accounting for roughly $10\%$ of the training samples. Dark curves represent the most frequent tokens and light yellow curves the least frequent tokens. SGD leaves the low-frequency groups at high loss, whereas SignSGD makes progress across all frequency groups. The overall SGD and SignSGD losses in \textbf{(a)} correspond to the averages of the groups in \textbf{(b)} and \textbf{(c)}, respectively.
    }
    \label{fig:transformer-exp-final}
\end{figure}

\looseness=-1
\textbf{Simplified models.} Prior works have developed simplified models to understand the benefits of adaptive optimizers such as Adam compared to SGD. For instance, \citet{kunstner_scaling_2025} study SignGD and GD under a diagonal quadratic model (bigram model with MSE loss) with heavy-tailed class imbalance. Other quadratic models are inspired by transformer Hessian structure or batch-size effects \citep{zhang_why_2024, sreckovic_batch_2025}. Gradient flow and continuous-time SignGD have been studied on a simple imbalanced problem \citep{kunstner_heavy-tailed_2024}. Instead of using a quadratic proxy, we introduce a new model by directly simplifying the transformer architecture and preserving the softmax objective. This gives a simplified language-modeling problem in which the softmax still couples coordinates, allowing us to isolate the effect of heavy-tailed class imbalance without introducing additional separable structure or using continuous-time analysis.



%% file: sections/minimal-language-model.tex
\section{Towards a Minimal Language Model: From Transformers to the Softmax Unigram Model}\label{sec:minimal_model}

\begin{figure}[!htbp]
    \centering
    \setlength{\tabcolsep}{0pt}

    \begin{subfigure}[t]{0.235\textwidth}
        \centering
        \includegraphics[width=\linewidth]{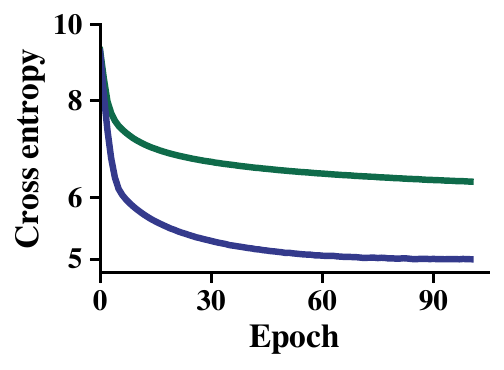}
        \caption{Frozen one-layer transformer}
        \label{fig:combined-embedding-transformer:frozen}
    \end{subfigure}\hfill
    \begin{subfigure}[t]{0.235\textwidth}
        \centering
        \includegraphics[width=\linewidth]{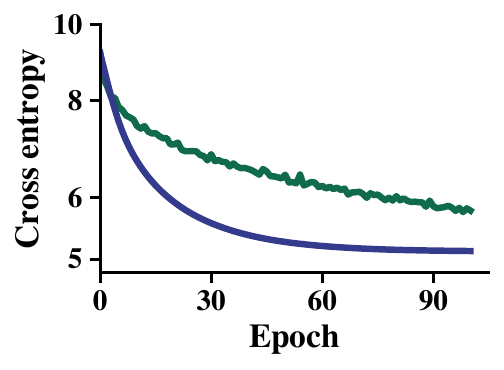}
        \caption{Bigram model}
        \label{fig:combined-embedding-transformer:embedding}
    \end{subfigure}\hfill
    \begin{subfigure}[t]{0.235\textwidth}
        \centering
        \includegraphics[width=\linewidth]{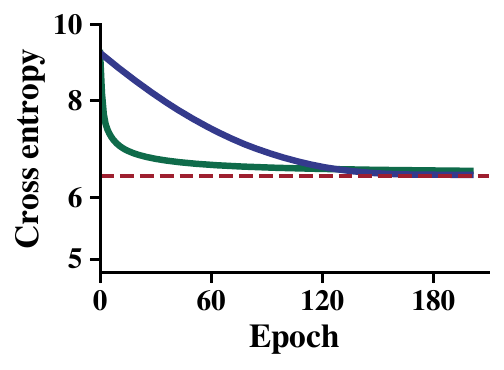}
        \caption{Unigram model}
        \label{fig:combined-embedding-transformer:embedding-one}
    \end{subfigure}\hfill
    \begin{subfigure}[t]{0.235\textwidth}
        \centering
        \includegraphics[width=\linewidth]{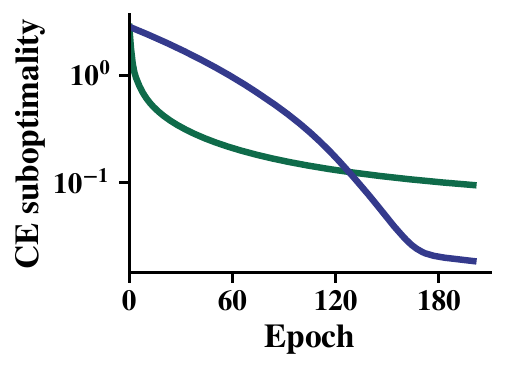}
        \caption{Unigram model sub-optimality}
        \label{fig:combined-embedding-transformer:embedding-one-sub}
    \end{subfigure}

    \vspace{0.35em}
    \includegraphics[width=0.4\textwidth]{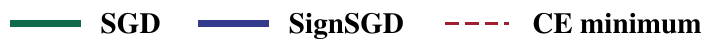}

    \caption{
    The SignSGD and SGD gap persists as the language-modeling setup is progressively simplified, starting from a one-layer transformer.
    \textbf{(a)} We freeze the embedding and attention layers and only train the classification head. 
    \textbf{(b)} We remove the attention layer from \textbf{(a)} and train only the output classifier, essentially yielding a bigram model.
    \textbf{(c)} We collapse all token embeddings to a shared vector, reducing the task to unigram softmax learning. The red dashed line shows the minimum achievable cross-entropy loss for the unigram model, i.e., the entropy of the token distribution. SGD initially descends faster but plateaus above this minimum, while SignSGD approaches it.
    \textbf{(d)} The same experiment as in \textbf{(c)}, but we plot the function sub-optimality, i.e., the difference between the cross-entropy loss and the entropy of the token distribution, to clearly visualize the gap. 
    }
    \label{main:figure_simplification}
\end{figure}

In this section, we progressively simplify the transformer language-modeling setup to obtain the softmax unigram model. \citet{kunstner_heavy-tailed_2024} has shown empirically that under heavy-tailed token frequencies, sign-based and adaptive methods can substantially outperform Euclidean methods for training transformers. As a starting point, we replicate this observation in \Cref{fig:transformer-exp-final}, where SignSGD outperforms SGD by a large margin on a one-layer transformer. However, these transformer experiments involve several ingredients at once: token embeddings, the attention layer, a trainable classification head, sequence dependence across tokens, and stochastic mini-batch noise. When all of these components are present simultaneously, the experiments do not reveal which of them are necessary for the gap, and the full setting is so far too complex for a rigorous theory.

We simplify the setup in two stages, checking after each step that the gap between sign-based and Euclidean methods survives. First, we simplify the architecture: we freeze all components except the classification head, then remove the attention layer, and finally collapse all token embeddings to a single shared vector. The result is the stochastic version of the softmax unigram model.
Second, we remove mini-batch noise and train the same model in the full-batch setting, where the optimizer gap persists.
Complete experimental details, including the independently tuned
learning rates, are provided in
\Cref{app:sec3-experimental-details}.

\subsection{Progressive Architectural Simplification}

Our starting point is the one-layer transformer of \Cref{fig:transformer-exp-final} trained on Penn Treebank (PTB) \citep{penn_treebank_ldc1999}. As a first step, we freeze the token embeddings and the attention layer and train only the classification head using mini-batch SignSGD and SGD (\Cref{fig:combined-embedding-transformer:frozen}).
This setting stays close to the experiment of \Cref{fig:transformer-exp-final}: the representation is fixed, but the model must still learn the final softmax readout under heavy-tailed language data.
Consistent with \citet{kunstner_heavy-tailed_2024}, who compare sign descent with GD in a similar frozen-head setting under full-batch training, SignSGD retains a large advantage over SGD. This already shows that jointly training all layers of the network is not necessary for the gap. 

Having verified in \Cref{fig:combined-embedding-transformer:frozen} that the optimizer gap persists after freezing most trainable components, we next simplify the architecture itself. Specifically, removing the attention layer yields a model with a frozen embedding layer and a trainable classification head.  This is a \textit{bigram language model}. It eliminates sequence-dependent attention and most architectural complexity while still retaining token-dependent predictions. As shown in \Cref{fig:combined-embedding-transformer:embedding}, the gap between SignSGD and SGD still exists in this simplified bigram setting.

Finally, we collapse the token representation so that every token maps to the same fixed embedding vector. This removes the remaining sequence dependence and leaves a model whose task is simply to fit the marginal token distribution with the softmax activation, i.e., a \textit{unigram model}. Even after this drastic simplification, a clear gap remains in \Cref{fig:combined-embedding-transformer:embedding-one}: SGD makes fast initial progress but plateaus above the minimum achievable loss, while SignSGD continues to improve and approaches it. The sub-optimality plot in \Cref{fig:combined-embedding-transformer:embedding-one-sub} shows this gap directly. 
This collapsed model therefore isolates the optimizer separation from architectural complexity, and we use it as the basis for the theory that follows.

\subsection{Removing Stochasticity: The Deterministic Softmax Unigram Model}\label{sec:softmax_unigram_model}
\looseness=-1
After simplifying the architecture, the remaining major complication is mini-batch noise. It can obscure whether the advantage of sign-based methods comes from the stochastic noise itself or from the underlying deterministic geometry of the loss (and these two factors could also be entangled). Therefore, we now remove stochasticity and train the collapsed model in the full-batch setting, which amounts to optimizing the followining objective with exact gradients.

\begin{equation}\label{main:unigram}
\boxed{
    f_{\mathrm{uni}}(\theta) = \KL\!\left(p\|\softmax(\theta)\right)
}
\end{equation}
where $\KL(p\|q)=\sum_{i=1}^d p_i\log\frac{p_i}{q_i}$ is the Kullback--Leibler divergence between two discrete distributions $p$ and $q$ over $d$ classes, and $\softmax(\theta)_i = \frac{\exp(\theta_i)}{\sum_{j=1}^d \exp(\theta_j)}$ is the softmax function for parameter vector $\theta\in\mathbb{R}^d$.

We call this the \textit{softmax unigram model}: given a vocabulary of size $d$, we learn logits $\theta\in\mathbb{R}^d$ to match the target token distribution $p$, where $p_k$ is the frequency of the $k$-th token in the dataset.
In particular, we take the target distribution $p$ to be the idealized Zipfian distribution of \Cref{main:zipf}, which characterizes heavy-tailed class imbalance. This is one further simplification step: it replaces the empirical token frequencies of the dataset with an exact power law, a standard approximation for token frequencies in natural language \citep{piantadosi_zipfs_2014}.

\begin{definition}[Zipfian distribution]\label{main:zipf}
    The \textit{Zipfian distribution} $p$ on $[d]$ is defined by $p_k=k^{-1}/\sum_{j=1}^d j^{-1}$ for $k\in[d]$, so that $p_k \propto 1/k$.
\end{definition}

For the deterministic case, the GD and SignGD algorithms are described in \Cref{alg:gd} and \Cref{alg:signgd}, respectively.
As shown in \Cref{main:figure_softmax}, even in the deterministic setting, coordinate-wise methods such as Adam and SignGD significantly outperform GD on the softmax unigram model. To understand why GD struggles, we inspect the loss for each token. We can express the KL loss in per-coordinate non-negative terms:
\begin{align*}
    f_{\mathrm{uni}}(\theta) = \sum_{i=1}^d p_i\bigg(\log\frac{p_i}{q_i} - 1 + \frac{q_i}{p_i}\bigg)
\end{align*}
where $q = \softmax(\theta)$.
The sum equals the KL because $p$ and $q$ each sum to one, and each term is non-negative because $\log x\ge 1-1/x$.
Note that each term is minimized when $q_i=p_i$.
Thus each term $\log\frac{p_i}{q_i} - 1 + \frac{q_i}{p_i}$ reflects how well $q_i$ matches the target probability $p_i$ for each coordinate $i$, which lets us inspect progress on different parts of the distribution. 
\Cref{main:figure_softmax}(c) and \Cref{main:figure_softmax}(d) reveal that GD makes especially slow progress on the low-probability tail tokens, shown by the yellow curves, whereas SignGD makes much more uniform progress across the distribution. The same qualitative contrast between SGD and SignSGD appears when training the transformer language model in \Cref{fig:transformer-exp-final}(b,c).

The softmax unigram model, which is convex and smooth, is what we analyze in the rest of the paper. It is equivalent to a language model with no sequence dependence, where all tokens have the same embedding vector, and it simply learns the token frequencies. At the same time, it retains the key optimization difficulty faced by GD under heavy-tailed class imbalance. The coordinates remain coupled through the softmax activation, unlike the simpler separable quadratic models often used in prior analyses.  In this sense, the model stays meaningfully connected to the language-modeling experiments above, while being substantially easier to analyze than the full transformer setup.


%% file: sections/deterministic-softmax-unigram.tex
%

\section{Deterministic Softmax Unigram Analysis}\label{sec:deterministic}
In \Cref{main:figure_softmax}, the deterministic experiments already show that GD learns tail tokens much more slowly than SignGD. At a high level, this phenomenon arises because tail logits receive tiny updates with GD, which makes progress on the tail slow. We formalize this bottleneck through upper and lower bounds on GD. We then show that SignGD avoids this bottleneck and achieves faster convergence, yielding a separation between SignGD and GD.

\textbf{Notation.} We provide some notation used throughout the deterministic and stochastic analyses.
We say a function $f:\R^d\to\R$ is $L$-smooth w.r.t. norm $\|\cdot\|$ if for all $x, y \in \R^d$ we have $\|\nabla f(x) - \nabla f(y)\|_\star \leq L \|x - y\|$ where $\|\cdot\|_\star$ denotes the dual norm of $\|\cdot\|$.
Let the smallest such constant be denoted by $\smoothness{f}{\|\cdot\|}$. We use $\|\cdot\|_p$ to denote the $\ell_p$ norm for $p \in [1, \infty]$. For a fixed norm $\norm{\cdot}$, any $x\in\R^d$, and any nonempty set $\mathcal S\subseteq\R^d$, define
\begin{equation*}
    \operatorname{dist}_{\norm{\cdot}}(x,\mathcal S)
    :=
    \inf_{y\in\mathcal S}\norm{x-y}.
\end{equation*}
When a norm $\norm{\cdot}$ is fixed and $\mathcal X^\star:=\argmin_{z\in\R^d}f(z)\neq\emptyset$, we let $x_\star\in\mathcal X^\star$ denote a minimizer of minimum norm, i.e. $\norm{x_\star}=\min_{z\in\mathcal X^\star}\norm{z}$.
We write \(f^\star := f(x_\star)=\min_{z\in\R^d}f(z)\) for the minimum value of \(f\).
We apply $\sign(\cdot)$ coordinate-wise to vectors, with the convention $\sign(0)=0$.

\subsection{Failure of GD: The Tail Learning Bottleneck}\label{sec:gd_tail_bottleneck}

The gradient of $f_{\mathrm{uni}}$ defined in \Cref{main:unigram} is $\nabla f_{\mathrm{uni}}(\theta)=\softmax(\theta)-p$. Thus, on tail coordinates where both the target and model probabilities are small, the residual $|\softmax(\theta)_k-p_k|$ is small, and GD makes only small updates to the corresponding logits. However, when starting from uniform initialization, reaching the true probabilities in the tail requires making large updates to the model tail logits. This mismatch causes GD to make slow progress.
We formalize this intuition with the following upper and lower bounds on GD.

\begin{theorem}[GD upper bound]\label{main:gd_upper_bound}
    Consider optimizing $f_{\mathrm{uni}}(\theta)$ in \Cref{main:unigram} using GD, i.e., $\theta_{t+1} = \theta_t - \eta \nabla f_{\mathrm{uni}}(\theta_t)$, initialized at $\theta_0=0$ with constant step size $\eta=2$.
    Then for any $T\geq 1$, it holds that
\begin{equation*}
    f_{\mathrm{uni}}(\theta_T)-f_{\mathrm{uni}}^\star
    \leq
    \frac{d}{4T}.
\end{equation*}
\end{theorem}

The upper bound in \Cref{main:gd_upper_bound} is linear in $d$, suggesting slow convergence and a vacuous convergence guarantee when $T \ll d$. This matches the empirical behavior of GD in \Cref{main:figure_softmax} where $d = 10^4$.
The proof of \Cref{main:gd_upper_bound} is in \Cref{app:proof_gd_upper_bound}.
Next, the following lower bound shows that such slow convergence is unavoidable for GD.

\begin{theorem}[GD lower bound]\label{main:gd_lower_bound}
There exist absolute constants $c,C>0$ such that the following holds for all sufficiently large $d$.
Suppose we run GD, $\theta_{t+1} = \theta_t - \eta \nabla f_{\mathrm{uni}}(\theta_t)$, initialized at $\theta_0 = 0$ with a constant step size $0 < \eta < d/2$ on $f_{\mathrm{uni}}(\theta)$ in \Cref{main:unigram}.
Then, whenever
$
    C\log d \leq \eta T \leq \frac{d\log d}{C}
$,
it holds that
\begin{equation*}
    f_{\mathrm{uni}}(\theta_T)-f_{\mathrm{uni}}^\star
    \geq
    \frac{c}{\log d}
    \log\left(\frac{d\log d}{\eta T}\right).
\end{equation*}
    Consequently, there exist absolute constants $A>0$ and $\epsilon_0\in(0,1/A)$ such that for every fixed $\epsilon\in(0,\epsilon_0)$, there exist constants $c_\epsilon,d_\epsilon>0$, depending only on $\epsilon$, for which the number of iterations required for GD to reach $\epsilon$-suboptimality satisfies, for all $d\geq d_\epsilon$ and all $0<\eta<d/2$,
\begin{equation*}
    \min\{T:f_{\mathrm{uni}}(\theta_T)-f_{\mathrm{uni}}^\star\leq \epsilon\}
    \geq
    \frac{c_\epsilon}{\eta}d^{1-A\epsilon}\log d.
\end{equation*}
Here the minimum is $+\infty$ if the target accuracy is never reached.
\end{theorem}

\begin{figure}[t]
\centering
\begin{minipage}[t]{0.48\textwidth}
\vspace{0pt}
\captionsetup{hypcap=false}
\begin{algorithm}[H]
\caption{Gradient Descent (GD)}\label{alg:gd}
\begin{algorithmic}[1]
\State \textbf{Input:} Start point $x_0$, learning rate $\eta$
\For{$t=1$ to $\dots$  $T$}
  \State $g_t \gets \nabla f(x_{t-1})$
  \State $x_t \gets x_{t-1}-\eta g_t$
\EndFor
\end{algorithmic}
\end{algorithm}
\end{minipage}\hfill
\begin{minipage}[t]{0.48\textwidth}
\vspace{0pt}
\captionsetup{hypcap=false}
\begin{algorithm}[H]
\caption{Sign Gradient Descent (SignGD)}\label{alg:signgd}
\begin{algorithmic}[1]
\State \textbf{Input:} Start point $x_0$, learning rate $\eta$
\For{$t=1$ to $\dots$  $T$}
  \State $g_t \gets \nabla f(x_{t-1})$
  \State $x_t \gets x_{t-1}-\eta\sign(g_t)$
\EndFor
\end{algorithmic}
\end{algorithm}
\end{minipage}
\end{figure}

We remark that this lower bound covers a broad range of learning rates $\eta\in(0,d/2)$.
Indeed, one can show that when using learning rate $\eta\gtrsim\log d$, GD would become unstable near a minimizer and hence does not converge; see \Cref{main:gd_local_instability} for details.

The proof of \Cref{main:gd_lower_bound} is in \Cref{app:proof_gd_lower_bound}.
The key idea is to track the marginal distribution for coordinate pairs $(i,2i)$.
For each such pair, the target logit gap is $\log(p_i/p_{2i})=\log 2$, while the initial logit gap is $0$ at the initialization $\theta_0 = 0$; meanwhile, GD can only grow the logit gap very slowly on tail coordinates because the corresponding updates are themselves tiny. As a result, many tail pairs remain far from their target relative probabilities for polynomially many iterations. Converting these unresolved tail errors into the KL loss yields the lower bound in \Cref{main:gd_lower_bound}.

Next, we show that SignGD avoids this tail learning bottleneck and converges faster.

\subsection{Success of SignGD}\label{sec:sign_avoids_tail}

\begin{figure}[!t]
\centering
\begin{minipage}[t]{0.48\textwidth}
\vspace{0pt}
\captionsetup{hypcap=false}
\begin{algorithm}[H]
\caption{Normalized Steepest Descent (NSD)}\label{alg:nsd}
\begin{algorithmic}[1]
\State \textbf{Input:} Start point $x_0$, norm $\norm{\cdot}$, learning rate $\eta$
\For{$t=1$ to $\dots$  $T$}
  \State $g_t \gets \nabla f(x_{t-1})$
  \State $u_t \gets \argmax_{\norm{y} \leq 1}\langle g_t,y\rangle$
  \State $x_t \gets x_{t-1}-\eta u_t$
\EndFor
\end{algorithmic}
\end{algorithm}
\end{minipage}\hfill
\begin{minipage}[t]{0.48\textwidth}
\vspace{0pt}
\captionsetup{hypcap=false}
\begin{algorithm}[H]
\caption{Normalized Gradient Descent}\label{alg:normalized-gd}
\begin{algorithmic}[1]
\State \textbf{Input:} Start point $x_0$, learning rate $\eta$
\For{$t=1$ to $\dots$  $T$}
  \State $g_t \gets \nabla f(x_{t-1})$
  \State $u_t \gets g_t/\norm{g_t}_2$
  \State $x_t \gets x_{t-1}-\eta u_t$
\EndFor
\end{algorithmic}
\end{algorithm}
\end{minipage}
\end{figure}

For the softmax unigram objective in \Cref{main:unigram}, $\nabla f_{\mathrm{uni}}(\theta)=\softmax(\theta)-p$, so the SignGD update specializes to
\[
    \theta_t
    =\theta_{t-1}-\eta\sign\!\left(\softmax(\theta_{t-1})-p\right).
\]
This update depends only on the coordinate-wise signs of the gradient, so every coordinate with a nonzero gradient can move at a comparable speed even when the original magnitude of its derivative $\abs{[\softmax(\theta)]_k-p_k}$ is small. In this sense, SignGD can move tail coordinates efficiently even when their gradient entries are small.

We derive the deterministic convergence rate of SignGD from a generic guarantee for normalized steepest descent~(NSD), whose update is given in \Cref{alg:nsd}.\footnote{We take the normalized gradient and the steepest-descent direction to be zero whenever their input vector is zero.}
Note that SignGD is NSD with respect to the $\ell_\infty$ norm, while normalized GD, defined in \Cref{alg:normalized-gd}, is NSD with respect to the $\ell_2$ norm. For comparison, we also derive convergence guarantees for normalized GD, illustrating the difference between the $\ell_\infty$ and $\ell_2$ geometries. The following theorem gives a generic convergence guarantee for NSD under a uniform bound on the distance of the iterates to the minimizers.

\begin{theorem}[Conditional convergence of NSD]\label{main:nsd_convergence}
    Suppose $f:\R^d\to\R$ is convex and $L$-smooth with respect to a norm $\norm{\cdot}$, with $\mathcal X^\star:=\argmin_{x\in\R^d}f(x)\neq\emptyset$. Fix a radius $R>0$. Suppose we run NSD for this norm $\|\cdot\|$ with horizon $T \geq 3$ and learning rate $\eta = \frac{R\log T}{T}$. If the iterates satisfy $\operatorname{dist}_{\norm{\cdot}}(x_{t-1},\mathcal X^\star)\le R$ for $1\le t\le T$, then
        \begin{equation*}
            f(x_T) - f^\star
            \le
            \frac{LR^2\log T}{T}.
        \end{equation*}
\end{theorem}

The proof is in \Cref{app:proof_nsd_convergence}.
To apply this conditional guarantee to obtain the convergence rate of NSD, the main technical obstacle is to establish such a uniform distance bound.
This condition is not automatically guaranteed for general functions, but for the softmax unigram model, we can exploit the structure of the softmax function and the exact gradient to show that the iterates of both SignGD and normalized GD remain at a uniformly bounded distance from the minimizer set, as formalized in the following two lemmas.

\begin{lemma}[Distance-to-minimizer-set bound for SignGD]\label{main:softmax_radius}
    Suppose we run SignGD on $f_{\mathrm{uni}}$ initialized at $\theta_0=0$ with horizon $T\ge2$, and  learning rate $\eta = \frac{\log d\,\log T}{2T}$. Then
    \begin{equation*}
        \operatorname{dist}_{\norm{\cdot}_\infty}(\theta_t,\argmin f_{\mathrm{uni}})\le\frac{\log d}{2}
        \qquad\text{for all }0\le t\le T.
    \end{equation*}
\end{lemma}

\begin{lemma}[Distance-to-minimizer-set bound for normalized GD]\label{main:normalized_gd_radius}
    Suppose we run normalized GD on $f_{\mathrm{uni}}$ with $\theta_0=0$, horizon $T\ge2$, and fixed learning rate $\eta =\frac{2\sqrt d\,\log T}{T}$. Then
    \begin{equation*}
        \operatorname{dist}_{\norm{\cdot}_2}(\theta_t,\argmin f_{\mathrm{uni}})\le 2\sqrt d
        \qquad\text{for all }0\le t\le T.
    \end{equation*}
\end{lemma}

The proofs are in \Cref{app:proof_softmax_radius,app:proof_normalized_gd_radius}.
To obtain the explicit convergence rates of SignGD and normalized GD, it then remains to determine the smoothness constants of $f_{\mathrm{uni}}$ with respect to the $\ell_\infty$ and $\ell_2$ norms.

\paragraph{Smoothness under different geometries.}
For a generic function, the $\ell_\infty$-smoothness constant can be up to $d$ times larger than the $\ell_2$-smoothness constant. For the softmax unigram model, however, \Cref{main:unigram_smoothness_lemma} gives
\begin{equation}\label{eq:unigram_smoothness_constants}
    L_{\norm{\cdot}_\infty}(f_{\mathrm{uni}})=1 \quad\text{and}\quad L_{\norm{\cdot}_2}(f_{\mathrm{uni}})=1/2,
\end{equation}
so the two constants differ only by a factor of two. The favorable $\ell_\infty$ smoothness of $f_{\mathrm{uni}}$ comes from its log-sum-exp term: $f_{\mathrm{uni}}$ differs from the log-sum-exp function only by an affine function and hence has the same $\ell_\infty$-smoothness constant. This smoothness property of log-sum-exp is also exploited in almost-linear-time algorithms for approximate max flow \citep{kelner_almost-linear-time_2013}.
Therefore, the different dimensional dependence of the two upper bounds instead comes from how far the iterates can be from the minimizer set: the bounds are $\cO(\sqrt d)$ in the $\ell_2$ norm and $\cO(\log d)$ in the $\ell_\infty$ norm. Substituting these smoothness constants and distance bounds into \Cref{main:nsd_convergence} gives the following convergence rates for SignGD and normalized GD.

\begin{corollary}[SignGD upper bound]\label{main:deterministic_sign}
    Suppose we run SignGD on $f_{\mathrm{uni}}$ initialized at $\theta_0=0$ with horizon $T\ge3$ and learning rate $\eta = \frac{\log d\,\log T}{2T}$. Then
        \begin{equation*}
            f_{\mathrm{uni}}(\theta_T)-f_{\mathrm{uni}}^\star
            \le \frac{(\log d)^2\log T}{4T}.
        \end{equation*}
\end{corollary}

\begin{corollary}[Normalized GD upper bound]\label{main:deterministic_normgd}
    Suppose we run normalized GD on $f_{\mathrm{uni}}$ with $\theta_0=0$, horizon $T\ge3$, and fixed learning rate $\eta=\frac{2\sqrt d\,\log T}{T}$. Then
        \begin{equation*}
            f_{\mathrm{uni}}(\theta_T)-f_{\mathrm{uni}}^\star
            \le \frac{2d\log T}{T}.
        \end{equation*}
\end{corollary}

The SignGD upper bound scales as $\widetilde{\cO}((\log d)^2/T)$, whereas the normalized-GD upper bound scales as $\widetilde{\cO}(d/T)$. Thus, the SignGD upper bound is smaller by a factor of order $d/(\log d)^2$. Furthermore, SignGD reaches suboptimality below $\epsilon$ in $\widetilde{\cO}((\log d)^2/\epsilon)$ iterations, while \Cref{main:gd_lower_bound} shows that GD requires $\Omega_\epsilon(\eta^{-1}d^{1-A\epsilon}\log d)$ iterations for every fixed $\epsilon\in(0,\epsilon_0)$. For a GD stepsize $\eta$ independent of $d$, this gives a strict separation between SignGD and GD on the softmax unigram model.

\subsection{Weight Decay as an Alternative Mechanism for Iterate Control}\label{sec:weight_decay}

The analysis in \Cref{sec:sign_avoids_tail} establishes convergence through a uniform bound on the distance of the iterates to the minimizer set. For SignGD and normalized GD, establishing this bound exploits the exact gradient structure of the softmax unigram objective, making the argument specific to the deterministic full-gradient setting. Weight decay provides a more general alternative: the contraction in each update directly controls the iterate norms and supplies a generic substitute for the problem-specific distance certificate.
We therefore analyze normalized steepest descent with weight decay (NSD-WD), defined in \Cref{alg:NSD-WD}, by adapting the convergence guarantee of \citet{xie_implicit_2024} to a fixed, horizon-dependent learning rate. This generic norm control also allows the framework to extend naturally to the stochastic-gradient setting.

\begin{figure}[!t]
\centering
\begin{minipage}{0.68\textwidth}
\vspace{0pt}
\captionsetup{hypcap=false}
\begin{algorithm}[H]
\caption{Normalized Steepest Descent with Weight Decay (NSD-WD)}\label{alg:NSD-WD}
\begin{algorithmic}[1]
\State \textbf{Input:} Start point $x_0$, learning rate $\eta$, weight decay $\lambda>0$
\For{$t=1$ to $\dots$  $T$}
  \State $u_t \gets \argmax_{\norm{y} \leq 1} \langle \nabla f(x_{t-1}), y\rangle$
  \State $x_t \gets (1-\lambda\eta)x_{t-1}-\eta u_t$
\EndFor
\end{algorithmic}
\end{algorithm}
\end{minipage}
\end{figure}

\begin{theorem}[Convergence rate of NSD-WD]\label{main:NSD_det}
    Suppose $f$ is convex and $L$-smooth with respect to $\norm{\cdot}$. Given horizon $T \geq 2$, suppose NSD-WD in \Cref{alg:NSD-WD} is run with $\lambda\leq 1/\norm{x_\star}$ and learning rate $\eta=\frac{\log T}{\lambda T}$. Let $D=\max\{\norm{x_0},1/\lambda\}$. Then
    \begin{equation*}
        f(x_T) - f^\star \leq \frac{2L(1 + \lambda D)^2\log T}{\lambda^2 T}.
    \end{equation*}
    In particular, choosing $x_0=0$ and $\lambda=1/\norm{x_\star}$ yields
    \begin{align*}
        f(x_T) - f^\star \leq \frac{8L\norm{x_\star}^2\log T}{T}.
    \end{align*}
\end{theorem}

Under the $\ell_\infty$ and $\ell_2$ norms, NSD-WD becomes SignGD and normalized GD with weight decay, respectively. By \eqref{eq:unigram_smoothness_constants} and \Cref{main:unigram_norm_lemma1}, SignGD has $L=1$ and $\norm{x_\star}_\infty=\frac12\log d$, while normalized GD has $L=1/2$ and $\norm{x_\star}_2\leq\sqrt d$; applying \Cref{main:NSD_det} with $x_0=0$, $\lambda=2/\log d$ for SignGD, and $\lambda=1/\sqrt d$ for normalized GD gives the following corollaries.

\begin{corollary}[SignGD with weight decay]\label{main:deterministic_sign_wd}
    For all sufficiently large $d$, SignGD on $f_{\mathrm{uni}}$ with $\theta_0=0$, $T\ge2$, weight decay factor $\lambda=\frac{2}{\log d}$, and fixed learning rate $\eta = \frac{\log d\,\log T}{2T}$ satisfies
\begin{equation*}
    f_{\mathrm{uni}}(\theta_T)-f_{\mathrm{uni}}^\star
    \leq
    \frac{2(\log d)^2\log T}{T}.
\end{equation*}
\end{corollary}

\begin{corollary}[Normalized GD with weight decay]\label{main:deterministic_normgd_wd}
    For all sufficiently large $d$, normalized GD on $f_{\mathrm{uni}}$ with $\theta_0=0$, $T\ge2$, weight decay factor $\lambda=\frac{1}{\sqrt d}$, and fixed learning rate $\eta = \frac{\sqrt d\,\log T}{T}$ satisfies
\begin{equation*}
    f_{\mathrm{uni}}(\theta_T)-f_{\mathrm{uni}}^\star
    \leq
    \frac{4d\log T}{T}.
\end{equation*}
\end{corollary}

Their dimensional dependence matches the corresponding results in \Cref{sec:sign_avoids_tail} up to constant factors: $\widetilde{\cO}((\log d)^2/T)$ for SignGD and $\widetilde{\cO}(d/T)$ for normalized GD. Thus, weight decay preserves the advantage of the $\ell_\infty$ geometry while providing generic iterate-norm control. In \Cref{fig:figure_wd_real_data}, we compare SignGD and normalized GD with and without weight decay. The same gap appears for synthetic Zipfian probabilities and for token frequencies from real text datasets, and adding weight decay leaves it essentially unchanged. This deterministic analysis isolates the core mechanism separating SignGD from GD. We next reintroduce mini-batch noise and study how it interacts with this loss geometry.

\begin{figure}[!t]
    \centering

    \begin{subfigure}[t]{0.235\textwidth}
        \centering
        \includegraphics[width=\linewidth]{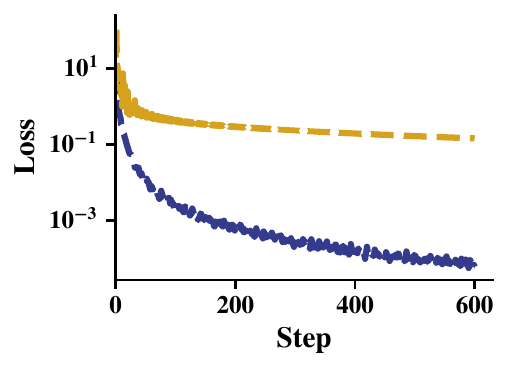}
        \caption{Synthetic}
    \end{subfigure}\hfill
    \begin{subfigure}[t]{0.235\textwidth}
        \centering
        \includegraphics[width=\linewidth]{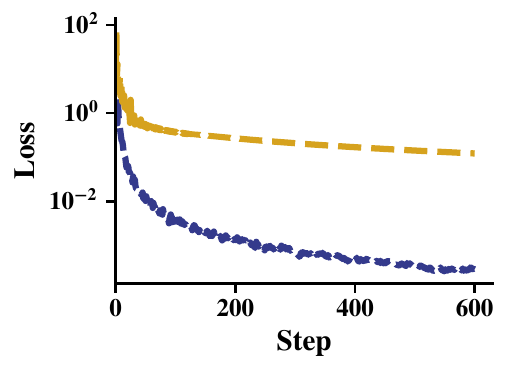}
        \caption{PTB}
    \end{subfigure}\hfill
    \begin{subfigure}[t]{0.235\textwidth}
        \centering
        \includegraphics[width=\linewidth]{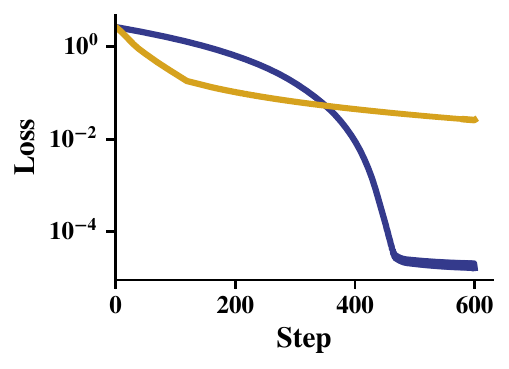}
        \caption{PTB}
    \end{subfigure}\hfill
    \begin{subfigure}[t]{0.235\textwidth}
        \centering
        \includegraphics[width=\linewidth]{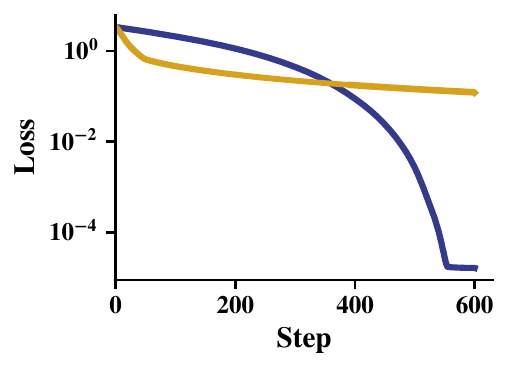}
        \caption{WikiText}
    \end{subfigure}

    \vspace{0.35em}
    \begin{minipage}{0.48\textwidth}
        \centering
        \includegraphics[width=0.87\linewidth]{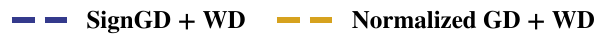}
    \end{minipage}\hfill
    \begin{minipage}{0.48\textwidth}
        \centering
        \includegraphics[width=0.65\linewidth]{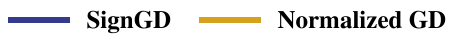}
    \end{minipage}

    \caption{
    \textbf{SignGD outperforms normalized GD across synthetic and real-data token distributions.}
    \textbf{(a,b)} Weight-decay runs on a synthetic distribution and the Penn Treebank (PTB) dataset, using the learning rates prescribed by \Cref{main:deterministic_sign_wd} and \Cref{main:deterministic_normgd_wd}.
    \textbf{(c,d)} Best-performing runs on the PTB and WikiText-2 datasets.
    }
    \label{fig:figure_wd_real_data}
\end{figure}
\FloatBarrier

%% file: sections/stochastic-softmax-unigram.tex
\section{Stochastic Softmax Unigram Analysis}\label{sec:stochastic}
In \Cref{sec:deterministic}, we worked in the deterministic setting to isolate the core mechanism that separates SignGD from GD. 
Here, we reintroduce mini-batch noise and study how it interacts with the deterministic loss geometry. 
We derive a convergence upper bound for SignSGD with weight decay that contains a deterministic term and an additional noise-dependent term. For sufficiently large batches, this bound yields an iteration-complexity separation from full-batch GD.

\subsection{SignSGD under Mini-Batch Noise}

\paragraph{Mini-batch stochastic gradient.}
We begin by deriving the stochastic gradient that arises from standard mini-batch training on the softmax unigram model. At iteration $t$, draw a mini-batch $\xi_{t,1},\ldots,\xi_{t,B}\stackrel{\mathrm{iid}}{\sim}\mathrm{Cat}(p)$. The standard mini-batch cross-entropy loss used in our experiments is
\begin{equation*}
    \widehat f_{\mathrm{ce},B}(\theta;\xi_{t,1:B})
    :=-\frac1B\sum_{b=1}^B\log[\softmax(\theta)]_{\xi_{t,b}}.
\end{equation*}
For the analysis, we subtract the entropy $H(p):=-\sum_{i=1}^d p_i\log p_i$ and use the centered mini-batch loss
\begin{equation*}
    \widehat f_{\mathrm{uni},B}(\theta;\xi_{t,1:B})
    :=\widehat f_{\mathrm{ce},B}(\theta;\xi_{t,1:B})
    -H(p),
    \qquad
    \E[\widehat f_{\mathrm{uni},B}]=f_{\mathrm{uni}}.
\end{equation*}
The centering term is independent of $\theta$, so the two losses have identical gradients and generate the same iterates, i.e.
\begin{equation}\label{eqn:softmax_unigram_stoch_grad}
    g_t :=\nabla\widehat f_{\mathrm{ce},B}(\theta_t;\xi_{t,1:B})
    =\nabla\widehat f_{\mathrm{uni},B}(\theta_t;\xi_{t,1:B})
    =\softmax(\theta_t)-\hat p_t,
    \qquad
    \hat p_t := \frac1B\sum_{b=1}^B e_{\xi_{t,b}}.
\end{equation}
Comparing this gradient with $\nabla f_{\mathrm{uni}}(\theta_t)=\softmax(\theta_t)-p$, the only difference is that the population distribution $p$ is replaced by the empirical token-frequency vector $\hat p_t$. Since $\E[\hat p_t]=p$, the mini-batch gradient is unbiased: $\E[g_t\mid\theta_t]=\nabla f_{\mathrm{uni}}(\theta_t)$. Crucially, the stochastic gradient noise $\varepsilon_t:=g_t-\nabla f_{\mathrm{uni}}(\theta_t)=p-\hat p_t$ is independent of $\theta_t$: it depends only on the sampled mini-batch through $\hat p_t$.

\paragraph{Generic stochastic NSD-WD bound.}
We next state a generic NSD-WD convergence bound for convex, smooth objectives in the stochastic setting. We instantiate it below for SignSGD with weight decay and later in \Cref{sec:sign_vs_normalized_stochastic} for normalized SGD with weight decay. The proof is given in \Cref{app_sec:nsd_proofs}.

\begin{theorem}[Stochastic NSD-WD]\label{main:NSD}
    Suppose $f$ is convex and $L$-smooth with respect to a norm $\norm{\cdot}$. Let $\E_t$ denote expectation conditional on the history up to $x_t$, and suppose the stochastic gradient $g_t$ satisfies $\E_t[g_t] = \nabla f(x_t)$ and $\E_t[\norm{g_t - \nabla f(x_t)}_\star] \leq \sigma$. For a horizon $T\ge2$, run stochastic NSD-WD with weight decay factor $\lambda \leq 1/\norm{x_\star}$ and fixed learning rate $\eta=\frac{\log T}{\lambda T}$. Then
\begin{equation*}
    \E[f(x_T)-f^\star] \leq \frac{2L(1+\lambda D)^2\log T}{\lambda^2T} +\frac{2\sigma}{\lambda},
\end{equation*}
where $D=\max\{\norm{x_0},1/\lambda\}$.
In particular, choosing $x_0=0$ and $\lambda=1/\norm{x_\star}$ yields
\begin{equation*}
    \E[f(x_T)-f^\star] \leq \frac{8L\norm{x_\star}^2\log T}{T}+2\norm{x_\star}\sigma.
\end{equation*}
\end{theorem}
When $\sigma=0$, \Cref{main:NSD} recovers the deterministic rate in \Cref{main:NSD_det}. For fixed $\sigma$ and $\lambda$, the stochastic term in this upper bound does not decay with $T$. In the mini-batch setting below, taking both the batch size and horizon sufficiently large gives a small-error guarantee.

\paragraph{Specialization to SignSGD.}
For the softmax unigram gradient, \Cref{main:unigram_noise_lemma} gives
\[
    \E\norm{g_t-\nabla f_{\mathrm{uni}}(\theta_t)}_2\le \frac1{\sqrt B},
    \qquad
    \E\norm{g_t-\nabla f_{\mathrm{uni}}(\theta_t)}_1
    \le \min\bigg\{2, \frac{\norm{\sqrt p}_1}{\sqrt B}\bigg\},
\]
where the $\sqrt{\cdot}$ operation is applied element-wise. Since the noise is independent of $\theta_t$, the bounds above are derived by bounding the distance from the estimator $\hat p_t$ to the true $p$ in each norm. We now use \Cref{main:NSD} and the stochastic-gradient error bound in the $\ell_1$ norm to obtain an upper bound for SignSGD with weight decay.

\begin{corollary}[SignSGD with weight decay]\label{main:stochastic_sign_fixed}
    For all sufficiently large $d$, consider applying SignSGD to $f_{\mathrm{uni}}$ using the mini-batch gradient $g_t$ defined in \Cref{eqn:softmax_unigram_stoch_grad}, with initialization $\theta_0=0$, batch size $B$, horizon $T\ge2$, weight-decay parameter $\lambda=2/\log d$, and fixed learning rate $\eta=\frac{\log d\,\log T}{2T}$. Then
\begin{equation*}
    \E[f_{\mathrm{uni}}(\theta_T)-f_{\mathrm{uni}}^\star]
    \leq
    \frac{2(\log d)^2\log T}{T}
    + \underbrace{\min\left\{2\log d,4\sqrt{\frac{d\log d}{B}}\right\}}_{=:N_B(d)}.
\end{equation*}
\end{corollary}

At fixed $d$ and $B$, the noise-dependent term $N_B(d)$ in \Cref{main:stochastic_sign_fixed} is independent of $T$. When $B<d$, we have $N_B(d)\geq\min\{2\log d,4\sqrt{\log d}\}$, so this bound does not certify any fixed small target error as $d$ grows. For example, when $B=1$, the noise-dependent term $N_1(d)=2\log d$ exceeds the initialization loss $0.5\log d-\log\log d+\cO(1)$ and therefore does not certify progress from initialization. We show in \Cref{fig:softmax_batch_noise} that SignSGD attains higher loss than SGD and normalized SGD at $B=64$ and $B=1024$, but lower loss at the largest plotted batch size, $B=32768>d$.

\paragraph{Large batches reveal the deterministic advantage.}
For sufficiently large batch sizes, stochastic noise is small enough that the deterministic advantage of SignSGD becomes visible. In this regime, an SGD baseline is not needed to demonstrate the benefits of SignSGD in terms of iteration complexity. We can compare large-batch SignSGD directly with full-batch GD, even though full-batch GD has access to the exact gradient and no sampling noise. The following corollary establishes the iteration complexity of SignSGD to reach a given accuracy.

\begin{figure}[t]
    \centering

    \begin{subfigure}[t]{0.21\textwidth}
        \centering
        \includegraphics[width=\linewidth]{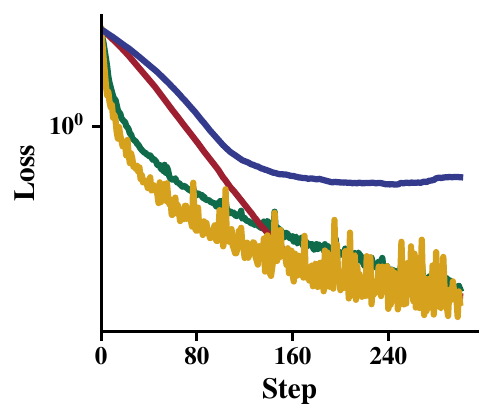}
        \caption{$B=64$}
    \end{subfigure}
    \begin{subfigure}[t]{0.21\textwidth}
        \centering
        \includegraphics[width=\linewidth]{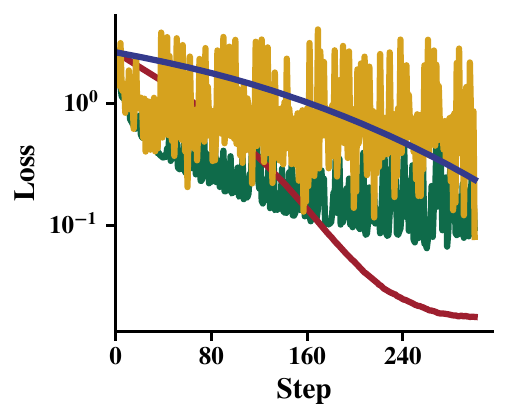}
        \caption{$B=1024$}
    \end{subfigure}
    \begin{subfigure}[t]{0.21\textwidth}
        \centering
        \includegraphics[width=\linewidth]{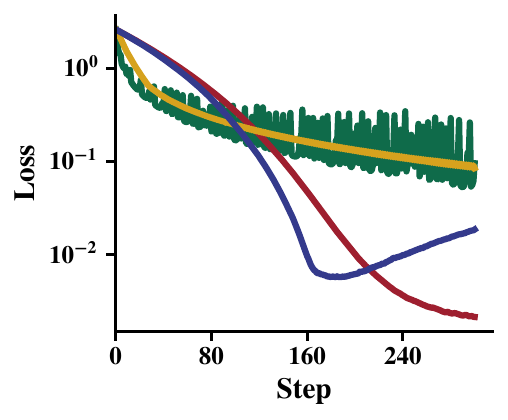}
        \caption{$B=32768$}
    \end{subfigure}

    \vspace{0.35em}
    \includegraphics[width=0.50\textwidth]{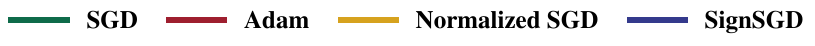}

    \caption{Batch-size comparison on the stochastic softmax unigram model with $p$ following \Cref{main:zipf} with $d=10^4$. The panels show the population KL loss for SGD, Adam, normalized SGD, and SignSGD using different mini-batch sizes.}
    \label{fig:softmax_batch_noise}
\end{figure}

\begin{corollary}[Large-batch SignSGD complexity]\label{main:large_batch_signsgd_complexity}
Let $\theta_{T,B}$ denote the $T$-th SignSGD iterate from \Cref{main:stochastic_sign_fixed} with batch size $B$.
For every fixed $\epsilon>0$ and all sufficiently large $d$, for any sufficiently large batch size satisfying
\[
B\ge \frac{64d\log d}{\epsilon^2},
\]
the number of iterations required to reach $\epsilon$-accuracy is bounded by
\[
\min\left\{T:\E[f_{\mathrm{uni}}(\theta_{T,B})-f_{\mathrm{uni}}^\star]\le\epsilon\right\}
\le
\left\lceil
\frac{8(\log d)^2}{\epsilon}
\log\left(\frac{8(\log d)^2}{\epsilon}\right)
\right\rceil.
\]
\end{corollary}

\Cref{main:large_batch_signsgd_complexity} shows that, with a sufficiently large batch size, SignSGD with weight decay reaches fixed accuracy in $\cO((\log d)^2\log\log d)$ iterations. In contrast, \Cref{main:gd_lower_bound} shows that full-batch GD requires $\Omega_\epsilon(\eta^{-1}d^{1-A\epsilon}\log d)$ iterations for every fixed $\epsilon\in(0,\epsilon_0)$, with $A$ and $\epsilon_0$ as in that theorem. For a GD stepsize $\eta$ independent of $d$, these results give an iteration-complexity separation between large-batch SignSGD with weight decay and full-batch GD, despite the latter having access to exact gradients.

\subsection{Batch-Size Tradeoffs under a Fixed Token Budget}
The previous subsection showed that a sufficiently large batch size can reveal SignSGD's deterministic advantage without constraining the total number of sampled tokens. We now fix the token budget $N=BT$, so increasing the batch size reduces the number of iterations, and compare SignSGD with normalized SGD and vanilla SGD under this constraint.

\subsubsection{Comparison with Normalized SGD}\label{sec:sign_vs_normalized_stochastic}

Recall \Cref{main:NSD} for generic stochastic NSD-WD which we now instantiate for normalized SGD with weight decay, i.e., NSD-WD under the $\ell_2$ norm.

\begin{corollary}[Normalized SGD with weight decay]\label{main:stochastic_norm_fixed}
    Consider applying normalized SGD to $f_{\mathrm{uni}}$ using the mini-batch gradient $g_t$ defined in \Cref{eqn:softmax_unigram_stoch_grad}, with initialization $\theta_0=0$, batch size $B$, horizon $T\ge2$, weight-decay parameter $\lambda=1/\sqrt d$, and fixed learning rate $\eta=\frac{\sqrt d\,\log T}{T}$. Then
\begin{equation*}
    \E[f_{\mathrm{uni}}(\theta_T)-f_{\mathrm{uni}}^\star]
    \leq
    \frac{4d\log T}{T}+2\sqrt{\frac{d}{B}}.
\end{equation*}
\end{corollary}

The stochastic terms in \Cref{main:stochastic_sign_fixed,main:stochastic_norm_fixed} are both of order $\cO(B^{-1/2})$. Their deterministic terms are $2(\log d)^2\log T/T$ and $4d\log T/T$, respectively, retaining their different dependence on $d$ in these upper bounds.
The noise-dependent terms can be reduced by increasing $B$.

We next fix the total sample budget $N=BT$ and optimize over the choice of batch size. This yields the batch-size choices and resulting bounds in \Cref{main:stochastic_sign_sample,main:stochastic_norm_sample}. Here $\widetilde{\cO}$ and $\widetilde{\Theta}$ hide logarithmic factors in $N$. For clarity, we treat the batch size as a continuous variable but rounding $B$ to a feasible integer and taking $T=\lfloor N/B\rfloor$ leaves the results unchanged.
\begin{corollary}[Batch-size choice for SignSGD]\label{main:stochastic_sign_sample}
    Fix a token budget $N=BT$. For all sufficiently large $d$, if $N\ge d$, the upper bound in \Cref{main:stochastic_sign_fixed} is minimized, up to logarithmic factors in $N$, by choosing
\begin{equation*}
    B=\widetilde{\Theta}\!\left(\frac{(dN^2)^{1/3}}{\log d}\right).
\end{equation*}
With this choice,
\begin{equation*}
    \E[f_{\mathrm{uni}}(\theta_T)-f_{\mathrm{uni}}^\star]
    =\widetilde{\cO}\!\left(\frac{d^{1/3}\log d}{N^{1/3}}\right).
\end{equation*}
\end{corollary}
\begin{corollary}[Batch-size choice for normalized SGD]\label{main:stochastic_norm_sample}
    Fix a token budget $N=BT$. For all sufficiently large $d$, if $N\ge4\sqrt d\log N$, the upper bound in \Cref{main:stochastic_norm_fixed} is minimized, up to logarithmic factors in $N$, by choosing
\begin{equation*}
    B=\widetilde{\Theta}\!\left(\frac{N^{2/3}}{d^{1/3}}\right).
\end{equation*}
With this choice,
\begin{equation*}
    \E[f_{\mathrm{uni}}(\theta_T)-f_{\mathrm{uni}}^\star]
    =\widetilde{\cO}\!\left(\frac{d^{2/3}}{N^{1/3}}\right).
\end{equation*}
\end{corollary}

Under a fixed sample budget, the resulting SignSGD bound is smaller than the normalized SGD bound by a factor of order $d^{1/3}/\log d$. The batch size obtained by optimizing the SignSGD upper bound is larger than that obtained for normalized SGD by a factor of $\Theta(d^{2/3}/\log d)$. \Cref{app:stochastic_momentum_nsd} gives analogous rates for the momentum variants and yields the same comparison between upper bounds.

\subsubsection{Comparison with SGD}\label{sec:compare_sgd}
We now compare SignSGD with weight decay to SGD. We begin with the standard fixed-budget SGD guarantee, where the stepsize is tuned in the usual way.

\begin{theorem}[Averaged-iterate SGD]\label{main:sgd_bound}
    Consider optimizing $f_{\mathrm{uni}}$ with initialization $\theta_0=0$, batch size $B\ge 1$, $T\ge 1$ iterations, and total sample budget $N=BT$. Let $\theta_\star$ be the minimum-$\ell_2$-norm minimizer, set $D=\norm{\theta_0-\theta_\star}_2$, and use the constant stepsize $\eta=\min\{1,D\sqrt{B/T}\}$. Then the average iterate $\bar\theta_T = \frac{1}{T}\sum_{t=0}^{T-1}\theta_t$ of SGD satisfies
\begin{equation*}
    \E[f_{\mathrm{uni}}(\bar \theta_T)-f_{\mathrm{uni}}^\star]
    \leq \frac{d}{T} + 2\sqrt{\frac{d}{BT}}
    =
    \frac{dB}{N} + 2\sqrt{\frac{d}{N}} .
\end{equation*}
\end{theorem}
The proof is given in \Cref{app:proof_sgd_bound}. In the fixed-budget form, the only $B$-dependent term is $dB/N$, which is strictly increasing in $B$, while the stochastic term $2\sqrt{d/N}$ is independent of $B$. Hence the bound is minimized at $B=1$, equivalently $T=N$, yielding the familiar $\cO(\sqrt{d/N})$ rate. By contrast, SignSGD with weight decay from \Cref{main:stochastic_sign_sample} scales as $\widetilde{\cO}((d/N)^{1/3})$. Therefore, when the total sample budget $N$ is sufficiently large relative to $d$, a properly tuned SGD already has a smaller upper bound than SignSGD with weight decay by using batch-size-one updates.

The comparison above allows the methods to choose different batch sizes. To isolate the role of the batch size, consider instead using the same $B$ for both methods. When $B$ is small, the stochastic term $4\sqrt{d\log d/B}$ dominates the SignSGD bound and masks its better deterministic dependence on $d$. As $B$ increases, this noise term decreases, and the comparison is eventually governed by the deterministic terms $2B(\log d)^2\log N/N$ for SignSGD and $dB/N$ for SGD; hence the SignSGD upper bound becomes smaller. \Cref{app:sign_sgd_batch_comparison} gives the fixed-batch calculation, including the extra horizon logarithm, under the mild growth condition $(\log d)^2\log N=o(d)$.

Thus, larger batches reduce the noise terms in these bounds, allowing the deterministic advantage to be reflected in the guarantees. This qualitative batch-size effect is also observed in the transformer experiment in \Cref{fig:transformer_batch_noise}.

\begin{figure}[t]
    \centering

    \begin{subfigure}[t]{0.21\textwidth}
        \centering
        \includegraphics[width=\linewidth]{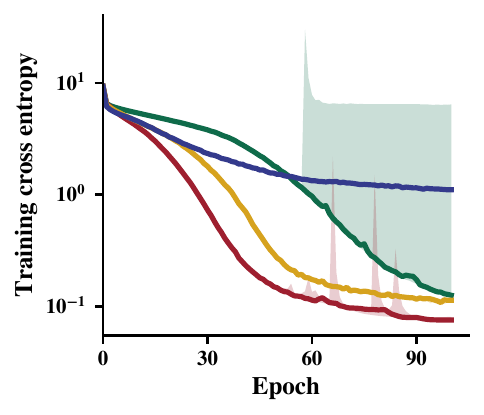}
        \caption{$B=32$}
    \end{subfigure}
    \begin{subfigure}[t]{0.21\textwidth}
        \centering
        \includegraphics[width=\linewidth]{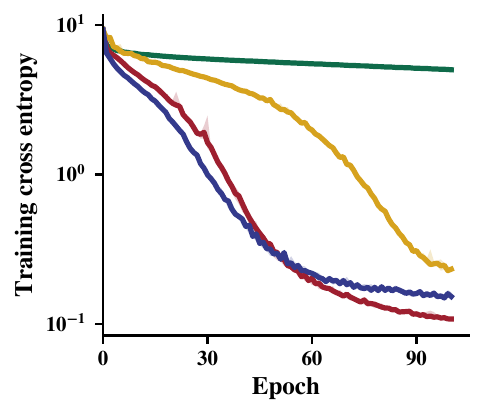}
        \caption{$B=256$}
    \end{subfigure}
    \begin{subfigure}[t]{0.21\textwidth}
        \centering
        \includegraphics[width=\linewidth]{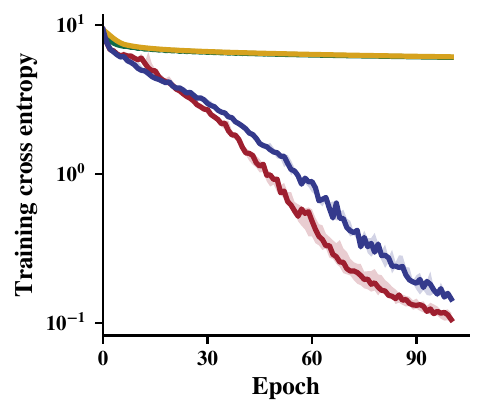}
        \caption{$B=1024$}
    \end{subfigure}

    \vspace{0.35em}
    \includegraphics[width=0.50\textwidth]{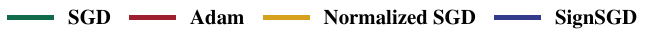}

    \caption{Batch-size comparison for a two-layer transformer trained on TinyPTB. The panels show training cross-entropy for SGD, Adam, normalized SGD, and SignSGD; curves are medians over three seeds and shaded regions show the minimum-to-maximum range.}
    \label{fig:transformer_batch_noise}
\end{figure}

%% file: sections/adam-extension-main.tex
\section{Challenges in Extending the Analysis to Adam}\label{sec:adam}

Although our theory so far focuses on SignGD, Adam exhibits the same qualitative advantage over Euclidean methods in both the deterministic softmax unigram experiment in \Cref{main:figure_softmax} and the transformer experiment in \Cref{fig:transformer-exp-final}. This motivates asking whether existing theory for Adam can explain its performance on the softmax unigram model. One proposed explanation is that the commonly assumed $\ell_2$ smoothness is too coarse to capture the optimization benefits of structured preconditioning \citep{xie_adam_2025, xie_structured_2025}. This line of work generalizes the notion of smoothness to adaptive methods with structured preconditioners, including e.g., Adam, blockwise Adam, and one-sided Shampoo. Each preconditioner structure induces a corresponding adaptive geometry. More specifically, let $\cH$ be a family of positive semidefinite matrices encoding the preconditioner structure.
The corresponding adaptive smoothness is the infimum of the traces of matrices in $\cH$ that dominate the loss Hessian:
\begin{equation*}
    \cL_{\cH}(f)
    :=
    \inf_{\mathbf H \in \cH:\forall x,\nabla^2 f(x)\preceq \mathbf H}
    \tr(\mathbf H).
\end{equation*}
For Adam, the admissible preconditioners are diagonal, so we specialize $\cH$ to the family of diagonal positive semidefinite matrices. This gives the following definition of diagonal adaptive smoothness.

\begin{definition}[Diagonal Adaptive Smoothness]
    The diagonal adaptive smoothness of a function $f$ is defined as the minimal trace of the diagonal matrix that dominates the Hessian:
    \begin{equation}\label{main:adap_diagonal}
    \cL_{\mathrm{diag}}(f) = \min_{\mathrm{diag}\ \mathbf H:\forall x, \nabla^2 f(x) \preceq \mathbf H} \tr(\mathbf H). 
\end{equation}
\end{definition}
The convergence rates in this framework are controlled by the product of the adaptive smoothness and the squared radius of the optimization domain under the norm associated with the preconditioner family \citep{xie_structured_2025}. For Adam, the diagonal preconditioner family gives smoothness $\cL_{\mathrm{diag}}(f)$ and the $\ell_\infty$ domain norm. For GD, the isotropic family gives smoothness $dL_{\norm{\cdot}_2}(f)$ and the normalized Euclidean domain norm $\norm{\cdot}_2/\sqrt d$. Thus, when the domain radii under these two norms are comparable, the theory predicts that Adam is faster than GD when $\cL_{\mathrm{diag}}(f)\ll dL_{\norm{\cdot}_2}(f)$.

We next evaluate the diagonal adaptive smoothness of $f_{\mathrm{uni}}$. The proof of the following lemma is in \Cref{app:proof_unigram_adap}.

\begin{lemma}\label{main:unigram_adap}
    The diagonal adaptive smoothness of the softmax unigram objective $f_{\mathrm{uni}}$ defined in \Cref{main:unigram} satisfies
    $\cL_{\mathrm{diag}}(f_{\mathrm{uni}})\geq d/2$.
\end{lemma}

By \Cref{main:unigram_smoothness_lemma}, $L_{\norm{\cdot}_2}(f_{\mathrm{uni}})=1/2$, so \Cref{main:unigram_adap} implies $\cL_{\mathrm{diag}}(f_{\mathrm{uni}})\geq dL_{\norm{\cdot}_2}(f_{\mathrm{uni}})$. Moreover, \Cref{main:unigram_norm_lemma1} gives $\norm{\theta_\star}_\infty^2=(\log d)^2/4$ for the minimum-$\ell_\infty$-norm solution and $\norm{\theta_\star}_2^2=\Theta(d)$ for the minimum-$\ell_2$-norm solution. Therefore, the quantity controlling the Adam upper bound is at least of order $d(\log d)^2$, while the corresponding quantity for the GD upper bound is of order $d$. In particular, the Adam upper bound has at least the same polynomial dependence on $d$ as the GD upper bound and is not smaller. Thus, the existing adaptive-smoothness theory in \citep{xie_structured_2025} cannot explain the empirical advantage of Adam in \Cref{main:figure_softmax}.

Indeed, this is a key distinction between the softmax unigram model and potentially simpler diagonal-quadratic setups. For instance, consider the simple diagonal quadratic example proposed in \citet{kunstner_heavy-tailed_2024}, $g(\theta) = \sum_{k=1}^d p_k\theta_k^2$, where $p \in \R^d$ is the Zipfian distribution of \Cref{main:zipf}. Its Hessian is $\nabla^2 g(\theta) = 2\diag(p)$, so $\cL_{\text{diag}}(g) = 2$, whereas $dL_{\norm{\cdot}_2}(g)=2d/H_d=\Theta(d/\log d)$. Thus, adaptive smoothness is much smaller than $dL_{\norm{\cdot}_2}(g)$ for large $d$, predicting faster convergence of Adam over SGD for diagonal quadratics with heavy-tailed class imbalance. This separation does not hold for the softmax unigram model, highlighting the difference between our setup and quadratic models. We leave the theoretical explanation of why Adam outperforms GD empirically on this softmax unigram model to future work.

%% file: sections/conclusion.tex
\section{Conclusion and Future Direction}\label{sec:discussion}

We show that the advantage of adaptive and non-Euclidean steepest descent methods over Euclidean methods under heavy-tailed class imbalance persists as transformer-based language modeling is simplified to the softmax unigram model. In this minimal problem, we prove that GD suffers from the tail learning bottleneck: rare-token logits receive tiny updates even though many target tail log-odds are order one. SignGD avoids this magnitude dependence and achieves faster deterministic rates. Our stochastic bounds contain additional noise-dependent terms that can obscure this advantage in the convergence guarantees and can be reduced by increasing the batch size.

Extending the analysis to adaptive methods such as Adam and to settings closer to realistic transformer training, including models with learned embeddings, remains open. Deriving a lower bound for normalized GD on the softmax unigram model is another interesting open problem. Such a bound is challenging because global normalization creates cross-coordinate coupling and can induce overshoot and repeated sign changes, allowing tail ratios to be learned, lost, and relearned. Controlling this trajectory-wide behavior requires new techniques, which we defer to future work. Together, these directions would clarify how broadly the tail-learning mechanism identified here applies, both across optimizers on the softmax unigram model and in more realistic transformer training.

%% file: sections/softmax-unigram-constants.tex
\section{Softmax-Unigram Constants}\label{app:softmax_unigram_constants}

\begin{lemma}[Initial Loss]\label{lemma:initial_loss}
    For the softmax unigram objective initialized at $\theta_0=0$,
    \begin{equation*}
        f_{\mathrm{uni}}(0)-f_{\mathrm{uni}}^\star
        =\KL\left(p\middle\|\frac{\mathbf 1}{d}\right)
        \le \log d.
    \end{equation*}
    For the Zipfian target in \Cref{main:zipf}, more precisely,
    \begin{equation*}
        f_{\mathrm{uni}}(0)-f_{\mathrm{uni}}^\star
        = \frac{1}{2}\log d - \log\log d + \cO(1).
    \end{equation*}
\end{lemma}
\begin{proof}
At $\theta_0=0$, the softmax distribution is uniform:
\[
\softmax(\theta_0)=u:=\frac{1}{d}\mathbf{1}.
\]
Since $f_{\mathrm{uni}}^\star=0$,
\begin{equation*}
    f_{\mathrm{uni}}(0)-f_{\mathrm{uni}}^\star
    =\KL(p\|u)
    =\log d-H(p)
    \le\log d.
\end{equation*}
For the Zipfian target $p_i=1/(iH_d)$, we further have
\begin{align*}
\KL(p\|u) = \sum_{i=1}^d p_i\log(dp_i) =\log d-\log H_d  -\frac{1}{H_d}\sum_{i=1}^d\frac{\log i}{i}.
\end{align*}
The harmonic number satisfies
\[
H_d=\log d+\cO(1),
\qquad
\log H_d=\log\log d+\cO\left(\frac{1}{\log d}\right).
\]
Moreover, by an integral comparison,
\[
\sum_{i=1}^d\frac{\log i}{i}
=
\int_1^d\frac{\log x}{x}\,dx+\cO(1)
=
\frac{1}{2}(\log d)^2+\cO(1).
\]
Therefore,
\[
\frac{1}{H_d}\sum_{i=1}^d\frac{\log i}{i}
=
\frac{1}{2}\log d+\cO(1).
\]
Substituting these estimates gives
\[
f_{\mathrm{uni}}(0)-f_{\mathrm{uni}}^\star
=
\frac{1}{2}\log d-\log\log d+\cO(1),
\]
as claimed.
\end{proof}

\begin{lemma}[Smoothness constants]\label{main:unigram_smoothness_lemma}
    We have
\begin{equation*}
    L_{\norm{\cdot}_2}(f_{\mathrm{uni}}) = \frac{1}{2},
    \qquad
    L_{\norm{\cdot}_\infty}(f_{\mathrm{uni}}) = 1.
\end{equation*}
\end{lemma}
\begin{proof}
For any $\theta\in \R^d$, 
\begin{align*}
    \nabla f_{\mathrm{uni}}(\theta) = \softmax(\theta) - p, \ \text{and} \quad \nabla^2f_{\mathrm{uni}}(\theta) = \diag(\softmax(\theta)) - \softmax(\theta)\softmax(\theta)^T.
\end{align*}
Therefore, for any $u \in \R^d$ with $\norm{u}_2 = 1$,
\begin{align*}
u^\top \nabla^2 f_{\mathrm{uni}}(\theta) u
&=
\sum_{k=1}^{d} [\softmax(\theta)]_k u_k^2 - \left( \sum_{k=1}^d [\softmax(\theta)]_k u_k \right)^2 \\
&= \Var_{I\sim \softmax(\theta)}[u_I].
\end{align*}
By Popoviciu's inequality,
\begin{equation*}
    \Var_{I\sim \softmax(\theta)}[u_I]
    \leq
    \frac{(\max_i u_i-\min_i u_i)^2}{4}.
\end{equation*}
Also, if $\norm{u}_2=1$, then
\begin{equation*}
    \max_i u_i-\min_i u_i \leq \sqrt{2},
\end{equation*}
because the difference of two coordinates is maximized by concentrating the $\ell_2$ mass on those two coordinates with opposite signs. Therefore,
\begin{equation*}
    u^\top \nabla^2 f_{\mathrm{uni}}(\theta) u \leq \frac{(\sqrt 2)^2}{4} = \frac{1}{2},
\end{equation*}
and taking the supremum over $\theta$ and $\norm{u}_2=1$ gives
\begin{equation*}
    L_{\norm{\cdot}_2}(f_{\mathrm{uni}})\leq \frac{1}{2}.
\end{equation*}

Similarly, for $u \in \R^d$ with $\norm{u}_\infty = 1$, 
\begin{align*}
u^\top \nabla^2 f_{\mathrm{uni}}(\theta) u
&\leq \sum_{k=1}^{d} [\softmax(\theta)]_k u_k^2 \\
&\leq \sum_{k=1}^{d} [\softmax(\theta)]_k \abs{u_k}
\le \norm{\softmax(\theta)}_1 \norm{u}_\infty = 1.
\end{align*}
Hence $L_{\norm{\cdot}_\infty}(f_{\mathrm{uni}})\leq 1$.

\subsubsection*{Lower Bounds}
Next, we compute the lower bounds on $L_{\norm{\cdot}_2}(f_{\mathrm{uni}})$ and $L_{\norm{\cdot}_\infty}(f_{\mathrm{uni}})$. Consider $\theta(t): \R \to \R^d$ as follows $\theta(t) = [t \ t \ -t \ \cdots \ \ -t]$. Then define
\begin{align*}
    s(t) := [\softmax(\theta(t))]_1 = [\softmax(\theta(t))]_2 =  \frac{e^t}{2e^t + (d-2)e^{-t}}. 
\end{align*}
For any $r \in \R^2$ with $\norm{r}_2 = 1$ define $u = [r_1 \ r_2 \ 0 \ \cdots \ 0] \in \R^d$. Let
\begin{align*}
    \mA(t) = 
    \begin{bmatrix}
        s(t)(1-s(t)) & -s(t)^2 \\
        -s(t)^2 & s(t)(1 - s(t))
    \end{bmatrix}.
\end{align*}
Then $u^T\nabla^2f_{\mathrm{uni}}(\theta(t))u = r^T\mA(t)r$. Since $r$ is arbitrary,
\begin{align*}
    \sup_{\norm{u}_2 = 1} u^T\nabla^2f_{\mathrm{uni}}(\theta(t))u \geq \sup_{\norm{r}_2 = 1} r^T\mA(t)r.
\end{align*}
The largest eigenvalue of $\mA(t)$ is $s(t)$, and $\lim_{t\to\infty}s(t)=1/2$. Therefore,
\begin{align*}
    L_{\norm{\cdot}_2}(f_{\mathrm{uni}})
    &\geq \lim_{t \to \infty} \sup_{\norm{u}_2 = 1} u^T\nabla^2f_{\mathrm{uni}}(\theta(t))u \\
    &\geq \lim_{t \to \infty} \sup_{\norm{r}_2 = 1} r^T\mA(t)r
    = \frac{1}{2}.
\end{align*}

For the $\ell_\infty$ lower bound, take $r\in\R^2$ with $\norm{r}_\infty=1$. Then again $u^\top\nabla^2f_{\mathrm{uni}}(\theta(t))u = r^\top\mA(t)r$, and the supremum is achieved at $\bar r=[1,-1]$, which gives $\bar r^\top \mA(t)\bar r = 2s(t)$. Hence
\begin{align*}
    L_{\norm{\cdot}_\infty}(f_{\mathrm{uni}})
    &\geq \lim_{t \to \infty} \sup_{\norm{u}_\infty = 1} u^T\nabla^2f_{\mathrm{uni}}(\theta(t))u \\
    &\geq \lim_{t \to \infty} \sup_{\norm{r}_\infty = 1} r^T\mA(t)r
    = 1.
\end{align*}
Together with the upper bounds above, this proves
$L_{\norm{\cdot}_2}(f_{\mathrm{uni}})=1/2$ and $L_{\norm{\cdot}_\infty}(f_{\mathrm{uni}})=1$.
\end{proof}

\begin{lemma}[Minimal-norm minimizers]\label{main:unigram_norm_lemma1}\label{main:unigram_norm_lemma2}
    Among all minimizers $\theta_\star \in \argmin f_{\mathrm{uni}}$, we have
\begin{align*}
    \min_{\theta_\star \in \argmin f_{\mathrm{uni}}}\norm{\theta_\star}_\infty
    &= \frac{1}{2}\log d, \\
    \min_{\theta_\star \in \argmin f_{\mathrm{uni}}}\norm{\theta_\star}_2
    &= \sqrt{d\Var_{k \sim \Unif[d]}[\log k]}.
\end{align*}
    Moreover, for every $d\geq 2$,
\begin{equation*}
    \Var_{k\sim\Unif[d]}[\log k] \leq 1 - \left(\frac{1+\log d}{d}\right)^2 < 1,
    \qquad\text{so}\qquad
    \min_{\theta_\star \in \argmin f_{\mathrm{uni}}}\norm{\theta_\star}_2\leq \sqrt d.
\end{equation*}
As $d\to\infty$, we also have
\begin{equation*}
    \Var_{k\sim\Unif[d]}[\log k]\to 1.
\end{equation*}
\end{lemma}
\begin{proof}
Of course, for a given $\theta \in \argmin f_{\mathrm{uni}}$ we must have that $\softmax(\theta) = p$. We note that the softmax map is invariant up to a constant shift. Since $p_k = k^{-1} / \sum_i i^{-1}$, the set of optimal solutions is described by
\begin{equation*}
    \theta_k = -\log(k) + c, \qquad k \in [d],
\end{equation*}
where $c \in \R$ is a fixed constant.

To minimize $\norm{\theta_\star}_\infty$, observe that $\norm{\theta_\star}_\infty = \max\{\abs{c}, \abs{-\log(d) + c}\}$, so the optimum is attained at the midpoint $c = \frac{\log(d)}{2}$. Thus,
\begin{equation*}
    \min_{\theta_\star \in \argmin f_{\mathrm{uni}}}\norm{\theta_\star}_\infty = \frac{\log(d)}{2}.
\end{equation*}

To minimize $\norm{\theta_\star}_2$, define
\begin{equation*}
    \gamma(c) = \sum_{k = 1}^d (-\log(k) + c)^2 = \norm{\theta_\star}_2^2.
\end{equation*}
Then
\begin{align*}
    \gamma'(c) = 2\sum_{k=1}^d (c - \log(k)), \qquad \gamma''(c) = 2d. 
\end{align*}
Solving $\gamma'(c)=0$ gives the optimal shift $c = \frac{1}{d}\sum_{k=1}^d \log(k)$. Since $\gamma''(c) > 0$, this is indeed a minimum. Therefore
\begin{align*}
    \min_{\theta_\star \in \argmin f_{\mathrm{uni}}} \norm{\theta_\star}_2
    &= \sqrt{\sum_{k=1}^d \left(- \log(k) + \frac{1}{d}\sum_{j=1}^d \log(j)\right)^2} \\
    &= \sqrt{d\Var_{k \sim \Unif[d]}[\log(k)]}.
\end{align*}

    Let $K_d\sim\Unif[d]$ and define
    \[
        X_d:=\log\left(\frac{K_d}{d}\right), \qquad
        \mu_d:=\E[X_d], \qquad
        \nu_d:=\E[X_d^2].
    \]
    Since variance is invariant under constant shifts,
    \[
        \Var[\log K_d]=\Var[X_d]=\nu_d-\mu_d^2.
    \]
    Because $x\mapsto \log x$ is increasing on $(0,1]$, the right- and left-Riemann-sum bounds give
    \begin{equation*}
        -1
        =
        \int_0^1\log x\,dx
        \leq
        \mu_d
        \leq
        \int_{1/d}^1\log x\,dx
        =
        -1+\frac{1+\log d}{d}.
    \end{equation*}
    Since $(1+\log d)/d<1$ for $d\geq 2$, we have $\mu_d<0$, and therefore
    \begin{equation*}
        \mu_d^2\geq \left(1-\frac{1+\log d}{d}\right)^2.
    \end{equation*}
    Also, $x\mapsto \log^2 x$ is decreasing on $(0,1]$, so
    \begin{align*}
        \nu_d
        &= \frac{\log^2 d}{d} + \frac{1}{d}\sum_{k=2}^d \log^2\left(\frac{k}{d}\right) \\
        &\leq \frac{\log^2 d}{d} + \int_{1/d}^1 \log^2 x\,dx \\
        &= \frac{\log^2 d}{d} + 2 - \frac{\log^2 d + 2\log d + 2}{d} \\
        &= 2 - \frac{2(1+\log d)}{d}.
    \end{align*}
    Consequently,
    \begin{align*}
        \Var_{k\sim\Unif[d]}[\log k]
        &= \nu_d-\mu_d^2 \\
        &\leq 2 - \frac{2(1+\log d)}{d} - \left(1-\frac{1+\log d}{d}\right)^2 \\
        &= 1-\left(\frac{1+\log d}{d}\right)^2
        < 1.
    \end{align*}
    This proves the explicit upper bound and the estimate $\min_{\theta_\star \in \argmin f_{\mathrm{uni}}}\norm{\theta_\star}_2\leq \sqrt d$.

    Finally, since $\log x$ and $\log^2 x$ are improper Riemann integrable on $(0,1]$, we also have
    \begin{equation*}
        \mu_d=\frac{1}{d}\sum_{k=1}^d \log\left(\frac{k}{d}\right)\to -1,
        \qquad
        \nu_d=\frac{1}{d}\sum_{k=1}^d \log^2\left(\frac{k}{d}\right)\to 2.
    \end{equation*}
    Therefore,
    \[
        \Var_{k\sim\Unif[d]}[\log k]=\nu_d-\mu_d^2 \to 2-(-1)^2=1,
    \]
    which proves the asymptotic claim.
\end{proof}

\begin{lemma}[Stochastic-gradient error]\label{main:unigram_noise_lemma}\label{main:unigram_noise_lemma2}
    For $f_{\mathrm{uni}}$ and the mini-batch objective $\widehat f_{\mathrm{uni},B}$ with batch size $B$, we have
\begin{align*}
     \E\left[\norm{\nabla \widehat f_{\mathrm{uni},B}(\theta;\xi_{1:B}) -  \nabla f_{\mathrm{uni}}(\theta)}_2\right] &\leq \sqrt{\frac{1}{B}}, \\
    \E\left[\norm{\nabla \widehat f_{\mathrm{uni},B}(\theta;\xi_{1:B}) -  \nabla f_{\mathrm{uni}}(\theta)}_1\right]
    &\leq \min\left\{2,\sqrt{\frac{1}{B}}\norm{\sqrt{p}}_1\right\},
\end{align*}
    where the $\sqrt{\cdot}$ operator is applied element-wise. Moreover, for large $d$,
\begin{equation*}
    \norm{\sqrt{p}}_1 \sim 2\sqrt{\frac{d}{\log d}}.
\end{equation*}
\end{lemma}
\begin{proof}
First, we compute the stochastic gradient of the softmax unigram model with batch size $1$:
\begin{equation*}
    \widehat f_{\mathrm{uni},1}(\theta;\xi)
    = -\log [\softmax(\theta)]_{\xi}-H(p)
    = \log\left(\sum_{i=1}^d e^{\theta_i} \right) - \theta_\xi-H(p),
    \qquad \xi \sim \mathrm{Cat}(p).
\end{equation*}
Then
\begin{equation*}
    \frac{\partial \widehat f_{\mathrm{uni},1}}{\partial \theta_j} =
    \begin{cases}
        [\softmax(\theta)]_j - 1 & j = \xi, \\
        [\softmax(\theta)]_j & j \neq \xi. 
    \end{cases}
\end{equation*}
Therefore, $\nabla \widehat f_{\mathrm{uni},1}(\theta;\xi) = \softmax(\theta) - e_\xi$ where $e_i$ is the $i$-th one-hot vector. The stochastic gradient for batch size $B$ is
\begin{equation*}
    \nabla \widehat f_{\mathrm{uni},B}(\theta;\xi_1,\ldots,\xi_B) = \softmax(\theta) - \hat p,
\end{equation*}
where $\hat p = \frac{1}{B}\sum_{b=1}^B e_{\xi_b}$. By definition, $\hat p \sim \frac{1}{B}\text{Multinomial}(B, p)$.

Now, compute the covariance of the stochastic gradient. Since $\nabla \widehat f_{\mathrm{uni},B}(\theta;\xi_{1:B}) - \nabla f_{\mathrm{uni}}(\theta) = p - \hat p$,
\begin{equation*}
    \text{Cov}(\hat p) = \frac{1}{B}\left(\diag(p) - pp^\top  \right),
\end{equation*}
which follows from the covariance of the multinomial distribution. By Jensen's inequality,
\begin{align*}
    \E[\norm{\hat p - p}_2]
    \leq \sqrt{\E[\norm{\hat p - p}_2^2]}
    =  \sqrt{\tr(\text{Cov}(\hat{p}))}
    = \sqrt{\frac{1 - \norm{p}_2^2}{B}}
    \leq \sqrt{\frac{1}{B}}.
\end{align*}

For the $\ell_1$ noise bound, Jensen again gives
\begin{equation*}
    \E[\norm{\hat p - p}_1] = \sum_{i=1}^d \E[\abs{\hat p_i - p_i}] \leq \sum_{i=1}^d \sqrt{\E[(\hat p_i - p_i)^2]}.
\end{equation*}
Let $K_i \sim \text{Bin}(B,p_i)$ denote the $i$th count in the multinomial draw. Then
\begin{equation*}
    \E[(\hat p_i - p_i)^2]
    = \frac{1}{B^2}\Var[K_i]
    = \frac{p_i(1-p_i)}{B}.
\end{equation*}
Therefore,
\begin{equation*}
    \E[\norm{\hat p - p}_1] \leq \sum_{i=1}^d \sqrt{\frac{p_i(1-p_i)}{B}} \leq \sqrt{\frac{1}{B}}\sum_{i=1}^d \sqrt{p_i} = \sqrt{\frac{1}{B}}\norm{\sqrt{p}}_1.
\end{equation*}
Since both $\hat p$ and $p$ are probability vectors, we also have
\begin{equation*}
    \norm{\hat p-p}_1\leq \norm{\hat p}_1+\norm{p}_1=2.
\end{equation*}
Combining the two bounds gives the stated minimum upper bound.

For large-$d$, note that
\begin{equation*}
    \norm{\sqrt{p}}_1 = \frac{1}{\sqrt{H_d}}\sum_{i=1}^d \frac{1}{\sqrt{i}},
\end{equation*}
where $H_d = \sum_{i=1}^d \frac{1}{i}$. Since $x\mapsto 1/x$ is decreasing,
\begin{align*}
    \log(d+1) \leq H_d \leq 1 + \log d.
\end{align*}
Likewise, since $x\mapsto x^{-1/2}$ is decreasing,
\begin{align*}
    2(\sqrt{d+1} - 1) \leq \sum_{i=1}^d \frac{1}{\sqrt{i}} \leq 1 + 2\sqrt{d} - 2.
\end{align*}
Combining these bounds shows
\begin{equation*}
    \frac{2(\sqrt{d+1} - 1)}{\sqrt{1 + \log d}} \leq \norm{\sqrt{p}}_1 \leq \frac{2\sqrt{d} - 1}{\sqrt{\log(d+1)}},
\end{equation*}
and hence $\norm{\sqrt{p}}_1 \sim 2\sqrt{d/\log d}$. 
\end{proof}

%% file: sections/proofs-without-weight-decay.tex
\section{Proofs without Weight Decay}\label{app:generic_radius_deterministic}

\begin{algorithm}[!htbp]
\caption{(Normalized) Steepest Descent }\label{alg:NSD}
\begin{algorithmic}[1]
\State \textbf{Input:} Start point $x_0$, learning rate $\eta$
\For{$t=1$ to $\dots$}
  \State Compute descent direction $u_t \gets \argmax_{\norm{y} \leq 1} \langle \nabla f(x_{t-1}), y\rangle$
  \State $x_t \gets 
  \begin{cases}
    x_{t-1} - \eta\norm{\nabla f(x_{t-1})}_\star u_t & \text{Standard} \\
    x_{t-1} - \eta u_t & \text{Normalized}
  \end{cases}
  $
\EndFor
\end{algorithmic}
\end{algorithm}

In this section, we prove convergence bounds for GD, SignGD, and unnormalized SignGD, all without weight decay. GD and unnormalized SignGD are instances of the standard steepest descent branch of \Cref{alg:NSD} with respect to the $\ell_2$ and $\ell_\infty$ norms, respectively, while fixed-step SignGD is an instance of the normalized branch. We first provide generic convergence proofs for normalized steepest descent in \Cref{app:generic_fixed_nsd_bound} and standard steepest descent in \Cref{app:generic_rescaled_nsd_bound}. The generic part uses only convexity, smoothness, and a certificate for the trajectory radius, where $\mathcal X^\star:=\argmin_{x\in\R^d}f(x)\neq\emptyset$:
\begin{equation*}
    R_T:=\sup_{0\le t\le T-1}
    \operatorname{dist}_{\norm{\cdot}}(x_t,\mathcal X^\star).
\end{equation*}
Equivalently, if the same certificate holds for the whole trajectory, one may take $R_{\mathrm{traj}}:=\sup_{t\ge 0}\operatorname{dist}_{\norm{\cdot}}(x_t,\mathcal X^\star)$; the finite-horizon definition above is the quantity needed for a $T$-step bound.
We then certify upper bounds on this quantity for the softmax unigram objective.

\begin{lemma}[Generic fixed-step NSD bound]\label{app:generic_fixed_nsd_bound}
Suppose $f:\R^d\to\R$ is convex and $L$-smooth with respect to $\norm{\cdot}$, and let $\mathcal X^\star:=\argmin_{x\in\R^d}f(x)\neq\emptyset$. Suppose we run normalized steepest descent with respect to $\norm{\cdot}$ with a constant learning rate $\eta$. Fix a horizon $T$, and let
\begin{equation*}
    R_T:=\sup_{0\le t\le T-1}
    \operatorname{dist}_{\norm{\cdot}}(x_t,\mathcal X^\star).
\end{equation*}
If $R\ge R_T$ and $0<\eta\le R$, then the iterates satisfy
\begin{equation*}
    f(x_T)-f^\star
    \le
    \left(1-\frac{\eta}{R}\right)^T(f(x_0)-f^\star)
    +\frac{L\eta R}{2}.
\end{equation*}
\end{lemma}
\begin{proof}
Let $G_t:=\norm{\nabla f(x_t)}_\star$.
By the definition of $u_t$, $\ip{\nabla f(x_t)}{u_t}=G_t$. Smoothness gives
\begin{align*}
    f(x_{t+1})
    &\le f(x_t)+\ip{\nabla f(x_t)}{x_{t+1}-x_t}
    +\frac{L}{2}\norm{x_{t+1}-x_t}^2 \\
    &\le f(x_t)-\eta G_t+\frac{L\eta^2}{2}.
\end{align*}
Let $\delta_t:=f(x_t)-f^\star$. Since $R\ge R_T$, for every $t=0,\ldots,T-1$ we can choose $x_t^\star\in\mathcal X^\star$ with $\norm{x_t-x_t^\star}\le R$. By convexity and H\"older's inequality,
\begin{equation*}
    \delta_t
    \le
    \ip{\nabla f(x_t)}{x_t-x_t^\star}
    \le
    \norm{\nabla f(x_t)}_\star\norm{x_t-x_t^\star}
    \le
    RG_t.
\end{equation*}
Thus $G_t\ge \delta_t/R$, and
\begin{equation*}
    \delta_{t+1}
    \le \left(1-\frac{\eta}{R}\right)\delta_t+\frac{L\eta^2}{2}.
\end{equation*}
Unrolling the recursion and using $0<\eta\le R$ gives
\begin{align*}
    \delta_T
    &\le \left(1-\frac{\eta}{R}\right)^T\delta_0
    +\frac{L\eta^2}{2}
    \sum_{s=0}^{T-1}\left(1-\frac{\eta}{R}\right)^s \\
    &\le \left(1-\frac{\eta}{R}\right)^T\delta_0+\frac{L\eta R}{2}.
\end{align*}
\end{proof}

\begin{lemma}[Generic steepest descent bound]\label{app:generic_rescaled_nsd_bound}
Suppose $f:\R^d\to\R$ is convex and $L$-smooth with respect to $\norm{\cdot}$, and let $\mathcal X^\star:=\argmin_{x\in\R^d}f(x)\neq\emptyset$. Suppose we run the standard steepest descent branch of \Cref{alg:NSD} with the learning rate $\eta=1/L$. Fix a horizon $T$, and let
\begin{equation*}
    R_T:=\sup_{0\le t\le T-1}
    \operatorname{dist}_{\norm{\cdot}}(x_t,\mathcal X^\star).
\end{equation*}
If $R\ge R_T$, then the iterates satisfy
\begin{equation*}
    f(x_T)-f^\star
    \le
    \frac{2 L R^2}{T}.
\end{equation*}
\end{lemma}
\begin{proof}
Let $G_t:=\norm{\nabla f(x_t)}_\star$. If $G_t=0$, then $x_t$ is optimal and the claim is trivial. Otherwise, by the definition of $u_t$, $\ip{\nabla f(x_t)}{u_t}=G_t$ and $\norm{u_t}\le 1$.
Smoothness and $\ip{\nabla f(x_t)}{u_t}=G_t$ give
\begin{align*}
    f(x_{t+1})
    &\le f(x_t)-\frac{1}{L}G_t^2+\frac{L}{2}\frac{1}{L^2}G_t^2 \\
    &= f(x_t)-\frac{1}{2L}G_t^2.
\end{align*}
With $\delta_t:=f(x_t)-f^\star$, the trajectory-radius bound gives the needed gradient lower bound. Since $R\ge R_T$, for every $t=0,\ldots,T-1$ we can choose $x_t^\star\in\mathcal X^\star$ with $\norm{x_t-x_t^\star}\le R$. By convexity and H\"older's inequality,
\begin{equation*}
    \delta_t
    \le
    \ip{\nabla f(x_t)}{x_t-x_t^\star}
    \le
    \norm{\nabla f(x_t)}_\star\norm{x_t-x_t^\star}
    \le
    RG_t.
\end{equation*}
Thus $G_t\ge \delta_t/R$, hence
\begin{equation*}
    \delta_{t+1}\le \delta_t-\frac{1}{2 L R^2}\delta_t^2.
\end{equation*}
If $\delta_t=0$ the claim is trivial. Otherwise $\delta_{t+1}\le\delta_t$ and
\begin{equation*}
    \frac{1}{\delta_{t+1}}-\frac{1}{\delta_t}
    =
    \frac{\delta_t-\delta_{t+1}}{\delta_t\delta_{t+1}}
    \ge
    \frac{\delta_t-\delta_{t+1}}{\delta_t^2}
    \ge
    \frac{1}{2 L R^2}.
\end{equation*}
Summing over $t=0,\ldots,T-1$ proves the result.
\end{proof}

\begin{theorem}[GD as Euclidean steepest descent]\label{app:generic_gd_as_rescaled_ngd}
Consider optimizing $f_{\mathrm{uni}}$ using full-batch GD initialized at $\theta_0=0$ with constant step size $\eta=2$. Then, for every $T\geq 1$,
\begin{equation*}
    f_{\mathrm{uni}}(\theta_T)-f_{\mathrm{uni}}^\star
    \leq
    \frac{d}{T}.
\end{equation*}
\end{theorem}
\begin{proof}
We instantiate the generic steepest descent bound in \Cref{app:generic_rescaled_nsd_bound}. The main task is to determine the radius $R$ for the iterates of GD on the softmax unigram model. 

Let $\theta^\star\in\argmin f_{\mathrm{uni}}$ be an arbitrary minimizer. Since $f_{\mathrm{uni}}$ is convex and $L_2$-smooth for any valid Euclidean smoothness constant $L_2$, cocoercivity gives
\begin{equation*}
    \ip{\nabla f_{\mathrm{uni}}(\theta_t)}{\theta_t-\theta^\star}
    \ge
    \frac{1}{L_2}\norm{\nabla f_{\mathrm{uni}}(\theta_t)}_2^2.
\end{equation*}
Therefore
\begin{align*}
    \norm{\theta_{t+1}-\theta^\star}_2^2
    &=
    \norm{\theta_t-\theta^\star-\eta\nabla f_{\mathrm{uni}}(\theta_t)}_2^2 \\
    &\le
    \norm{\theta_t-\theta^\star}_2^2
    -\eta\left(\frac{2}{L_2}-\eta\right)
    \norm{\nabla f_{\mathrm{uni}}(\theta_t)}_2^2.
\end{align*}
Thus, if $0<\eta\le 2/L_2$, the distance to every minimizer is nonincreasing. Let $\theta_2^\star$ be a minimum-$\ell_2$-norm minimizer, and define the finite-horizon Euclidean trajectory radius
\begin{equation*}
    R_T^{(2)}
    :=
    \sup_{0\le t\le T-1}
    \operatorname{dist}_2(\theta_t,\argmin f_{\mathrm{uni}}).
\end{equation*}
We now upper bound $R_T^{(2)}$ without interchanging the supremum over $t$ and the infimum inside the distance-to-set definition. For each fixed $t$, the distance to the set is at most the distance to any particular minimizer. In particular,
\begin{equation*}
    \operatorname{dist}_2(\theta_t,\argmin f_{\mathrm{uni}})
    =
    \inf_{\theta^\star\in\argmin f_{\mathrm{uni}}}
    \norm{\theta_t-\theta^\star}_2
    \le
    \norm{\theta_t-\theta_2^\star}_2.
\end{equation*}
Taking the supremum over $t=0,\ldots,T-1$ only after this pointwise inequality gives
\begin{equation*}
    R_T^{(2)}
    \le
    \sup_{0\le t\le T-1}\norm{\theta_t-\theta_2^\star}_2.
\end{equation*}
Since $\theta_0=0$ and the distance to the fixed minimizer $\theta_2^\star$ is nonincreasing,
\begin{equation*}
    \sup_{0\le t\le T-1}\norm{\theta_t-\theta_2^\star}_2
    \le
    \norm{\theta_0-\theta_2^\star}_2
    =
    \norm{\theta_2^\star}_2
    =:
    R_2.
\end{equation*}
By \Cref{main:unigram_smoothness_lemma}, the exact Euclidean smoothness constant is $L_{\norm{\cdot}_2}(f_{\mathrm{uni}})=1/2$. By \Cref{main:unigram_norm_lemma1},
\begin{equation*}
    R_2^2=d\Var_{k\sim\Unif[d]}[\log k]\le d.
\end{equation*}
Applying \Cref{app:generic_rescaled_nsd_bound} with $L=1/2$, hence with the standard steepest descent learning rate $\eta=1/L=2$, gives
\begin{equation*}
    f_{\mathrm{uni}}(\theta_T)-f_{\mathrm{uni}}^\star
    \le
    \frac{R_2^2}{T}
    \le
    \frac{d}{T}.
\end{equation*}
This matches the direct smooth-convex proof of \Cref{main:gd_upper_bound}.
\end{proof}

\begin{lemma}[Softmax residual radius for SignGD]\label{app:generic_sign_radius}
For a fixed iterate $\theta_t$, let
\begin{equation*}
    q_t:=\softmax(\theta_t),
    \qquad
    r_{t,i}:=\log\frac{q_{t,i}}{p_i}.
\end{equation*}
We write
\begin{equation*}
    M_t:=\max_i r_{t,i},
    \qquad
    \ell_t:=\min_i r_{t,i},
    \qquad
    V_t:=M_t-\ell_t.
\end{equation*}
Then
\begin{equation*}
    \operatorname{dist}_\infty(\theta_t,\argmin f_{\mathrm{uni}})
    =
    \frac{V_t}{2}.
\end{equation*}
Moreover, if $\theta_0=0$, then $V_0=\log d$.
\end{lemma}
\begin{proof}
Since $q_{t,i}-p_i=p_i(e^{r_{t,i}}-1)$, the signs of $q_{t,i}-p_i$ and $r_{t,i}$ agree coordinate-wise.

The minimizer set is
\begin{equation*}
    \mathcal X^\star=\argmin f_{\mathrm{uni}}
    =
    \{\log p+c\mathbf 1:c\in\R\}.
\end{equation*}
For the current iterate $\theta_t$, define the particular shifted minimizer
\begin{equation*}
    \bar\theta_t:=\log p+c_t\mathbf 1,
    \qquad
    c_t:=\log\left(\sum_j e^{\theta_{t,j}}\right)+\frac{M_t+\ell_t}{2}.
\end{equation*}
The notation $\bar\theta_t$ emphasizes that this optimizer may depend on $t$: it is not a fixed optimizer used for all times, but the shift of the minimizer set that is closest to the current iterate in $\ell_\infty$ norm. Indeed, using
\begin{equation*}
    r_{t,i}
    =
    \theta_{t,i}-\log p_i-\log\left(\sum_j e^{\theta_{t,j}}\right),
\end{equation*}
we get
\begin{equation*}
    \theta_{t,i}-\bar\theta_{t,i}
    =
    r_{t,i}-\frac{M_t+\ell_t}{2}.
\end{equation*}
Therefore the coordinates of $\theta_t-\bar\theta_t$ lie in the interval
\begin{equation*}
    \left[\frac{\ell_t-M_t}{2},\frac{M_t-\ell_t}{2}\right],
\end{equation*}
and so
\begin{equation*}
    \norm{\theta_t-\bar\theta_t}_\infty=\frac{V_t}{2}.
\end{equation*}
Conversely, fix any shift $c\in\R$ and write
\begin{equation*}
    z_i(c):=\theta_{t,i}-(\log p_i+c).
\end{equation*}
Using the definition of $r_{t,i}$, this can be rewritten as
\begin{equation*}
    z_i(c)
    =
    r_{t,i}
    +
    \log\left(\sum_j e^{\theta_{t,j}}\right)-c.
\end{equation*}
The last two terms are independent of $i$. Therefore changing the shift $c$ only adds the same scalar to every coordinate, so it does not change the coordinate span:
\begin{equation*}
    \max_i z_i(c)-\min_i z_i(c)
    =
    \max_i r_{t,i}-\min_i r_{t,i}
    =
    V_t.
\end{equation*}
For any vector $z$, its coordinate span is bounded by twice its $\ell_\infty$ norm,
\begin{equation*}
    V_t = \max_i z_i-\min_i z_i
    \le
    \max_i |z_i|+\max_i |z_i|
    =
    2\norm{z}_\infty.
\end{equation*}
Thus every shifted minimizer $\log p+c\mathbf 1$ is at $\ell_\infty$ distance at least $V_t/2$ from $\theta_t$. Combined with the explicit choice $\bar\theta_t$, which achieves distance $V_t/2$, this proves
\begin{equation*}
    \operatorname{dist}_\infty(\theta_t,\mathcal X^\star)
    =
    \frac{V_t}{2}.
\end{equation*}
At initialization $\theta_0=0$, $q_0=(1/d)\mathbf 1$. Since $p_i=1/(iH_d)$,
\begin{equation*}
    r_{0,i}=\log\frac{q_{0,i}}{p_i}
    =
    \log\left(\frac{iH_d}{d}\right),
    \qquad
    V_0=\log d.
\end{equation*}
\end{proof}

\begin{theorem}[Fixed-step SignGD without weight decay]\label{app:generic_fixed_sign}
Consider optimizing $f_{\mathrm{uni}}$ with fixed-step SignGD initialized at $\theta_0=0$. For any horizon $T\geq 2$, use the constant stepsize
\begin{equation*}
    \eta=\frac{\log d\log T}{2T}.
\end{equation*}
Then
\begin{equation*}
    f_{\mathrm{uni}}(\theta_T)-f_{\mathrm{uni}}^\star
    \leq
    \frac{\log d}{T}
    +\frac{(\log d)^2\log T}{8T}.
\end{equation*}
\end{theorem}
\begin{proof}

Consider fixed-step SignGD
\begin{equation*}
    \theta_{t+1}
    =
    \theta_t-\eta\sign(\nabla f_{\mathrm{uni}}(\theta_t)).
\end{equation*}
The residual update has the form
\begin{equation*}
    r_{t+1,i}
    =
    r_{t,i}-\eta\sign(r_{t,i})+b_t,
\end{equation*}
where $b_t$ is a scalar independent of $i$. Thus $b_t$ cancels when taking spans. Since $\ell_t\le 0\le M_t$, the transformed values
\begin{equation*}
    r_{t,i}-\eta\sign(r_{t,i})
\end{equation*}
have span at most $\max\{V_t,2\eta\}$. Therefore
\begin{equation*}
    V_{t+1}\le \max\{V_t,2\eta\}.
\end{equation*}
If $0<\eta\le \log d/2$, then induction and $V_0=\log d$ (because we initialized $\theta_0 = 0$) imply $V_t\le\log d$ for all $t$. Thus the finite-horizon $\ell_\infty$ trajectory radius
\begin{equation*}
    R_T^{(\infty)}
    :=
    \sup_{0\le t\le T-1}
    \operatorname{dist}_\infty(\theta_t,\mathcal X^\star)
    \le
    \frac{\log d}{2}.
\end{equation*}
Applying \Cref{app:generic_fixed_nsd_bound} with $\norm{\cdot}=\norm{\cdot}_\infty$, $L=L_{\norm{\cdot}_\infty}(f_{\mathrm{uni}})=1$, and $R=\log d/2$ gives
\begin{equation*}
    f_{\mathrm{uni}}(\theta_T)-f_{\mathrm{uni}}^\star
    \le
    \left(1-\frac{2\eta}{\log d}\right)^T
    \left(f_{\mathrm{uni}}(\theta_0)-f_{\mathrm{uni}}^\star\right)
    +\frac{\eta\log d}{4}.
\end{equation*}

For a prescribed horizon $T\geq 2$, choose the constant stepsize
\begin{equation*}
    \eta=\frac{\log d\log T}{2T}.
\end{equation*}
This choice satisfies $0<\eta\leq \log d/2$. Moreover,
\begin{equation*}
    \left(1-\frac{2\eta}{\log d}\right)^T
    \leq
    \exp\left(-\frac{2\eta T}{\log d}\right)
    =\frac{1}{T}.
\end{equation*}
By \Cref{lemma:initial_loss}, $f_{\mathrm{uni}}(\theta_0)-f_{\mathrm{uni}}^\star\le\log d$, so
\begin{equation*}
    f_{\mathrm{uni}}(\theta_T)-f_{\mathrm{uni}}^\star
    \leq
    \frac{\log d}{T}
    +\frac{(\log d)^2\log T}{8T}.
\end{equation*}
Thus, a horizon-tuned fixed stepsize preserves the polylogarithmic dependence on $d$, with only an additional $\log T$ factor compared with the $\cO((\log d)^2/T)$ rate below for unnormalized SignGD.
\end{proof}

\begin{theorem}[Unnormalized SignGD without weight decay]\label{app:generic_rescaled_sign}
Consider optimizing $f_{\mathrm{uni}}$ from $\theta_0=0$ using the update
\begin{equation*}
    \theta_{t+1}
    =
    \theta_t-\eta\norm{\nabla f_{\mathrm{uni}}(\theta_t)}_1
    \sign(\nabla f_{\mathrm{uni}}(\theta_t))
\end{equation*}
with constant learning rate $\eta=1/4$. Then, for every $T\geq 1$,
\begin{equation*}
    f_{\mathrm{uni}}(\theta_T)-f_{\mathrm{uni}}^\star
    \leq
    \frac{8(\log d)^2}{7T}.
\end{equation*}
\end{theorem}
\begin{proof}
This is rescaled steepest descent in the $\ell_\infty$ norm. We first recall the softmax-specific stepsize control. Since $\sum_i(q_{t,i}-p_i)=0$,
\begin{equation*}
    \frac{1}{2}\norm{\nabla f_{\mathrm{uni}}(\theta_t)}_1
    =
    \sum_{i:r_{t,i}>0}(q_{t,i}-p_i).
\end{equation*}
For $r_{t,i}>0$,
\begin{equation*}
    q_{t,i}-p_i=q_{t,i}(1-e^{-r_{t,i}})\le q_{t,i}r_{t,i}\le q_{t,i}M_t,
\end{equation*}
so $\norm{\nabla f_{\mathrm{uni}}(\theta_t)}_1/2\le M_t$. Applying the same argument to the negative residuals gives
\begin{equation*}
    \frac{1}{2}\norm{\nabla f_{\mathrm{uni}}(\theta_t)}_1
    \le
    \min\{M_t,-\ell_t\}.
\end{equation*}
If $0<\eta\le 1/4$, then the actual coordinate step length
\begin{equation*}
    s_t:=\eta\norm{\nabla f_{\mathrm{uni}}(\theta_t)}_1
\end{equation*}
satisfies $s_t\le \frac12\min\{M_t,-\ell_t\}$. The residual-span argument then gives
\begin{equation*}
    V_{t+1}=V_t-2s_t\le V_t.
\end{equation*}
Therefore $V_t\le V_0=\log d$, and the finite-horizon trajectory radius satisfies
\begin{equation*}
    R_T^{(\infty)}
    :=
    \sup_{0\le t\le T-1}
    \operatorname{dist}_\infty(\theta_t,\mathcal X^\star)
    \le
    \frac{\log d}{2}.
\end{equation*}
For rescaled SignGD, the residual-span certificate above requires the smaller learning rate $\eta=1/4$, rather than the generic choice $1/L=1$. With $G_t=\norm{\nabla f_{\mathrm{uni}}(\theta_t)}_1$, smoothness gives
\begin{equation*}
    f_{\mathrm{uni}}(\theta_{t+1})
    \le
    f_{\mathrm{uni}}(\theta_t)
    -\frac{1}{4}G_t^2+\frac{1}{2}\frac{1}{16}G_t^2
    =
    f_{\mathrm{uni}}(\theta_t)-\frac{7}{32}G_t^2.
\end{equation*}
Let $R=\log d/2$ and $\Delta_t:=f_{\mathrm{uni}}(\theta_t)-f_{\mathrm{uni}}^\star$. The trajectory-radius certificate above gives $R_T^{(\infty)}\le R$, so for every $t=0,\ldots,T-1$,
\begin{equation*}
    \operatorname{dist}_\infty(\theta_t,\mathcal X^\star)\le R.
\end{equation*}
Use the shifted minimizer $\bar\theta_t$ from \Cref{app:generic_sign_radius}, which satisfies $\norm{\theta_t-\bar\theta_t}_\infty\le R$. By convexity and H\"older's inequality in the $\ell_\infty/\ell_1$ pair,
\begin{equation*}
    \Delta_t
    \le
    \ip{\nabla f_{\mathrm{uni}}(\theta_t)}{\theta_t-\bar\theta_t}
    \le
    \norm{\nabla f_{\mathrm{uni}}(\theta_t)}_1
    \norm{\theta_t-\bar\theta_t}_\infty
    \le
    RG_t.
\end{equation*}
Therefore $G_t\ge \Delta_t/R$, and
\begin{equation*}
    \Delta_{t+1}
    \le
    \Delta_t-\frac{7}{32R^2}\Delta_t^2.
\end{equation*}
The same reciprocal-sum argument as in \Cref{app:generic_rescaled_nsd_bound} gives
\begin{equation*}
    f_{\mathrm{uni}}(\theta_T)-f_{\mathrm{uni}}^\star
    \le
    \frac{(\log d)^2/4}{(7/32)T}
    =
    \frac{8(\log d)^2}{7T}.
\end{equation*}
\end{proof}


\subsection[Proof of the conditional NSD convergence theorem]{Proof of \Cref{main:nsd_convergence}}\label{app:proof_nsd_convergence}

The proof iterates the following one-step recursion.

\begin{lemma}[One-step radius-to-rate recursion]\label{app:radius_to_rate}
    Suppose $f:\R^d\to\R$ is convex and $L$-smooth with respect to a norm $\norm{\cdot}$, with $\mathcal X^\star:=\argmin f\neq\emptyset$, and let $x_t=x_{t-1}-\eta_tu_t$ be one step of NSD. Write $\delta_t:=f(x_t)-f^\star$. If $\operatorname{dist}_{\norm{\cdot}}(x_{t-1},\mathcal X^\star)\le R$ and $0<\eta_t\le R$, then
    \begin{equation*}
        \delta_{t}
        \le
        \left(1-\frac{\eta_t}{R}\right)\delta_{t-1}
        +\frac{L\eta_t^2}{2}.
    \end{equation*}
\end{lemma}
\begin{proof}
If $\nabla f(x_{t-1})=0$, then $x_{t-1}$ is a minimizer, the iterates remain unchanged, and there is nothing to prove. Otherwise $\norm{u_t}=1$ and $\ip{\nabla f(x_{t-1})}{u_t}=\norm{\nabla f(x_{t-1})}_\star$, so smoothness gives
\begin{equation*}
    f(x_{t})
    \le
    f(x_{t-1})-\eta_t\ip{\nabla f(x_{t-1})}{u_t}+\frac{L\eta_t^2}{2}\norm{u_t}^2
    =
    f(x_{t-1})-\eta_t\norm{\nabla f(x_{t-1})}_\star+\frac{L\eta_t^2}{2}.
\end{equation*}
Since $\mathcal X^\star$ is closed, we may choose $x_{t-1}^\star\in\mathcal X^\star$ with $\norm{x_{t-1}-x_{t-1}^\star}=\operatorname{dist}_{\norm{\cdot}}(x_{t-1},\mathcal X^\star)\le R$. Convexity and H\"older's inequality give
\begin{equation*}
    \delta_{t-1}
    \le
    \ip{\nabla f(x_{t-1})}{x_{t-1}-x_{t-1}^\star}
    \le
    \norm{\nabla f(x_{t-1})}_\star\norm{x_{t-1}-x_{t-1}^\star}
    \le
    R\,\norm{\nabla f(x_{t-1})}_\star,
\end{equation*}
so $\norm{\nabla f(x_{t-1})}_\star\ge\delta_{t-1}/R$. Substituting into the previous display proves the claim.
\end{proof}

The schedule is constant, $\eta_t\equiv\eta=\frac{R\log T}{T}$, and $0<\eta\le R$. Iterating \Cref{app:radius_to_rate} for $t=1,\ldots,T$ gives
\begin{equation*}
    \delta_T
    \le
    \left(1-\frac{\eta}{R}\right)^T\delta_0
    +\frac{L\eta R}{2}.
\end{equation*}
The radius assumption at $t=1$ gives $\operatorname{dist}_{\norm{\cdot}}(x_0,\mathcal X^\star)\le R$. Thus smoothness at a closest minimizer gives $\delta_0\le LR^2/2$. Also,
\begin{equation*}
    \left(1-\frac{\log T}{T}\right)^T\le\frac1T.
\end{equation*}
Therefore,
\begin{equation*}
    \delta_T
    \le \frac{LR^2(1+\log T)}{2T}
    \le \frac{LR^2\log T}{T},
\end{equation*}
where the last inequality uses $T\ge3$. \qed

\subsection[Proof of the SignGD radius certificate]{Proof of \Cref{main:softmax_radius}}\label{app:proof_softmax_radius}

The quotient-distance identity $\operatorname{dist}_\infty(\theta_t,\mathcal X^\star)=V_t/2$ and the initialization value $V_0=\log d$ are proved in \Cref{app:generic_sign_radius}. It remains to prove the span recursion for an arbitrary step size $\eta_t>0$.

Write $Z_t:=\log\sum_je^{\theta_{t,j}}$, $M_t:=\max_ir_{t,i}$, and $m_t:=\min_ir_{t,i}$. Since $q_{t,i}-p_i=p_i(e^{r_{t,i}}-1)$, we have $\sign(q_{t,i}-p_i)=\sign(r_{t,i})$, so the SignGD update $\theta_t=\theta_{t-1}-\eta_t\sign(q_{t-1}-p)$ yields
\begin{equation*}
    r_{t,i}
    =
    r_{t-1,i}-\eta_t\sign(r_{t-1,i})+b_t,
    \qquad
    b_t:=Z_{t-1}-Z_{t},
\end{equation*}
where $b_t$ does not depend on $i$ and therefore cancels in the span. Because $q_{t-1}$ and $p$ are both probability vectors, some coordinate has $q_{t-1,i}\ge p_i$ and some has $q_{t-1,i}\le p_i$, hence $M_{t-1}\ge0\ge m_{t-1}$. The map $r\mapsto r-\eta_t\sign(r)$ sends each positive residual into $(-\eta_t,\,M_{t-1}-\eta_t]$, each negative residual into $[m_{t-1}+\eta_t,\,\eta_t)$, and zero to zero. Every updated residual therefore lies in the interval $\left[\min\{m_{t-1}+\eta_t,-\eta_t\},\,\max\{M_{t-1}-\eta_t,\eta_t\}\right]$, so
\begin{equation*}
    V_{t}
    \le
    \max\{M_{t-1}-\eta_t,\eta_t\}-\min\{m_{t-1}+\eta_t,-\eta_t\}
    \le
    \max\{V_{t-1}-2\eta_t,\;M_{t-1},\;-m_{t-1},\;2\eta_t\}
    \le
    \max\{V_{t-1},2\eta_t\},
\end{equation*}
where the middle inequality checks the four combinations and the last uses $M_{t-1}\le V_{t-1}$ and $-m_{t-1}\le V_{t-1}$, both consequences of $M_{t-1}\ge0\ge m_{t-1}$.

Finally, if $\eta_t\le\frac{\log d}{2}$ for all $t$, then induction from $V_0=\log d$ gives $V_t\le\log d$ for all $t$, and the quotient-distance identity converts this into $\operatorname{dist}_\infty(\theta_t,\mathcal X^\star)\le\frac{\log d}{2}$. \qed

\subsection[Proof of the SignGD upper bound]{Proof of \Cref{main:deterministic_sign}}\label{app:proof_deterministic_sign}

Set $R:=\frac{\log d}{2}$. By \Cref{main:softmax_radius}, the hypotheses of \Cref{main:nsd_convergence} hold with $L=1$ and $\eta=\frac{R\log T}{T}$. Therefore,
\begin{equation*}
    \delta_T
    \le
    \frac{R^2\log T}{T}
    =
    \frac{(\log d)^2\log T}{4T}.
    \tag*{\qed}
\end{equation*}

\subsection[Proof of the normalized-GD radius certificate]{Proof of \Cref{main:normalized_gd_radius}}\label{app:proof_normalized_gd_radius}

Let $\theta^\star$ be a minimizer closest to $\theta_0$ in $\ell_2$. By \Cref{main:unigram_norm_lemma1}, $\norm{\theta_0-\theta^\star}_2^2\le d$. If $\nabla f_{\mathrm{uni}}(\theta_{t-1})=0$, then $\theta_t=\theta_{t-1}$. Otherwise, convexity gives $\ip{\nabla f_{\mathrm{uni}}(\theta_{t-1})}{\theta_{t-1}-\theta^\star}\ge0$, and hence
\begin{equation*}
    \norm{\theta_{t}-\theta^\star}_2^2
    =
    \norm{\theta_{t-1}-\theta^\star}_2^2
    -\frac{2\eta_t\ip{\nabla f_{\mathrm{uni}}(\theta_{t-1})}{\theta_{t-1}-\theta^\star}}{\norm{\nabla f_{\mathrm{uni}}(\theta_{t-1})}_2}
    +\eta_t^2
    \le
    \norm{\theta_{t-1}-\theta^\star}_2^2+\eta_t^2,
\end{equation*}
In either case, iterating the squared-distance inequality gives $\norm{\theta_t-\theta^\star}_2^2\le d+\sum_{s\le t}\eta_s^2$.

With $\eta:=\frac{2\sqrt d\,\log T}{T}$, for every $t\le T$,
\begin{equation*}
    \norm{\theta_t-\theta^\star}_2^2
    \le
    d+T\eta^2
    =
    d\left(1+\frac{4(\log T)^2}{T}\right)
    \le
    4d,
\end{equation*}
since $\sup_{T\ge2}(\log T)^2/T=4/e^2<3/4$. Therefore $\operatorname{dist}_2(\theta_t,\mathcal X^\star)\le2\sqrt d$ for every $0\le t\le T$. \qed

\subsection[Proof of the normalized-GD upper bound]{Proof of \Cref{main:deterministic_normgd}}\label{app:proof_deterministic_normgd}

By \Cref{main:normalized_gd_radius}, the trajectory radius is at most $R:=2\sqrt d$ and $\eta=\frac{R\log T}{T}$. Since \Cref{main:unigram_smoothness_lemma} gives $L=1/2$, \Cref{main:nsd_convergence} yields
\begin{equation*}
    \delta_T
    \le
    \frac{LR^2\log T}{T}
    =
    \frac{2d\log T}{T}.
    \tag*{\qed}
\end{equation*}

%% file: sections/deterministic-rate-proofs.tex
\section{Proofs of Deterministic Rates}

\subsection[Proof of the horizon-tuned NSD-WD theorem]{Proof of \Cref{main:NSD_det}}\label{app:proof_nsd_det}

Let $\delta_t:=f(x_t)-f^\star$. Since $\eta\lambda=\log T/T\le1$ and $\lambda D\ge1$, induction gives $\norm{x_t}\le D$: if $\norm{x_{t-1}}\le D$, then
\begin{equation*}
    \norm{x_t}
    \le (1-\eta\lambda)\norm{x_{t-1}}+\eta
    \le D+\eta(1-\lambda D)
    \le D.
\end{equation*}
Let $u_t$ be the normalized steepest-descent direction. Since $\norm{u_t}\le1$, $\ip{\nabla f(x_{t-1})}{u_t}=\norm{\nabla f(x_{t-1})}_\star$, and $\lambda\norm{x_\star}\le1$, convexity gives
\begin{equation*}
    \ip{\nabla f(x_{t-1})}{u_t+\lambda x_{t-1}}
    \ge \lambda\ip{\nabla f(x_{t-1})}{x_{t-1}-x_\star}
    \ge \lambda\delta_{t-1}.
\end{equation*}
Smoothness and $\norm{u_t+\lambda x_{t-1}}\le1+\lambda D$ therefore imply
\begin{equation*}
    \delta_t
    \le (1-\eta\lambda)\delta_{t-1}
    +\frac{L\eta^2}{2}(1+\lambda D)^2.
\end{equation*}
Iterating this recursion yields
\begin{equation*}
    \delta_T
    \le (1-\eta\lambda)^T\delta_0
    +\frac{L\eta(1+\lambda D)^2}{2\lambda}.
\end{equation*}
Moreover, smoothness at a minimizer and $\norm{x_0-x_\star}\le D+1/\lambda$ give
\begin{equation*}
    \delta_0
    \le \frac{L}{2}\norm{x_0-x_\star}^2
    \le \frac{L(1+\lambda D)^2}{2\lambda^2}.
\end{equation*}
With $\eta=\frac{\log T}{\lambda T}$, we have $(1-\eta\lambda)^T\le T^{-1}$. Hence
\begin{equation*}
    \delta_T
    \le \frac{L(1+\lambda D)^2(1+\log T)}{2\lambda^2T}
    \le \frac{2L(1+\lambda D)^2\log T}{\lambda^2T},
\end{equation*}
where the last inequality holds for $T\ge2$. If $x_0=0$ and $\lambda=1/\norm{x_\star}$, then $D=\norm{x_\star}$ and $1+\lambda D=2$, proving the final claim. \qed

\subsection[Proof of the deterministic SignGD theorem]{Proof of \Cref{main:deterministic_sign_wd}}\label{app:proof_deterministic_sign_wd}
We apply \Cref{main:NSD_det} with respect to $\norm{\cdot}_\infty$. By \Cref{main:unigram_smoothness_lemma}, $L_{\norm{\cdot}_\infty}(f_{\mathrm{uni}})=1$. By \Cref{main:unigram_norm_lemma1}, the smallest $\ell_\infty$-norm minimizer has norm
\begin{equation*}
    R_\infty = \min_{\theta_\star\in\argmin f_{\mathrm{uni}}}\norm{\theta_\star}_\infty=\frac{\log d}{2}.
\end{equation*}
With $\theta_0=0$, choose $\lambda=1/R_\infty=2/\log d$. The theorem prescribes the fixed learning rate
\[
    \eta_t\equiv\frac{\log d\,\log T}{2T}.
\]
Therefore,
\begin{equation*}
    f_{\mathrm{uni}}(\theta_T)-f_{\mathrm{uni}}^\star
    \leq \frac{8L_{\norm{\cdot}_\infty}(f_{\mathrm{uni}})R_\infty^2\log T}{T}
    = \frac{2(\log d)^2\log T}{T}.
\end{equation*}

\subsection[Proof of the deterministic normalized-GD theorem]{Proof of \Cref{main:deterministic_normgd_wd}}\label{app:proof_deterministic_normgd_wd}
Again apply \Cref{main:NSD_det}, now with respect to $\norm{\cdot}_2$. By \Cref{main:unigram_smoothness_lemma},
\begin{equation*}
    L_{\norm{\cdot}_2}(f_{\mathrm{uni}})=\frac{1}{2}.
\end{equation*}
By \Cref{main:unigram_norm_lemma1}, the smallest $\ell_2$-norm minimizer satisfies
\begin{equation*}
    \min_{\theta_\star\in\argmin f_{\mathrm{uni}}}\norm{\theta_\star}_2 \leq \sqrt d,
\end{equation*}
so $\lambda=1/\sqrt d$ is admissible in \Cref{main:NSD_det}. With $\theta_0=0$, we have $D=\sqrt d$, and the theorem prescribes the fixed learning rate
\[
    \eta_t\equiv\frac{\sqrt d\,\log T}{T}.
\]
Therefore,
\begin{equation*}
    f_{\mathrm{uni}}(\theta_T)-f_{\mathrm{uni}}^\star
    \leq \frac{2L_{\norm{\cdot}_2}(f_{\mathrm{uni}})(1+\lambda D)^2\log T}{\lambda^2 T}
    = \frac{4d\log T}{T}.
\end{equation*}
\subsection[Proof of the GD upper bound]{Proof of \Cref{main:gd_upper_bound}}\label{app:proof_gd_upper_bound}
Let $\theta_\star\in\argmin f_{\mathrm{uni}}$ minimize the $\ell_2$ norm, so by \Cref{main:unigram_norm_lemma1},
\begin{equation*}
    \norm{\theta_\star}_2^2 = d\Var_{k\sim\Unif[d]}[\log k].
\end{equation*}
Let $g_t:=\nabla f_{\mathrm{uni}}(\theta_t)$ and run full-batch GD with stepsize $\eta=2$:
\begin{equation*}
    \theta_{t+1}=\theta_t-2g_t.
\end{equation*}
Since \Cref{main:unigram_smoothness_lemma} gives $L_{\norm{\cdot}_2}(f_{\mathrm{uni}})=1/2$, smoothness implies
\begin{equation*}
    f_{\mathrm{uni}}(\theta_{t+1})
    \leq
    f_{\mathrm{uni}}(\theta_t)-\norm{g_t}_2^2.
\end{equation*}
Also,
\begin{align*}
    \norm{\theta_{t+1}-\theta_\star}_2^2
    &=\norm{\theta_t-\theta_\star}_2^2
    -4\ip{g_t}{\theta_t-\theta_\star}
    +4\norm{g_t}_2^2 \\
    &\leq \norm{\theta_t-\theta_\star}_2^2
    -4\bigl(f_{\mathrm{uni}}(\theta_t)-f_{\mathrm{uni}}^\star\bigr)
    +4\bigl(f_{\mathrm{uni}}(\theta_t)-f_{\mathrm{uni}}(\theta_{t+1})\bigr) \\
    &= \norm{\theta_t-\theta_\star}_2^2
    -4\bigl(f_{\mathrm{uni}}(\theta_{t+1})-f_{\mathrm{uni}}^\star\bigr),
\end{align*}
where we used convexity in the second line. Summing from $t=0$ to $T-1$ gives
\begin{equation*}
    4\sum_{t=1}^{T}\bigl(f_{\mathrm{uni}}(\theta_t)-f_{\mathrm{uni}}^\star\bigr)
    \leq
    \norm{\theta_0-\theta_\star}_2^2
    =
    d\Var_{k\sim\Unif[d]}[\log k].
\end{equation*}
Because the objective decreases along GD iterates, we have
\begin{equation*}
    T\bigl(f_{\mathrm{uni}}(\theta_T)-f_{\mathrm{uni}}^\star\bigr)
    \leq
    \sum_{t=1}^{T}\bigl(f_{\mathrm{uni}}(\theta_t)-f_{\mathrm{uni}}^\star\bigr),
\end{equation*}
and therefore
\begin{equation*}
    f_{\mathrm{uni}}(\theta_T)-f_{\mathrm{uni}}^\star
    \leq
    \frac{d\Var_{k\sim\Unif[d]}[\log k]}{4T}.
\end{equation*}
By \Cref{main:unigram_norm_lemma1}, we have $\Var_{k\sim\Unif[d]}[\log k] < 1$, so this implies the simpler main-text bound
\begin{equation*}
    f_{\mathrm{uni}}(\theta_T)-f_{\mathrm{uni}}^\star
    \leq
    \frac{d}{4T}.
\end{equation*}

%% file: sections/gd-lower-bound.tex
\section{Proof of the Deterministic GD Lower Bound}\label{app:gd_lower_bound}

Throughout this section, let $q_t=\softmax(\theta_t)$ and let $H_d=\sum_{j=1}^d 1/j$. Since $\nabla f_{\mathrm{uni}}(\theta)=\softmax(\theta)-p$, full-batch GD with stepsize $\eta$ satisfies
\begin{equation*}
    \theta_{t+1}=\theta_t+\eta(p-q_t).
\end{equation*}

\begin{lemma}[Dyadic log-ratio control]\label{app:gd_log_ratio}
    For $1\leq i\leq \floor{d/2}$, define
    \begin{equation*}
        \omega_{t,i}:=\log\frac{q_{t,i}}{q_{t,2i}}=\theta_{t,i}-\theta_{t,2i}.
    \end{equation*}
    If $\theta_0=0$ and $0<\eta<d/2$, then for every $t\geq 0$ and every $i\leq \floor{d/2}$ satisfying
    \begin{equation*}
        \frac{\eta t}{iH_d}\leq \log 2,
    \end{equation*}
    we have
    \begin{equation*}
        0\leq \omega_{t,i}\leq \frac{\eta t}{2iH_d}.
    \end{equation*}
\end{lemma}
\begin{proof}
    We prove the claim by induction on $t$. The base case holds because $\theta_0=0$.
    Suppose the claim holds at time $t$. Fix an index $i$ satisfying
    \begin{equation*}
        \frac{\eta(t+1)}{iH_d}\leq \log 2.
    \end{equation*}
    Then the induction hypothesis applies to the same index at time $t$, so
    \begin{equation*}
        0\leq \omega_{t,i}\leq \frac{\eta t}{2iH_d}.
    \end{equation*}

    We first bound the mass of the pair $\{i,2i\}$. Since
    \begin{equation*}
        \sum_{k=1}^d(p_k-q_{s,k})=0
    \end{equation*}
    for every $s$, and $\theta_0=0$, we have $\sum_{k=1}^d\theta_{t,k}=0$. Hence, by AM-GM,
    \begin{equation*}
        \sum_{k=1}^d e^{\theta_{t,k}}\geq d.
    \end{equation*}
    Also, for every coordinate $k$,
    \begin{equation*}
        \theta_{t,k}
        =
        \eta\sum_{s=0}^{t-1}(p_k-q_{s,k})
        \leq \eta t p_k
        =
        \frac{\eta t}{kH_d}.
    \end{equation*}
    Therefore
    \begin{equation*}
        \theta_{t,i}\leq \log 2,
        \qquad
        \theta_{t,2i}\leq \frac{\log 2}{2},
    \end{equation*}
    and consequently
    \begin{equation*}
        q_{t,i}+q_{t,2i}
        =
        \frac{e^{\theta_{t,i}}+e^{\theta_{t,2i}}}{\sum_{k=1}^d e^{\theta_{t,k}}}
        \leq
        \frac{2+\sqrt 2}{d}
        \leq
        \frac{4}{d}.
    \end{equation*}

    Write $\omega_t=\omega_{t,i}$. The GD update gives
    \begin{align*}
        \omega_{t+1}
        &=\omega_t+\eta\left((p_i-p_{2i})-(q_{t,i}-q_{t,2i})\right).
    \end{align*}
    Since $\omega_t\geq 0$, we have $q_{t,i}\geq q_{t,2i}$ and
    \begin{align*}
        q_{t,i}-q_{t,2i}
        &=(q_{t,i}+q_{t,2i})
        \frac{q_{t,i}-q_{t,2i}}{q_{t,i}+q_{t,2i}} \\
        &=(q_{t,i}+q_{t,2i})\tanh(\omega_t/2)
        \leq \frac{2}{d}\omega_t,
    \end{align*}
    where we used $\tanh x\leq x$ for $x\geq 0$. Therefore
    \begin{align*}
        \omega_{t+1}
        &\geq \omega_t+\eta(p_i-p_{2i})-\frac{2\eta}{d}\omega_t \\
        &= \left(1-\frac{2\eta}{d}\right)\omega_t+\frac{\eta}{2iH_d}
        \geq 0,
    \end{align*}
    since $\eta<d/2$.

    Since $\omega_t\geq 0$, we also have $q_{t,i}-q_{t,2i}\geq 0$, so
    \begin{align*}
        \omega_{t+1}
        \leq \omega_t+\eta(p_i-p_{2i})
        =\omega_t+\frac{\eta}{2iH_d}.
    \end{align*}
    Combining this with the induction hypothesis gives
    \begin{equation*}
        \omega_{t+1}
        \leq
        \frac{\eta(t+1)}{2iH_d}.
    \end{equation*}
    This proves the claim at time $t+1$, completing the induction.
\end{proof}

\begin{lemma}[KL chain rule]\label{app:kl_chain_rule}
    Let $p,q$ be probability distributions on $[d]$, and let $\cP=\{A_b\}_{b=1}^B$ be a partition of $[d]$ with $q(A_b)>0$ for every $b$. Define $p_\cP(b)=p(A_b)$ and $q_\cP(b)=q(A_b)$, and let $p(\cdot\mid A_b)$ and $q(\cdot\mid A_b)$ be the conditional distributions inside $A_b$. Then
    \begin{equation*}
        \KL(p\|q)
        =
        \KL(p_\cP\|q_\cP)
        +
        \sum_{b=1}^B p(A_b)\KL(p(\cdot\mid A_b)\|q(\cdot\mid A_b)).
    \end{equation*}
    In particular,
    \begin{equation*}
        \KL(p\|q)
        \geq
        \sum_{b=1}^B p(A_b)\KL(p(\cdot\mid A_b)\|q(\cdot\mid A_b)).
\end{equation*}
\end{lemma}
\begin{proof}
    Expanding the KL divergence over the partition gives
    \begin{align*}
        \KL(p\|q)
        &=\sum_{b=1}^B\sum_{j\in A_b}p_j\log\frac{p_j}{q_j} \\
        &=\sum_{b=1}^B p(A_b)\log\frac{p(A_b)}{q(A_b)}
        +\sum_{b=1}^B p(A_b)
        \sum_{j\in A_b}\frac{p_j}{p(A_b)}
        \log\frac{p_j/p(A_b)}{q_j/q(A_b)}.
    \end{align*}
    The first term is $\KL(p_\cP\|q_\cP)$ and the second term is the conditional KL sum.
\end{proof}

\begin{lemma}[Odd harmonic tail]\label{app:odd_harmonic_tail}
    For integers $1\leq M\leq n$,
    \begin{equation*}
        \sum_{\substack{M\leq i\leq n\\ i \text{ odd}}}\frac{1}{i}
        \geq
        \frac{1}{2}\log\left(\frac{n+1}{M+1}\right).
\end{equation*}
\end{lemma}
\begin{proof}
    If the interval contains no odd integer, then necessarily $M=n$ is even and the right-hand side is zero. Otherwise, let $m$ be the first odd integer at least $M$ and let $r$ be the last odd integer at most $n$. Then $m\leq M+1$ and $r+2\geq n+1$. For each odd $i$,
    \begin{equation*}
        \frac{1}{i}\geq \frac{1}{2}\int_i^{i+2}\frac{dx}{x}.
    \end{equation*}
    Summing over the odd integers $m,m+2,\ldots,r$ yields
    \begin{align*}
        \sum_{\substack{M\leq i\leq n\\ i \text{ odd}}}\frac{1}{i}
        &\geq
        \frac{1}{2}\int_m^{r+2}\frac{dx}{x}
        =
        \frac{1}{2}\log\left(\frac{r+2}{m}\right)
        \geq
        \frac{1}{2}\log\left(\frac{n+1}{M+1}\right).
    \end{align*}
\end{proof}

\subsection[Proof of the GD lower bound]{Proof of \Cref{main:gd_lower_bound}}\label{app:proof_gd_lower_bound}
\begin{proof}
    For a dyadic pair $B_i=\{i,2i\}$ with $i\leq \floor{d/2}$, the conditional target distribution is
    \begin{equation*}
        p(\cdot\mid B_i)=\left(\frac{2}{3},\frac{1}{3}\right),
    \end{equation*}
    while the conditional model distribution is
    \begin{equation*}
        q_t(\cdot\mid B_i)
        =
        \left(
        \frac{q_{t,i}}{q_{t,i}+q_{t,2i}},
        \frac{q_{t,2i}}{q_{t,i}+q_{t,2i}}
        \right)
        =
        \left(
        \frac{e^{\omega_{t,i}}}{1+e^{\omega_{t,i}}},
        \frac{1}{1+e^{\omega_{t,i}}}
        \right).
    \end{equation*}
    Define
    \begin{equation*}
        \psi(\rho):=
        \KL\left(
        \left(\frac{2}{3},\frac{1}{3}\right)
        \middle\|
        \left(\frac{e^\rho}{1+e^\rho},\frac{1}{1+e^\rho}\right)
        \right).
    \end{equation*}
    Then $\psi(\log 2)=0$ and $\psi'(\log 2)=0$. Moreover,
    \begin{equation*}
        \psi''(\rho)=\frac{e^\rho}{(1+e^\rho)^2}\geq \frac{2}{9},
        \qquad 0\leq \rho\leq \log 2.
    \end{equation*}
    Hence, by Taylor's theorem around $\log 2$,
    \begin{equation}\label{app:pair_kl_lower}
        \psi(\rho)\geq \frac{1}{9}(\log 2-\rho)^2,
        \qquad 0\leq \rho\leq \log 2.
    \end{equation}

    Fix $T\leq d\log d\log 2/(4\eta)$. Since $H_d\geq \log d$, this implies
    \begin{equation*}
        T\leq \frac{dH_d\log 2}{4\eta}.
    \end{equation*}
    Define
    \begin{equation*}
        M_T:=\max\left\{1,\left\lceil\frac{\eta T}{H_d\log 2}\right\rceil\right\},
        \qquad
        n:=\floor{\frac{d}{2}},
        \qquad
        I_T:=\{i:\ M_T\leq i\leq n,\ i \text{ is odd}\}.
    \end{equation*}
    For every $i\in I_T$, we have $\eta T/(iH_d)\leq \log 2$, so \Cref{app:gd_log_ratio} implies
    \begin{equation*}
        0\leq \omega_{T,i}\leq \frac{\eta T}{2iH_d}\leq \frac{\log 2}{2}.
    \end{equation*}
    Combining this with \eqref{app:pair_kl_lower} gives
    \begin{equation*}
        \KL(p(\cdot\mid B_i)\|q_T(\cdot\mid B_i))
        =
        \psi(\omega_{T,i})
        \geq
        \frac{(\log 2)^2}{36}.
    \end{equation*}

    The pairs $\{i,2i\}$ are disjoint for odd $i$. Applying the KL chain rule to the partition formed by $\{B_i:i\in I_T\}$ and the remaining indices,
    \begin{align*}
        f_{\mathrm{uni}}(\theta_T)-f_{\mathrm{uni}}^\star
        =\KL(p\|q_T)
        &\geq
        \sum_{i\in I_T}p(B_i)\KL(p(\cdot\mid B_i)\|q_T(\cdot\mid B_i)) \\
        &\geq
        \sum_{i\in I_T}
        \left(\frac{1}{iH_d}+\frac{1}{2iH_d}\right)
        \frac{(\log 2)^2}{36} \\
        &=
        \frac{(\log 2)^2}{24H_d}\sum_{i\in I_T}\frac{1}{i}.
    \end{align*}
    By \Cref{app:odd_harmonic_tail}, using $n+1\geq d/2$ and $M_T+1\leq \eta T/(H_d\log 2)+2$,
    \begin{align*}
        \sum_{i\in I_T}\frac{1}{i}
        &\geq
        \frac{1}{2}\log\left(\frac{n+1}{M_T+1}\right) \\
        &\geq
        \frac{1}{2}\log\left(
        \frac{d/2}{\eta T/(H_d\log 2)+2}
        \right)
        =
        \frac{1}{2}\log\left(\frac{dH_d\log 2}{2\eta T+4H_d\log 2}\right).
    \end{align*}
    This proves the displayed lower bound.

    It remains to convert the loss lower bound into an iteration lower bound. Let
    \begin{equation*}
        T_{\mathrm{GD}}(\epsilon):=\min\{T:f_{\mathrm{uni}}(\theta_T)-f_{\mathrm{uni}}^\star\leq \epsilon\},
        \qquad
        a:=\frac{48\epsilon}{(\log 2)^2}.
    \end{equation*}
    For the stated consequence, take $A=48/(\log 2)^2$ and restrict to any absolute $\epsilon_0>0$ with $A\epsilon_0<1$, so $a=A\epsilon<1$. If $T_{\mathrm{GD}}(\epsilon)>d\log d\log 2/(4\eta)$, then the desired polynomial lower bound follows immediately for all sufficiently large $d$ because $a<1$. Otherwise, since $H_d\geq \log d$, we have $T_{\mathrm{GD}}(\epsilon)\leq dH_d\log 2/(4\eta)$ and the bound above applies at $T=T_{\mathrm{GD}}(\epsilon)$, so
    \begin{equation*}
        \epsilon
        \geq
        \frac{(\log 2)^2}{48H_d}
        \log\left(\frac{dH_d\log 2}{2\eta T_{\mathrm{GD}}(\epsilon)+4H_d\log 2}\right).
    \end{equation*}
    Rearranging,
    \begin{equation*}
        T_{\mathrm{GD}}(\epsilon)
        \geq
        \frac{H_d\log 2}{2\eta}
        \left(d\exp(-aH_d)-4\right).
    \end{equation*}
    Since $H_d\leq 1+\log d$,
    \begin{equation*}
        d\exp(-aH_d)\geq e^{-a}d^{1-a}.
    \end{equation*}
    Because $a<1$, the term $e^{-a}d^{1-a}$ dominates the constant $4$ for all sufficiently large $d$. Using also $H_d\geq \log d$, there are constants $c_\epsilon,d_\epsilon>0$ such that for every $d\geq d_\epsilon$,
    \begin{equation*}
        T_{\mathrm{GD}}(\epsilon)
        \geq
        \frac{c_\epsilon}{\eta}d^{1-\frac{48\epsilon}{(\log 2)^2}}\log d.
    \end{equation*}
    This completes the proof.
\end{proof}

\subsection{GD with Large Step Sizes is unstable}
\begin{theorem}[Large GD steps are locally unstable]\label{main:gd_local_instability}
    Let $p_i=1/(iH_d)$, where $H_d=\sum_{j=1}^d 1/j$, and let $\theta_\star\in\argmin f_{\mathrm{uni}}$. Define
    \[
        H_\star:=\nabla^2 f_{\mathrm{uni}}(\theta_\star).
    \]
    Then $\lambda_{\max}(H_\star)=\Theta(1/\log d)$. Moreover, consider full-batch GD with constant stepsize $\eta$ initialized
    near $\theta_\star$. The linearized perturbation dynamics around
    $\theta_\star$ are
    \[
        \delta_{t+1}
        =
        (I-\eta H_\star)\delta_t.
    \]
    If
    \[
        \eta>\frac{2}{\lambda_{\max}(H_\star)},
    \]
    then the linearized map has spectral radius greater than one, and perturbations with a nonzero component along a top eigenvector grow to first order. In particular, for all sufficiently large $d$, every stepsize larger than a sufficiently large constant multiple of $\log d$ is locally unstable for full-batch GD.
\end{theorem}
\begin{proof}
    At any minimizer $\theta_\star$, we have $\softmax(\theta_\star)=p$. Therefore
\begin{equation*}
    H_\star=\nabla^2 f_{\mathrm{uni}}(\theta_\star)
    =
    \diag(p)-pp^\top.
\end{equation*}
For any unit vector $u\in\R^d$,
\begin{equation*}
    u^\top H_\star u
    =
    \sum_{i=1}^d p_i u_i^2
    -
    \left(\sum_{i=1}^d p_i u_i\right)^2
    \leq
    \sum_{i=1}^d p_i u_i^2
    \leq
    p_1\norm{u}_2^2
    =
    \frac{1}{H_d}.
\end{equation*}
Hence $\lambda_{\max}(H_\star)\leq 1/H_d$. For the lower bound, take $u=(e_1-e_2)/\sqrt{2}$. Since $p_1=1/H_d$ and $p_2=1/(2H_d)$,
\begin{align*}
    u^\top H_\star u
    &=
    \frac{p_1+p_2}{2}
    -
    \frac{(p_1-p_2)^2}{2} \\
    &=
    \frac{3}{4H_d}
    -
    \frac{1}{8H_d^2}.
\end{align*}
This proves the claimed two-sided bound on $\lambda_{\max}(H_\star)$. Since $H_d=\Theta(\log d)$, we have $\lambda_{\max}(H_\star)=\Theta(1/\log d)$. 
Now consider full-batch GD near $\theta_\star$:
\[
    \theta_{t+1}
    =
    \theta_t-\eta \nabla f_{\mathrm{uni}}(\theta_t).
\]
Write a nearby iterate as
\[
    \theta_t=\theta_\star+\delta_t.
\]
Since $\nabla f_{\mathrm{uni}}(\theta_\star)=0$, Taylor expansion gives
\[
    \nabla f_{\mathrm{uni}}(\theta_\star+\delta_t)
    =
    H_\star \delta_t
    +
    o(\|\delta_t\|_2).
\]
Substituting this into the GD update,
\[
    \theta_\star+\delta_{t+1}
    =
    \theta_\star+\delta_t
    -
    \eta\left(H_\star\delta_t+o(\|\delta_t\|_2)\right).
\]
Thus
\[
    \delta_{t+1}
    =
    (I-\eta H_\star)\delta_t
    +
    o(\|\delta_t\|_2).
\]
This is the linearization of GD around $\theta_\star$.

Let $v$ be a unit eigenvector of $H_\star$ with eigenvalue
$\lambda_{\max}(H_\star)$. If $\delta_t=\alpha v$ with $\alpha$ small, then
the linearized update gives
\[
    \delta_{t+1}
    \approx
    \alpha(1-\eta\lambda_{\max}(H_\star))v.
\]
Therefore this perturbation grows to first order whenever
\[
    |1-\eta\lambda_{\max}(H_\star)|>1,
\]
which is equivalent to
\[
    \eta>\frac{2}{\lambda_{\max}(H_\star)}.
\]
Since $\lambda_{\max}(H_\star)=\Theta(1/\log d)$, this threshold is
$\Theta(\log d)$. Hence any stepsize larger than a sufficiently large
constant multiple of $\log d$ is locally unstable for GD.
\end{proof}

\begin{figure}[htbp]
    \centering
    \includegraphics[width=0.48\textwidth]{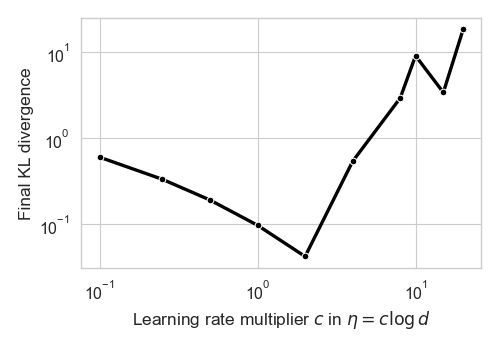}
    \caption{Grid search with vanilla GD in the intermediate- and large-step-size regimes for the softmax unigram model with $d=10^4$. Each point reports the final KL divergence for a learning rate $\eta=c\log d$. Large learning-rate multipliers lead to instability, consistent with \Cref{main:gd_local_instability}.}
    \label{main:figure_gd_instability}
\end{figure}

%% file: sections/generic-stochastic-nsd-wd.tex
\section{Generic Stochastic (Momentum) NSD-WD Proofs}\label{app_sec:nsd_proofs}

In this section, we will present the generic convergence proof of stochastic NSD-WD in \Cref{main:NSD} and also provide a convergence proof of the momentum version in \Cref{main:NSD_mom} and \Cref{main:NSD_mom_tuned}. For completeness, the stochastic NSD-WD algorithm with its momentum version is presented in \Cref{alg:stochastic-NSD-WD}.

\begin{algorithm}[!htbp]
\caption{Stochastic (Momentum) NSD with Weight Decay}\label{alg:stochastic-NSD-WD}
\begin{algorithmic}[1]
\State \textbf{Input:} Start point $x_0$, initial momentum $v_0$ (default $v_0=0$), weight decay $\lambda$, learning rate $\eta_t$, momentum $\rho_t$
\For{$t=1$ to $\dots$}
  \State Compute stochastic estimator
  $v_t = 
  \begin{cases}
      g_{t-1} & \text{Vanilla Stochastic} \\
      (1 - \rho_t) v_{t-1} + \rho_t g_{t-1} & \text{Momentum}
  \end{cases}
  $
  \State $u_t \gets \argmax_{\norm{y} \leq 1} \langle v_t, y\rangle$
  \State $x_t \gets (1-\lambda\eta_t)x_{t-1} - \eta_tu_t$
\EndFor
\end{algorithmic}
\end{algorithm}

First, we present a descent lemma for stochastic NSD-WD that is independent of whether momentum is used. \Cref{app:NSD_descent} only relies on the fact that the descent direction $u_t$ is computed directly from $v_t$. It does not use any information about the structure of the current estimator $v_t$. The estimator can be the current stochastic gradient (the standard stochastic case) or an exponential moving average (the momentum case). 

\begin{lemma}[Stochastic Descent Lemma]\label{app:NSD_mom_descent}\label{app:NSD_descent}
    Suppose $f$ is convex and $L$-smooth with respect to $\norm{\cdot}$. Let $D = \max\{ \norm{x_0}, \frac{1}{\lambda}\}$. If $\lambda \leq 1/\norm{x_\star}$ and $\eta_t \leq 1/\lambda$, then the iterates $\{x_t\}$ of stochastic NSD-WD satisfy,
    \begin{equation*}
        \delta_t \leq (1 - \eta_t\lambda)\delta_{t-1} + 2\eta_t \E[\norm{\nabla f(x_{t-1}) - v_t}_\star] + \frac{L\eta_t^2}{2}(1 + \lambda D)^2,
    \end{equation*}
    where $\delta_t = \E[f(x_t) - f^\star]$.
\end{lemma}

\begin{proof}
    First, we note that the iterates remain in the radius-$D$ ball. Indeed, if $\norm{x_{t-1}} \leq D$, then using $\norm{u_t} \leq 1$ and $\eta_t \leq 1/\lambda$,
    \begin{align*}
        \norm{x_t}
        &\leq (1-\lambda\eta_t)\norm{x_{t-1}} + \eta_t\norm{u_t} \\
        &\leq (1-\lambda\eta_t)D + \eta_t
        \leq D,
    \end{align*}
    where the last inequality follows from $\lambda D \geq 1$. Since $\norm{x_0} \leq D$, induction gives $\norm{x_t} \leq D$ for all $t$.

    Now apply $L$-smoothness:
    \begin{align*}
        f(x_t) \leq f(x_{t-1}) + \ip{\nabla f(x_{t-1})}{x_t - x_{t-1}} + \frac{L}{2}\norm{x_t - x_{t-1}}^2.
    \end{align*}
    The quadratic term is bounded by the triangle inequality and $\norm{u_t} \leq 1$:
    \begin{align*}
        \norm{x_t - x_{t-1}}^2 &= \norm{-\lambda \eta_t x_{t-1} - \eta_t u_t}^2 \\
        &= \norm{\lambda \eta_t x_{t-1} + \eta_t u_t}^2 \\
        &\leq \left(\lambda \eta_t\norm{x_{t-1}} + \eta_t \norm{u_t}\right)^2 \\
        &\leq \eta_t^2\left(\lambda D  + 1\right)^2.
    \end{align*}
    It remains to bound the inner product term. Note that the inner product can be expressed as $-\eta_t \ip{\nabla f(x_{t-1})}{\lambda x_{t-1} + u_t}$.

    Since $u_t \in \argmax_{\norm{u} \leq 1} \ip{v_t}{u}$, we have that $\langle v_t,u_t\rangle=\|v_t\|_\star$.
    Therefore,
    \begin{align*}
        \langle \nabla f(x_{t-1}),u_t\rangle &= \langle v_t,u_t\rangle + \langle \nabla f(x_{t-1})-v_t,u_t\rangle \\
        &\ge \|v_t\|_\star - \|\nabla f(x_{t-1})-v_t\|_\star\|u_t\| \\
        &\ge \|v_t\|_\star - \|\nabla f(x_{t-1})-v_t\|_\star
    \end{align*}

    By the reverse triangle inequality, $\norm{v_t}_\star = \norm{\nabla f(x_{t-1})-(\nabla f(x_{t-1})-v_t)}_\star \geq \|\nabla f(x_{t-1})\|_\star-\|\nabla f(x_{t-1})-v_t\|_\star$. Thus, $\ip{\nabla f(x_{t-1})}{u_t} \geq \|\nabla f(x_{t-1})\|_\star-2\|\nabla f(x_{t-1})-v_t\|_\star$. Next, by the convexity of $f$,
    \begin{equation*}
         f(x_{t-1})-f^\star \leq \langle \nabla f(x_{t-1}),x_{t-1}-x_\star\rangle = \langle \nabla f(x_{t-1}),x_{t-1}\rangle-\langle \nabla f(x_{t-1}),x_\star\rangle.
    \end{equation*}
    Using Cauchy Schwarz with dual norm, $-\langle \nabla f(x_{t-1}),x_\star\rangle \leq \|\nabla f(x_{t-1})\|_\star\|x_\star\|$. Since $\lambda\le 1/\|x_\star\|$, we have $\|x_\star\|\le 1/\lambda$, and hence
    \begin{equation*}
        f(x_{t-1})-f^\star \le \langle \nabla f(x_{t-1}),x_{t-1}\rangle+\frac1\lambda\|\nabla f(x_{t-1})\|_\star .
    \end{equation*}
    Multiplying by $\lambda$ gives
    \begin{equation*}
        \lambda(f(x_{t-1})-f^\star) \le \lambda\langle \nabla f(x_{t-1}),x_{t-1}\rangle+\|\nabla f(x_{t-1})\|_\star .
    \end{equation*}

    Combining this inequality with the lower bound on
    $\langle \nabla f(x_{t-1}),u_t\rangle$, we get
    \begin{align*}
        \langle \nabla f(x_{t-1}),\lambda x_{t-1}+u_t\rangle &= \lambda\langle \nabla f(x_{t-1}),x_{t-1}\rangle + \langle \nabla f(x_{t-1}),u_t\rangle \\
        &\ge \lambda\langle \nabla f(x_{t-1}),x_{t-1}\rangle + \|\nabla f(x_{t-1})\|_\star - 2\|\nabla f(x_{t-1})-v_t\|_\star \\
        &\ge \lambda(f(x_{t-1})-f^\star) - 2\|\nabla f(x_{t-1})-v_t\|_\star .
    \end{align*}

    Substituting into the smoothness inequality yields
    \begin{align*}
        f(x_t)-f^\star &\le f(x_{t-1})-f^\star - \eta_t\lambda(f(x_{t-1})-f^\star) + 2\eta_t\|\nabla f(x_{t-1})-v_t\|_\star +
        \frac{L\eta_t^2}{2}(1+\lambda D)^2 \\
        &= (1-\eta_t\lambda)(f(x_{t-1})-f^\star) + 2\eta_t\|\nabla f(x_{t-1})-v_t\|_\star + \frac{L\eta_t^2}{2}(1+\lambda D)^2 .
    \end{align*}
    Finally, taking expectation gives
    \begin{equation*}
        \delta_t \le (1-\eta_t\lambda)\delta_{t-1} + 2\eta_t\mathbb E\|\nabla f(x_{t-1})-v_t\|_\star + \frac{L\eta_t^2}{2}(1+\lambda D)^2 .
    \end{equation*}
\end{proof}

\subsection{Stochastic NSD-WD Proofs}

Here, we prove \Cref{main:NSD} by solving the recursion in \Cref{app:NSD_descent} with a horizon-tuned fixed learning rate.

\begin{proof}[Proof of \Cref{main:NSD}]
    Let $\delta_t:=\E[f(x_t)-f^\star]$, $C:=(1+\lambda D)^2$, and $\eta=\frac{\log T}{\lambda T}$. Since $T\ge2$, $\eta\le1/\lambda$. For the vanilla estimator $v_t=g_{t-1}$, the noise assumption and \Cref{app:NSD_descent} give
    \begin{equation*}
        \delta_t
        \le (1-\eta\lambda)\delta_{t-1}
        +2\eta\sigma+\frac{L\eta^2C}{2}.
    \end{equation*}
    Iterating and summing the geometric series yields
    \begin{equation*}
        \delta_T
        \le (1-\eta\lambda)^T\delta_0
        +\frac{2\sigma}{\lambda}
        +\frac{L\eta C}{2\lambda}.
    \end{equation*}
    Smoothness at a minimizer and $\norm{x_0-x_\star}\le D+1/\lambda$ imply
    \begin{equation*}
        \delta_0
        \le \frac{L}{2}\norm{x_0-x_\star}^2
        \le \frac{LC}{2\lambda^2}.
    \end{equation*}
    Since $(1-\eta\lambda)^T\le e^{-\eta\lambda T}=T^{-1}$,
    \begin{equation*}
        \delta_T
        \le \frac{LC(1+\log T)}{2\lambda^2T}
        +\frac{2\sigma}{\lambda}
        \le \frac{2LC\log T}{\lambda^2T}
        +\frac{2\sigma}{\lambda}.
    \end{equation*}
    If $x_0=0$ and $\lambda=1/\norm{x_\star}$, then $D=\norm{x_\star}$ and $C=4$, which gives the second conclusion.
\end{proof}

\subsection{Stochastic Momentum NSD-WD Proofs}

In this subsection, we present a generic convergence bound for stochastic momentum NSD-WD in \Cref{main:NSD_mom}. Then, we tune this bound in terms of the momentum parameter $\rho$ to provide a convergence rate for momentum NSD-WD in \Cref{main:NSD_mom_tuned}. \Cref{main:NSD_mom_tuned} will be used in \Cref{app:stochastic_momentum_nsd} to compare upper bounds for momentum SignSGD and momentum normalized SGD on the stochastic softmax unigram model. For the proof of \Cref{main:NSD_mom}, we first need to bound the momentum error in the dual norm which is presented in \Cref{app:NSD_mom_error}. 

\begin{lemma}[Momentum estimator error]\label{app:NSD_mom_error}
    Suppose $f$ is $L$-smooth with respect to $\norm{\cdot}$. Suppose the stochastic estimator $g_t$ satisfies $\E_t[g_t]=\nabla f(x_t)$ and has bounded Euclidean second moment, $\E_t\left[\norm{g_t-\nabla f(x_t)}_2^2\right]\le \nu_2^2$. Let $\kappa := \sup_{z\neq 0}\norm{z}_\star/\norm{z}_2$ and $D = \max\{ \norm{x_0}, 1/\lambda\}$. Suppose the algorithm uses fixed $0<\eta\leq1/\lambda$ and $\rho\in(0,1]$, so that $v_t=(1-\rho)v_{t-1}+\rho g_{t-1}$ for $t\geq1$. Then
    \begin{equation*}
        \E[\norm{\nabla f(x_{t-1})-v_t}_\star] \leq (1-\rho)^t \E[\norm{\nabla f(x_0)-v_0}_\star] + \kappa\nu_2\sqrt{\rho} + \frac{L\eta(1+\lambda D)}{\rho}.
    \end{equation*}
\end{lemma}
\begin{proof}
    Let $\xi_t = g_t - \nabla f(x_t)$. Then $\E_t[\xi_t] = 0$, and the second-moment assumption gives $\E_t[\norm{\xi_t}_2^2] \leq \nu_2^2$. The momentum estimator can be expressed as
    \begin{equation*}
        v_t = (1 - \rho)v_{t-1} + \rho g_{t-1}.
    \end{equation*}
    Define $a_0:=v_0-\nabla f(x_0)$ and, for $t\geq1$, define the error $a_t:=v_t-\nabla f(x_{t-1})$. For $t=1$,
    \begin{equation*}
        a_1=(1-\rho)a_0+\rho\xi_0.
    \end{equation*}
    For $t\geq2$,
    \begin{align*}
        a_t
        &= (1-\rho)v_{t-1}+\rho g_{t-1}-\nabla f(x_{t-1}) \\
        &= (1-\rho)a_{t-1}
        +\rho\bigl(g_{t-1}-\nabla f(x_{t-1})\bigr)
        +(1-\rho)\bigl(\nabla f(x_{t-2})-\nabla f(x_{t-1})\bigr) \\
        &= (1 - \rho)a_{t-1}
        +(1-\rho)\bigl(\nabla f(x_{t-2})-\nabla f(x_{t-1})\bigr)
        +\rho\xi_{t-1}.
    \end{align*}
    Let $\beta=1-\rho$. Unrolling the recursion gives, for every $t\geq1$,
    \begin{align*}
        a_t
        &=\beta^t a_0
        +\rho\sum_{k=0}^{t-1}\beta^{t-1-k}\xi_k
        +\sum_{k=1}^{t-1}\beta^{t-k}
        \bigl(\nabla f(x_{k-1})-\nabla f(x_k)\bigr).
    \end{align*}
    By applying the triangle inequality, we can get an upper bound on $\E\norm{a_t}_\star$ by computing the dual norm of each term. We start with the last term. Note that we can apply smoothness,
    \begin{equation*}
        \norm{\nabla f(x_k) - \nabla f(x_{k-1})}_\star \leq L\norm{x_k - x_{k-1}} \leq L\eta(1 + \lambda\norm{x_{k-1}}) \leq L\eta(1 + \lambda D).
    \end{equation*}
    Then,
    \begin{equation*}
        \norm{\sum_{k=1}^{t-1} (1-\rho)^{t-k}\bigl(\nabla f(x_{k-1}) - \nabla f(x_k)\bigr)}_\star \leq L\eta(1 + \lambda D)\sum_{j=1}^\infty(1-\rho)^j
    \end{equation*}
    Then,
    \begin{equation*}
        L\eta(1 + \lambda D)\sum_{j=1}^\infty(1-\rho)^j \leq L\eta(1 + \lambda D)\frac{1 - \rho}{\rho} \leq \frac{L\eta(1 + \lambda D)}{\rho}
    \end{equation*}
    Now, we handle the stochastic variance term in the middle, which requires converting to the $\ell_2$ norm and incurring the norm penalty $\kappa$. So,
    \begin{equation*}
        \E\norm{\rho\sum_{k=0}^{t-1} (1 - \rho)^{t-1-k}\xi_k}_\star \leq \kappa \E\norm{\rho\sum_{k=0}^{t-1} (1 - \rho)^{t-1-k}\xi_k}_2 \leq \kappa\sqrt{\E\norm{\rho\sum_{k=0}^{t-1} (1 - \rho)^{t-1-k}\xi_k}_2^2}
    \end{equation*}
    The last inequality is obtained via Jensen's inequality. 

    Now compute the second moment. Let
    \begin{equation*}
        \mathbb E \|S_t\|_2^2 = \mathbb E\norm{\rho \sum_{k=0}^{t-1} (1-\rho)^{t-1-k}\xi_k}_2^2   .
    \end{equation*}
    Expanding the square gives
    \begin{equation*}
        \mathbb E\|S_t\|_2^2 = \rho^2\sum_{k=0}^{t-1} (1-\rho)^{2(t-1-k)}\mathbb E\|\xi_k\|_2^2 + 2\rho^2 \sum_{0\leq k<\ell\leq t-1} (1-\rho)^{t-1-k}(1-\rho)^{t-1-\ell} \mathbb E\langle \xi_k,\xi_\ell\rangle .
    \end{equation*}
    The cross terms vanish because $(\xi_k)_{k\ge 0}$ is a martingale difference sequence. Indeed, for $k<\ell$, since $\xi_k$ is $\mathcal F_{\ell-1}$-measurable,
    \begin{equation*}
        \mathbb E\langle \xi_k,\xi_\ell\rangle = \mathbb E\left[\mathbb E\left[\langle \xi_k,\xi_\ell\rangle \mid \mathcal F_{\ell-1}\right]\right] = \mathbb E\left[\left\langle \xi_k, \mathbb E[\xi_\ell\mid \mathcal F_{\ell-1}] \right\rangle\right]=0.
    \end{equation*}
    Therefore,
    \begin{equation*}
        \mathbb E\|S_t\|_2^2 = \rho^2 \sum_{k=0}^{t-1} (1-\rho)^{2(t-1-k)} \mathbb E\|\xi_k\|_2^2 \leq \rho^2\nu_2^2 \sum_{k=0}^{t-1} (1-\rho)^{2(t-1-k)}
    \end{equation*}
    Re-indexing the sum with $j=t-1-k$ gives
    \begin{equation*}
        \sum_{k=0}^{t-1} (1-\rho)^{2(t-1-k)} = \sum_{j=0}^{t-1} (1 - \rho)^{2j} \leq \sum_{j=0}^{\infty} (1 - \rho)^{2j} = \frac{1}{1 - (1 - \rho)^2} = \frac{1}{\rho(2 - \rho)}.
    \end{equation*}
    Thus,
    \begin{equation*}
        \sqrt{\E\norm{S_t}_2^2} \leq \sqrt{\frac{\nu_2^2\rho}{2 - \rho}} \leq \nu_2 \sqrt{\frac{\rho}{2 - \rho}} \leq \nu_2\sqrt{\rho}. 
    \end{equation*}
    The last inequality comes from the fact that $0 < \rho \leq 1$ so $2 - \rho \geq 1$. Therefore,
    \begin{equation*}
        \E\norm{\rho\sum_{k=0}^{t-1} (1 - \rho)^{t-1-k}\xi_k}_\star \leq \kappa\nu_2\sqrt{\rho}.
    \end{equation*}
    Putting it all together,
    \begin{equation*}
        \E\norm{\nabla f(x_{t-1})-v_t}_\star = \E\norm{a_t}_\star \leq (1 - \rho)^t\E\norm{\nabla f(x_0)-v_0}_\star + \kappa\nu_2\sqrt{\rho} + \frac{L\eta(1 + \lambda D)}{\rho}.
\end{equation*}
\end{proof}

\looseness=-1
Although the stochastic momentum variant of \Cref{alg:stochastic-NSD-WD} can be used with time-varying $\rho_t$ and $\eta_t$, for our analysis, we use fixed $\rho$ and $\eta$ in \Cref{app:NSD_mom_error} and \Cref{main:NSD_mom}. We later set $\rho$ and $\eta$ in terms of the horizon $T$ to achieve a convergence rate for stochastic momentum NSD-WD in \Cref{main:NSD_mom_tuned}. Using a fixed $\rho$ allows us to tune the bound in \Cref{main:NSD_mom} and set $\rho$ in terms of problem-dependent constants such as $\sigma$ and $\norm{x_\star}$ to achieve a tighter bound. If we naively use a time-varying $\rho_t$, the upper bound we get will not be tight enough to obtain a separation between momentum SignSGD and normalized SGD. As a technicality, to complete the proof of \Cref{main:NSD_mom} using a fixed $\rho$ we must also use a fixed $\eta$. 

\begin{theorem}[Stochastic Momentum NSD-WD]\label{main:NSD_mom}
     Suppose $f$ is convex and $L$-smooth with respect to $\norm{\cdot}$. Suppose the stochastic estimator $g_t$ satisfies $\E_t[g_t]=\nabla f(x_t)$ and it has bounded Euclidean second moment $\E_t\left[\norm{g_t-\nabla f(x_t)}_2^2\right]\le \nu_2^2$.
     Suppose we run stochastic momentum NSD with weight decay $\lambda \leq \frac{1}{\norm{x_\star}}$. Let $D = \max\{ \norm{x_0}, 1/\lambda\}$ and $\kappa = \sup_{z \neq 0}\frac{\norm{z}_\star}{\norm{z}_2}$. For fixed $\eta \leq 1/\lambda$ and $\rho \in (0,1]$, the iterates satisfy
    \begin{equation*}
        \E[f(x_T)-f^\star] \leq (1-\lambda\eta)^T(f(x_0)-f^\star)
        + \frac{2\eta \E[\norm{\nabla f(x_0)-v_0}_\star]}{\rho}
        + \frac{2\kappa\nu_2\sqrt{\rho}}{\lambda}
        + \frac{3L\eta(1+\lambda D)^2}{2\rho\lambda}.
    \end{equation*}
\end{theorem}
\begin{proof}
    Let $\delta_t := \E[f(x_t)-f^\star], \tau := 1-\lambda\eta$, $r := 1-\rho$, and
    \begin{equation*}
        A := \kappa\nu_2\sqrt{\rho} + \frac{L\eta(1+\lambda D)}{\rho}, \qquad
        B := \frac{L\eta^2(1+\lambda D)^2}{2}.
    \end{equation*}
    Combining \Cref{app:NSD_mom_descent,app:NSD_mom_error} gives, for $t=1,\ldots,T$,
    \begin{align*}
        \delta_t
        &\leq \tau\delta_{t-1} + 2\eta\left(r^t \E[\norm{\nabla f(x_0)-v_0}_\star] + A\right) + B.
    \end{align*}
    Unrolling the recursion,
    \begin{align*}
        \delta_T
        &\leq \tau^T\delta_0
        + 2\eta \E[\norm{\nabla f(x_0)-v_0}_\star]\sum_{t=1}^{T}\tau^{T-t}r^t
        + (2\eta A+B)\sum_{t=1}^{T}\tau^{T-t} \\
        &\leq \tau^T\delta_0
        + 2\eta \E[\norm{\nabla f(x_0)-v_0}_\star]\sum_{t=0}^{\infty}r^t
        + (2\eta A+B)\sum_{t=0}^{\infty}\tau^t \\
        &= \tau^T\delta_0
        + \frac{2\eta \E[\norm{\nabla f(x_0)-v_0}_\star]}{\rho}
        + \frac{2A}{\lambda}
        + \frac{B}{\lambda\eta}.
    \end{align*}
    Substituting the definitions of $A$ and $B$ gives
    \begin{equation*}
        \delta_T \leq \tau^T\delta_0
        + \frac{2\eta \E[\norm{\nabla f(x_0)-v_0}_\star]}{\rho}
        + \frac{2\kappa\nu_2\sqrt{\rho}}{\lambda}
        + \frac{L\eta(1+\lambda D)}{\lambda}\left(\frac{2}{\rho}+\frac{1+\lambda D}{2}\right).
    \end{equation*}
    Since $\lambda D\geq 1$, we have $1+\lambda D\geq 2$. Also, $\rho\leq 1$. Therefore,
    \begin{equation*}
        \frac{2}{\rho}+\frac{1+\lambda D}{2}
        \leq \frac{2}{\rho}+\frac{1+\lambda D}{2\rho}
        = \frac{4+(1+\lambda D)}{2\rho}
        \leq \frac{3(1+\lambda D)}{2\rho},
    \end{equation*}
    which proves the claim.
\end{proof}

\begin{corollary}[Optimized Stochastic Momentum NSD-WD]\label{main:NSD_mom_tuned}
    Under the assumptions of \Cref{main:NSD_mom}, initialize at $x_0=0$, and choose the optimal weight decay $\lambda=1/\norm{x_\star}$. For horizon $T\geq 3$, set the fixed learning rate and momentum
\begin{equation*}
    \eta = \frac{\norm{x_\star}\log T}{T},
    \qquad
    \rho =
    \left(\frac{4L\norm{x_\star}}{\kappa\nu_2}\right)^{\frac{2}{3}}
    \left(\frac{\log T}{T}\right)^{\frac{2}{3}}.
\end{equation*}
    If the horizon satisfies $T = \tilde{\Omega}\left(\frac{L\norm{x_\star}}{\kappa\nu_2}\right)$ so that $\rho\leq 1$, then the expected suboptimality satisfies
    \begin{equation*}
        \begin{aligned}
            \E[f(x_T)-f^\star]
            \leq\;&
            \left[f(x_0)-f^\star + 2L\norm{x_\star}^2\right]\frac{\log T}{T} \\
            &+
            \frac{2(\kappa\nu_2)^{\frac{2}{3}}\norm{x_\star}^{\frac{1}{3}}}{L^{\frac{2}{3}}}
            \left[
            \E[\norm{\nabla f(x_0)-v_0}_\star]+3L\norm{x_\star}
            \right]
            \left(\frac{\log T}{T} \right)^{\frac{1}{3}}.
        \end{aligned}
    \end{equation*}
\end{corollary}
\begin{proof}
    We now make our choice for learning rate $\eta$ and momentum $\rho$. Specifically, let $\eta = \frac{\log T}{\lambda T}$. Then,
\begin{equation*}
    (1 - \lambda\eta)^T \leq e^{-\lambda\eta T} = e^{-\log T} = \frac{1}{T} \leq \frac{\log T}{T}. 
\end{equation*}
Using the bound from \Cref{main:NSD_mom}, we obtain
\begin{equation*}
    \delta_T \leq \frac{\delta_0}{T} + \frac{2\E[\norm{\nabla f(x_0)-v_0}_\star]\log T}{\rho\lambda T} + \frac{2\kappa \nu_2 \sqrt{\rho}}{\lambda} + \frac{2L(1+\lambda D)\log T}{\lambda^2 \rho T} + \frac{L(1 + \lambda D)^2\log T}{2\lambda^2T}
\end{equation*}
To choose $\rho$, we balance the two dominant $\rho$-dependent contributions
\begin{equation*}
    g(\rho) = \frac{2\kappa \nu_2 \sqrt{\rho}}{\lambda} + \frac{2L(1+\lambda D)\log T}{\lambda^2 \rho T}.
\end{equation*}
Solving $g'(\rho) = 0$ gives the tuning rule $\rho_\star = \alpha\left(\frac{\log T}{T}\right)^{2/3}$, where $\alpha = \left( \frac{2L(1+\lambda D)} {\kappa\nu_2\lambda}\right)^{2/3}$, and hence
\begin{equation*}
    g(\rho_\star) = 3\cdot 2^{1/3} \frac{(\kappa\nu_2)^{2/3} \bigl(L(1+\lambda D)\log T\bigr)^{1/3}} {\lambda^{4/3}T^{1/3}}.
\end{equation*}
We must have $\rho \leq 1$, so this implies that $T \geq \log T \alpha^{3/2}$. Since we are interested in regimes where our budget and iterations are large, we suppress $\log T$ factors to simplify our analysis and let $T = \tilde{\Omega}(\alpha^{3/2})$. Next, we simply substitute $\rho$ to obtain the bound,
\begin{equation*}
\begin{aligned}
    \delta_T
    \leq\;&
    \left[\delta_0 + \frac{L(1 + \lambda D)^2}{2\lambda^2}\right]\frac{\log T}{T} \\
    &+
    \left[
    \frac{2^{\frac{1}{3}}(\kappa\nu_2)^{2/3}\E[\norm{\nabla f(x_0)-v_0}_\star]}
    {\lambda^{\frac{1}{3}}(L(1+\lambda D))^{\frac{2}{3}}}
    + 3 \cdot 2^{1/3}
    \frac{(\kappa\nu_2)^{\frac{2}{3}}\bigl(L(1+\lambda D)\bigr)^{\frac{1}{3}}}
    {\lambda^{4/3}}
    \right]
    \left(\frac{\log T}{T} \right)^{\frac{1}{3}} .
\end{aligned}
\end{equation*}
The remaining term
\[
    \frac{2\E[\norm{\nabla f(x_0)-v_0}_\star]\log T}{\rho_\star\lambda T}
\]
also scales as $(\log T/T)^{1/3}$ under this choice, which is why it appears in the same bracket above after substitution.
    When $x_0 = 0$ and $\lambda = 1/\norm{x_\star}$, we have $D = \norm{x_\star}$ and $1 + \lambda D = 2$, so
\[
    \rho_\star
    =
    \left(\frac{4L\norm{x_\star}}{\kappa\nu_2}\right)^{\frac{2}{3}}
    \left(\frac{\log T}{T}\right)^{\frac{2}{3}}.
\]
    The condition $\rho_\star \leq 1$ is therefore equivalent, up to logarithmic factors, to $T = \tilde{\Omega}\left(\frac{L\norm{x_\star}}{\kappa\nu_2}\right)$. After this substitution, the two coefficients needed below satisfy $2^{-1/3}\leq2$ and $3\cdot2^{2/3}\leq6$. Thus
\begin{equation*}
    \begin{aligned}
        \delta_T
        \leq\;&
        \left[\delta_0 + 2L\norm{x_\star}^2\right]\frac{\log T}{T} \\
        &+
        \frac{2(\kappa\nu_2)^{\frac{2}{3}}\norm{x_\star}^{\frac{1}{3}}}{L^{\frac{2}{3}}}
        \left[
        \E[\norm{\nabla f(x_0)-v_0}_\star]
        + 3L\norm{x_\star}
        \right]
        \left(\frac{\log T}{T} \right)^{\frac{1}{3}},
    \end{aligned}
\end{equation*}
as claimed.
\end{proof}

%% file: sections/stochastic-softmax-proofs.tex
\section{Stochastic Softmax Unigram Analysis Proofs}
\begin{proof}[Proof of \Cref{main:stochastic_sign_fixed}]\label{app:proof_stochastic_sign_fixed}
    We apply \Cref{main:NSD} with respect to $\norm{\cdot}_\infty$, whose dual norm is $\norm{\cdot}_1$. By \Cref{main:unigram_smoothness_lemma},
    \begin{equation*}
        L_{\norm{\cdot}_\infty}(f_{\mathrm{uni}})=1.
    \end{equation*}
    By \Cref{main:unigram_norm_lemma1}, the minimum $\ell_\infty$-norm solution has norm $\log(d)/2$. Thus, with $\lambda=2/\log d$ and $\theta_0=0$, we have
    \begin{equation*}
        D=1/\lambda=\log(d)/2
        \qquad\text{and}\qquad
        1+\lambda D=2.
    \end{equation*}
    Also, by \Cref{main:unigram_noise_lemma,main:unigram_noise_lemma2}, for all sufficiently large $d$,
    \begin{equation*}
        \sigma_\infty \leq \min\left\{2,4\sqrt{\frac{d}{B\log d}}\right\}.
    \end{equation*}
    Applying \Cref{main:NSD} yields
    \begin{align*}
        \E[f_{\mathrm{uni}}(\theta_T)-f_{\mathrm{uni}}^\star]
        &\leq \frac{2(1+\lambda D)^2\log T}{\lambda^2T}
        + \frac{2\sigma_\infty}{\lambda} \\
        &\leq
        \frac{2(\log d)^2\log T}{T}
        +\min\left\{2\log d,4\sqrt{\frac{d\log d}{B}}\right\}.
    \end{align*}
\end{proof}

\begin{proof}[Proof of \Cref{main:large_batch_signsgd_complexity}]\label{app:proof_large_batch_signsgd_complexity}
    Let $\ell:=\log d$ and
    \begin{equation*}
        T_0:=\frac{8\ell^2}{\epsilon}
        \log\left(\frac{8\ell^2}{\epsilon}\right).
    \end{equation*}
    For sufficiently large $d$, the map $T\mapsto\log T/T$ is decreasing at $T_0$, and
    \begin{equation*}
        \frac{2\ell^2\log T_0}{T_0}\le\frac{\epsilon}{2}.
    \end{equation*}
    Indeed, with $A:=4\ell^2/\epsilon$, we have $T_0=2A\log(2A)$ and $\log T_0\le2\log(2A)$. Moreover, the batch-size assumption gives
    \begin{equation*}
        N_B(d)
        \le4\sqrt{\frac{d\ell}{B}}
        \le\frac{\epsilon}{2}.
    \end{equation*}
    Applying \Cref{main:stochastic_sign_fixed} at $T=\lceil T_0\rceil$ proves the claim.
\end{proof}

\begin{proof}[Proof of \Cref{main:stochastic_norm_fixed}]\label{app:proof_stochastic_norm_fixed}
    We apply \Cref{main:NSD} with respect to $\norm{\cdot}_2$. By \Cref{main:unigram_smoothness_lemma},
\begin{equation*}
    L_{\norm{\cdot}_2}(f_{\mathrm{uni}})=\frac{1}{2}.
\end{equation*}
By \Cref{main:unigram_norm_lemma1,main:unigram_norm_lemma2}, the minimum $\ell_2$-norm solution has norm of order $\sqrt d$, and the stated choice $\lambda=1/\sqrt d$ gives
\begin{equation*}
    D=\max\{\norm{\theta_0}_2,1/\lambda\}=\sqrt d
\end{equation*}
for $\theta_0=0$. Finally, \Cref{main:unigram_noise_lemma} gives the dual-norm stochastic-gradient error
\begin{equation*}
    \sigma_2 \leq \frac{1}{\sqrt B}.
\end{equation*}
Substituting these bounds into \Cref{main:NSD},
\begin{align*}
    \E[f_{\mathrm{uni}}(\theta_T)-f_{\mathrm{uni}}^\star]
    &\leq \frac{(1+\lambda D)^2\log T}{\lambda^2T}
    + \frac{2\sigma_2}{\lambda} \\
    &= \frac{4d\log T}{T}+2\sqrt{\frac{d}{B}}.
\end{align*}
\end{proof}

\begin{proof}[Proof of \Cref{main:stochastic_sign_sample}]\label{app:proof_stochastic_sign_sample}
    Let $\ell:=\log d$ and $K:=\log N$. Under $N=BT$, define
\begin{equation*}
    F_\infty(T):=\frac{2\ell^2\log T}{T}
    +\min\left\{2\log d,4\sqrt{\frac{Td\log d}{N}}\right\},
    \qquad
    2\leq T\leq N.
\end{equation*}
By \Cref{main:stochastic_sign_fixed}, the error is upper bounded by $F_\infty(T)$. Since $\log T\le K$, define the upper envelope
\[
    \widetilde F_\infty(T)
    :=\frac{2\ell^2K}{T}+4\sqrt{\frac{Td\ell}{N}} .
\]
    Optimizing this envelope gives
\begin{equation*}
    \widetilde F_\infty'(T)
    =2\sqrt{\frac{d\ell}{NT}}-\frac{2\ell^2K}{T^2},
    \qquad
    T_\infty^\star=\ell K^{2/3}\left(\frac{N}{d}\right)^{1/3}.
\end{equation*}
When $N\ge d$ and $d$ is large, this choice lies in $[2,N]$. Therefore, choosing
\[
    T=T_\infty^\star,
    \qquad
    B=\frac{N}{T_\infty^\star}
    =\frac{(dN^2)^{1/3}}{\ell K^{2/3}},
\]
yields
\begin{align*}
    F_\infty(T_\infty^\star)
    &\leq \widetilde F_\infty(T_\infty^\star) \\
    &=\frac{6d^{1/3}\ell K^{1/3}}{N^{1/3}}.
\end{align*}
This proves the claimed result.
\end{proof}

\begin{proof}[Proof of \Cref{main:stochastic_norm_sample}]\label{app:proof_stochastic_norm_sample}
    Let $K:=\log N$. Since $\log T\le K$, \Cref{main:stochastic_norm_fixed} is upper bounded by
\begin{equation*}
    \widetilde F_2(T):=\frac{4dK}{T}+2\sqrt{\frac{dT}{N}},
    \qquad
    2\leq T\leq N.
\end{equation*}
Its derivative is
\begin{equation*}
    \widetilde F_2'(T)=\sqrt{\frac{d}{NT}}-\frac{4dK}{T^2}.
\end{equation*}
    The stationary point satisfies
\begin{equation*}
    T^{3/2}=4K\sqrt{dN},
\end{equation*}
    and hence
\begin{equation*}
    T_2^\star=2^{4/3}K^{2/3}(dN)^{1/3}.
\end{equation*}
    This is the unique minimizer, and $T_2^\star\ge2$. Moreover,
\begin{equation*}
    T_2^\star\le N
    \quad\Longleftrightarrow\quad
    2^{4/3}K^{2/3}(dN)^{1/3}\le N
    \quad\Longleftrightarrow\quad
    16dK^2\le N^2,
\end{equation*}
    which follows from $N\ge4\sqrt d\log N$. Therefore,
\begin{align*}
    \widetilde F_2(T_2^\star)
    &=\frac{3\cdot 2^{2/3}d^{2/3}K^{1/3}}{N^{1/3}}.
\end{align*}
    The corresponding batch size is
\begin{equation*}
    B_2^\star=\frac{N}{T_2^\star}=\frac{N^{2/3}}{2^{4/3}d^{1/3}K^{2/3}},
\end{equation*}
which gives the claimed result.
\end{proof}

\begin{proof}[Proof of \Cref{main:sgd_bound}]\label{app:proof_sgd_bound}
    Let $\theta_\star$ be the minimum-$\ell_2$-norm minimizer of $f_{\mathrm{uni}}$, and set
\[
    D:=\norm{\theta_0-\theta_\star}_2=\norm{\theta_\star}_2.
\]
By \Cref{main:unigram_norm_lemma1,main:unigram_norm_lemma2}, we have the exact bound
\[
    D\le \sqrt d.
\]
Let
\begin{equation*}
    \bar g_t=\nabla \widehat f_{\mathrm{uni},B}(\theta_t;\xi_{t,1:B})
\end{equation*}
be the mini-batch gradient used by SGD. Since
\[
    \nabla \widehat f_{\mathrm{uni},B}(\theta_t;\xi_{t,1:B})-\nabla f_{\mathrm{uni}}(\theta_t)=p-\hat p_t,
\]
the covariance computation from the proof of \Cref{main:unigram_noise_lemma} gives
\begin{equation*}
    \E_t[\bar g_t]=\nabla f_{\mathrm{uni}}(\theta_t),
    \qquad
    \E_t\left[\norm{\bar g_t-\nabla f_{\mathrm{uni}}(\theta_t)}_2^2\right]
    =\frac{1-\norm{p}_2^2}{B}
    \leq \frac{1}{B}.
\end{equation*}
Choose the constant stepsize
\begin{equation*}
    \eta=\min\left\{1,\frac{D\sqrt B}{\sqrt T}\right\}.
\end{equation*}
The update is $\theta_{t+1}=\theta_t-\eta \bar g_t$. For any minimizer $\theta^\star$, the squared-distance recursion gives
\begin{align*}
    \E_t\norm{\theta_{t+1}-\theta^\star}_2^2
    &=\norm{\theta_t-\theta^\star}_2^2
    -2\eta\ip{\nabla f_{\mathrm{uni}}(\theta_t)}{\theta_t-\theta^\star}
    +\eta^2\E_t\norm{\bar g_t}_2^2 \\
    &\leq \norm{\theta_t-\theta^\star}_2^2
    -2\eta\left(f_{\mathrm{uni}}(\theta_t)-f_{\mathrm{uni}}^\star\right)
    +\eta^2\norm{\nabla f_{\mathrm{uni}}(\theta_t)}_2^2
    +\frac{\eta^2}{B}.
\end{align*}
By \Cref{main:unigram_smoothness_lemma}, $L_{\norm{\cdot}_2}(f_{\mathrm{uni}})=1/2$, so convex smoothness gives
\[
    \norm{\nabla f_{\mathrm{uni}}(\theta_t)}_2^2
    \leq f_{\mathrm{uni}}(\theta_t)-f_{\mathrm{uni}}^\star.
\]
Since $\eta\leq 1$,
\begin{equation*}
    \E_t\norm{\theta_{t+1}-\theta^\star}_2^2
    \leq \norm{\theta_t-\theta^\star}_2^2
    -\eta\left(f_{\mathrm{uni}}(\theta_t)-f_{\mathrm{uni}}^\star\right)
    +\frac{\eta^2}{B}.
\end{equation*}
Summing from $t=0$ to $T-1$ and using Jensen's inequality for $\bar{\theta}_T=T^{-1}\sum_{t=0}^{T-1}\theta_t$ yields
\begin{equation*}
    \E[f_{\mathrm{uni}}(\bar{\theta}_T)-f_{\mathrm{uni}}^\star]
    \leq \frac{D^2}{\eta T}+\frac{\eta}{B}.
\end{equation*}
Let $a=1$ and $b=D\sqrt{B/T}$, so $\eta=\min\{a,b\}$. Since $1/\eta\leq 1/a+1/b$ and $\eta\leq b$,
\begin{align*}
    \frac{D^2}{\eta T}+\frac{\eta}{B}
    &\leq \frac{D^2}{T}\left(1+\frac{\sqrt T}{D\sqrt B}\right)
    +\frac{D}{\sqrt{BT}} \\
    &=\frac{D^2}{T}+\frac{2D}{\sqrt{BT}} \\
    &\leq \frac{d}{T}+\frac{2\sqrt d}{\sqrt{BT}},
  \end{align*}
  which proves the theorem; the identity
  \[
      \frac{d}{T}+\frac{2\sqrt d}{\sqrt{BT}}
      =
      \frac{dB}{N}+2\sqrt{\frac{d}{N}}
  \]
  follows immediately from $N=BT$.
\end{proof}

\subsection{Fixed-batch SignGD vs.\ SGD}\label{app:sign_sgd_batch_comparison}\label{app:proof_sign_sgd_batch_comparison}
Under $N=BT$, the SignSGD bound contains both a minimum and $\log T=\log(N/B)$. For a simple conservative comparison, we use its Jensen branch and $\log T\le\log N$, defining
\begin{equation*}
    \Delta_{\mathrm{SGD}}(B)
    := \frac{dB}{N}+2\sqrt{\frac{d}{N}},
    \qquad
    \Delta_{\mathrm{Sign}}(B)
    := \frac{2B(\log d)^2\log N}{N}+4\sqrt{\frac{d\log d}{B}}.
\end{equation*}
The following lemma formalizes the fixed-batch comparison discussed in \Cref{sec:compare_sgd}.
\begin{lemma}[Crossing batch size for SignGD vs.\ SGD]\label{app:sign_sgd_batch_crossing}
    Suppose $d$ is sufficiently large, $N\ge d$, and $(\log d)^2\log N=o(d)$. There is a threshold batch size $B_{\mathrm{cross}}
    =\Theta\!\left(\left(\frac{N^2\log d}{d}\right)^{1/3}\right)$ such that $B_{\mathrm{cross}}\in[1,N/2]$ and
\begin{equation*}
    \Delta_{\mathrm{Sign}}(B) < \Delta_{\mathrm{SGD}}(B)
    \iff
    B>B_{\mathrm{cross}}.
\end{equation*}
\end{lemma}
\begin{proof}
    Let $\ell=\log d$, $K=\log N$, and define the difference between the two upper-bound curves
\begin{equation*}
    H(B):=\Delta_{\mathrm{SGD}}(B)-\Delta_{\mathrm{Sign}}(B)
    =\frac{(d-2\ell^2K)B}{N}+2\sqrt{\frac{d}{N}}
    -4\sqrt{\frac{d\ell}{B}}.
\end{equation*}
    Since $\ell^2K=o(d)$,
\begin{equation*}
    H'(B)=\frac{d-2\ell^2K}{N}+2\sqrt{d\ell}\,B^{-3/2}>0.
\end{equation*}
Thus $H$ is strictly increasing. Since $N\geq d$,
\begin{equation*}
    H(1)\leq 3-4\sqrt{d\ell}<0
\end{equation*}
for large $d$, while
\begin{equation*}
    H(N/2)=\frac{d-2\ell^2K}{2}+2\sqrt{\frac{d}{N}}-4\sqrt{\frac{2d\ell}{N}}>0
\end{equation*}
for large $d$. Hence there is a unique root $B_{\mathrm{cross}}\in[1,N/2]$, and
$\Delta_{\mathrm{Sign}}(B)<\Delta_{\mathrm{SGD}}(B)$ exactly when $H(B)>0$, equivalently when $B>B_{\mathrm{cross}}$.

It remains to identify the scale of this crossing. Define
\begin{equation*}
    \widehat B_{\mathrm{cross}}
    :=\left(\frac{16N^2\ell}{d}\right)^{1/3},
    \qquad
    A:=\frac{d\widehat B_{\mathrm{cross}}}{N}
    =4\sqrt{\frac{d\ell}{\widehat B_{\mathrm{cross}}}}.
\end{equation*}
For any fixed $r>0$,
\begin{align*}
    \frac{H(r\widehat B_{\mathrm{cross}})}{A}
    &=r-r^{-1/2}-\frac{2\ell^2K}{d}r
    +\frac{2\sqrt{d/N}}{d\widehat B_{\mathrm{cross}}/N}.
\end{align*}
    The last two terms are $o(1)$ because $\ell^2K/d\to0$ and
\begin{equation*}
    \frac{2\sqrt{d/N}}{d\widehat B_{\mathrm{cross}}/N}
    =\frac{2}{16^{1/3}d^{1/6}N^{1/6}\ell^{1/3}}
    =o(1).
\end{equation*}
Therefore $H(r\widehat B_{\mathrm{cross}})/A=r-r^{-1/2}+o(1)$. Since $r-r^{-1/2}$ has its unique zero at $r=1$, the unique root satisfies
\begin{equation*}
    B_{\mathrm{cross}}=(1+o(1))\widehat B_{\mathrm{cross}}
    =(1+o(1))\left(\frac{16N^2\log d}{d}\right)^{1/3}.
\end{equation*}
\end{proof}
\paragraph{Post-crossing regimes.}\label{app:sign_sgd_batch_regimes}\label{app:post_crossing_regimes}
Define
\(
B_{\mathrm{sat}}:=\frac{(4N^2d)^{1/3}}{\log d\,(\log N)^{2/3}}.
\)
This is the scale where the two terms in $\Delta_{\mathrm{Sign}}$ are equal, since
\(
    4\sqrt{\frac{d\log d}{B_{\mathrm{sat}}}}
    =\frac{2B_{\mathrm{sat}}(\log d)^2\log N}{N}.
\)
Also, \Cref{app:sign_sgd_batch_crossing} gives
\(
    B_{\mathrm{cross}}
    =\Theta\!\left(\left(\frac{N^2\log d}{d}\right)^{1/3}\right),
\)
so
\(
    \frac{B_{\mathrm{sat}}}{B_{\mathrm{cross}}}
    =\Theta\!\left(\left(\frac{d}{(\log d)^2\log N}\right)^{2/3}\right)\to\infty.
\)
Now suppose $B_{\mathrm{cross}}<B\ll B_{\mathrm{sat}}$. Since $B\ge B_{\mathrm{cross}}$,
\begin{align*}
    \frac{\sqrt{d/N}}{dB/N}
    \leq
    \frac{\sqrt{d/N}}{dB_{\mathrm{cross}}/N}
    =\Theta\!\left(\frac{(d/N)^{1/6}}{(\log d)^{1/3}}\right)
    =o(1), \\
    \frac{2B(\log d)^2\log N/N}{4\sqrt{d\log d/B}}
    =\left(\frac{B}{B_{\mathrm{sat}}}\right)^{3/2}
    =o(1).
\end{align*}
Hence $\Delta_{\mathrm{SGD}}(B)=\Theta(dB/N)$ and $\Delta_{\mathrm{Sign}}(B)=\Theta(\sqrt{d\log d/B})$, so
\begin{equation*}
    \frac{\Delta_{\mathrm{SGD}}(B)}{\Delta_{\mathrm{Sign}}(B)}
    =\Theta\!\left(\left(\frac{B}{B_{\mathrm{cross}}}\right)^{3/2}\right).
\end{equation*}
Finally, if $B_{\mathrm{sat}}\ll B\le N/2$, then $\Delta_{\mathrm{SGD}}(B)=\Theta(dB/N)$ and $\Delta_{\mathrm{Sign}}(B)=\Theta(B(\log d)^2\log N/N)$, hence
\begin{equation*}
    \frac{\Delta_{\mathrm{SGD}}(B)}{\Delta_{\mathrm{Sign}}(B)}
    =\Theta\!\left(\frac{d}{(\log d)^2\log N}\right).
\end{equation*}

%% file: sections/stochastic-momentum-softmax.tex
\section{Stochastic Momentum Softmax Unigram Analysis}\label{app:stochastic_momentum_nsd}


\begin{theorem}[Stochastic momentum SignGD rate under a fixed sample budget]\label{main:momentum_sign_sample}
    Consider optimizing $f_{\mathrm{uni}}$ using the single-sample objective $\widehat f_{\mathrm{uni},1}$ with large $d$, total sample budget $N$, and initializations $\theta_0=0$ and $v_0=0$. Suppose the momentum $\rho$ is chosen as in \Cref{main:NSD_mom_tuned}. Then stochastic momentum SignGD with weight decay $\lambda_\infty = 2/\log d$ and learning rate $\eta = \frac{\log d\log N}{2N}$ satisfies
\begin{align*}
    \E[f_{\mathrm{uni}}(\theta_N)-f_{\mathrm{uni}}^\star]
    &=
    O\!\left(d^{1/3}(\log d)^{4/3}\left(\frac{\log N}{N}\right)^{1/3}\right).
\end{align*}
\end{theorem}
\begin{proof}
    Let $u=(1/d)\mathbf 1$ be the uniform distribution over the vocabulary. At initialization $\theta_0=0$, we have $\softmax(\theta_0)=u$ and
\begin{equation*}
    \delta_0 = f_{\mathrm{uni}}(\theta_0)-f_{\mathrm{uni}}^\star = \KL(p\|u)\leq \log d.
\end{equation*}
In the single-sample regime $B=1$, the covariance computation in the proof of \Cref{main:unigram_noise_lemma} gives
\begin{equation*}
    \E\left[\norm{\nabla \widehat f_{\mathrm{uni},1}(\theta;\xi)-\nabla f_{\mathrm{uni}}(\theta)}_2^2\right]\leq 1,
\end{equation*}
so in \Cref{main:NSD_mom} we may take $\nu_2=1$ and set $T=N$.

Now consider stochastic momentum SignGD, i.e. NSD with respect to $\norm{\cdot}_\infty$. Here the dual norm is $\norm{\cdot}_1$, so
\begin{equation*}
    \kappa=\sup_{z\neq 0}\frac{\norm{z}_1}{\norm{z}_2}=\sqrt d.
\end{equation*}
By \Cref{main:unigram_smoothness_lemma,main:unigram_norm_lemma1}, $L=1$ and the choice $\lambda_\infty=2/\log d$ is optimal for the minimum $\ell_\infty$-norm solution. Thus $D=1/\lambda_\infty=\log d/2$ and $1+\lambda_\infty D=2$. Since $v_0=0$,
\begin{equation*}
    \E[\norm{\nabla f_{\mathrm{uni}}(\theta_0)-v_0}_1]=\norm{u-p}_1\leq 2.
\end{equation*}
The horizon condition becomes
\begin{equation*}
    T=\tilde{\Omega}\left(\frac{1+\lambda_\infty D}{\lambda_\infty\kappa\nu_2}\right)
    =\tilde{\Omega}\left(\frac{\log d}{\sqrt d}\right).
\end{equation*}
For sufficiently large $d$, this lower bound is at most a constant, so it is automatically satisfied in the regime considered here. Applying \Cref{main:NSD_mom_tuned} with $T=N$,
\begin{align*}
    \delta_N
    &\leq \left[\log d+\frac{(\log d)^2}{2}\right]\frac{\log N}{N}
    +\left[
    4d^{1/3}(\log d)^{1/3}
    +3d^{1/3}(\log d)^{4/3}
    \right]\left(\frac{\log N}{N}\right)^{1/3}.
\end{align*}
The stepsize is $\eta=\log N/(\lambda_\infty N)=\log d\log N/(2N)$, completing the proof.
\end{proof}

\begin{theorem}[Stochastic momentum normalized-GD rate under a fixed sample budget]\label{main:momentum_norm_sample}
    Consider optimizing $f_{\mathrm{uni}}$ using the single-sample objective $\widehat f_{\mathrm{uni},1}$ with large $d$, total sample budget $N=\tilde{\Omega}(\sqrt d)$, and initializations $\theta_0=0$ and $v_0=0$. Suppose the momentum $\rho$ is chosen as in \Cref{main:NSD_mom_tuned}. Then stochastic momentum normalized GD with weight decay $\lambda_2 = 1/\sqrt d$, learning rate $\eta = \frac{\sqrt d\log N}{N}$, and batch size $B=1$ satisfies
\begin{align*}
    \E[f_{\mathrm{uni}}(\theta_N)-f_{\mathrm{uni}}^\star]
    &=
    O\!\left(d^{2/3}\left(\frac{\log N}{N}\right)^{1/3}\right).
\end{align*}
\end{theorem}
\begin{proof}
    Let $u=(1/d)\mathbf 1$ be the uniform distribution over the vocabulary. At initialization $\theta_0=0$, we have $\softmax(\theta_0)=u$ and
\begin{equation*}
    \delta_0 = f_{\mathrm{uni}}(\theta_0)-f_{\mathrm{uni}}^\star = \KL(p\|u)\leq \log d.
\end{equation*}
In the single-sample regime $B=1$, the covariance computation in the proof of \Cref{main:unigram_noise_lemma} gives
\begin{equation*}
    \E\left[\norm{\nabla \widehat f_{\mathrm{uni},1}(\theta;\xi)-\nabla f_{\mathrm{uni}}(\theta)}_2^2\right]\leq 1,
\end{equation*}
so in \Cref{main:NSD_mom} we may take $\nu_2=1$ and set $T=N$.

Now consider stochastic momentum normalized GD, i.e. NSD with respect to $\norm{\cdot}_2$. By \Cref{main:unigram_smoothness_lemma}, $L=1/2$. For large $d$, \Cref{main:unigram_norm_lemma1,main:unigram_norm_lemma2} justify the choice $\lambda_2=1/\sqrt d$, and with $\theta_0=0$ we have $D=1/\lambda_2=\sqrt d$ and $1+\lambda_2D=2$. The dual norm is again $\norm{\cdot}_2$, so $\kappa=1$. Since $v_0=0$,
\begin{equation*}
    \E[\norm{\nabla f_{\mathrm{uni}}(\theta_0)-v_0}_2]=\norm{u-p}_2\leq 1.
\end{equation*}
The horizon condition in \Cref{main:NSD_mom_tuned} is
\begin{equation*}
    N=\tilde{\Omega}\left(\frac{1+\lambda_2D}{\lambda_2\kappa\nu_2}\right)
    =\tilde{\Omega}\left(\sqrt d\right),
\end{equation*}
which is exactly the stated assumption. Substituting the above constants into \Cref{main:NSD_mom_tuned} with $T=N$ gives
\begin{align*}
    \delta_N
    &\leq \left[\log d+d\right]\frac{\log N}{N}
    +\left[4d^{1/6}+5d^{2/3}\right]
    \left(\frac{\log N}{N}\right)^{1/3}.
\end{align*}
The stepsize in \Cref{main:NSD_mom_tuned} is $\eta=\log N/(\lambda_2N)=\sqrt d\log N/N$.
\end{proof}

%% file: sections/adam-extension.tex
\section{Proof for Challenges in Extending to Adam}\label{app:proof_unigram_adap}

\begin{proof}[Proof of \Cref{main:unigram_adap}]
    Let $\mathbf H=\diag(\alpha_1,\ldots,\alpha_d)$ be any diagonal positive semidefinite matrix such that
    \begin{equation*}
        \nabla^2 f_{\mathrm{uni}}(\theta)\preceq \mathbf H
        \qquad\text{for all }\theta\in\R^d.
    \end{equation*}
    We will show that $\tr(\mathbf H)\geq d/2$.

    Fix any pair $i<j$. For $s>0$, define $\theta^{(s,i,j)}\in\R^d$ by
    \begin{equation*}
        \theta^{(s,i,j)}_i=\theta^{(s,i,j)}_j=s,
        \qquad
        \theta^{(s,i,j)}_k=-s
        \quad\text{for }k\notin\{i,j\}.
    \end{equation*}
    Then
    \begin{equation*}
        [\softmax(\theta^{(s,i,j)})]_i
        =
        [\softmax(\theta^{(s,i,j)})]_j
        =
        \frac{1}{2+(d-2)e^{-2s}},
    \end{equation*}
    while for $k\notin\{i,j\}$,
    \begin{equation*}
        [\softmax(\theta^{(s,i,j)})]_k
        =
        \frac{e^{-2s}}{2+(d-2)e^{-2s}}.
    \end{equation*}
    Hence, as $s\to\infty$,
    \begin{equation*}
        \softmax(\theta^{(s,i,j)}) \to \frac{1}{2}e_i+\frac{1}{2}e_j.
    \end{equation*}

    Since the Hessian map
    \begin{equation*}
        q \mapsto \diag(q)-qq^\top
    \end{equation*}
    is continuous in $q$, we obtain
    \begin{equation*}
        \nabla^2 f_{\mathrm{uni}}(\theta^{(s,i,j)})
        \to
        H^{(i,j)},
    \end{equation*}
    where $H^{(i,j)}$ is the matrix whose only nonzero entries lie in the $(i,j)$ principal block:
    \begin{equation*}
        H^{(i,j)}_{\{i,j\},\{i,j\}}
        =
        \begin{pmatrix}
            \frac14 & -\frac14 \\
            -\frac14 & \frac14
        \end{pmatrix}.
    \end{equation*}
    Because $\mathbf H-\nabla^2 f_{\mathrm{uni}}(\theta^{(s,i,j)})\succeq 0$ for every $s$, and the cone of positive semidefinite matrices is closed, we also have
    \begin{equation*}
        \mathbf H-H^{(i,j)}\succeq 0.
    \end{equation*}

    Now test this PSD inequality with the vector $x=e_i-e_j$. Then
    \begin{equation*}
        x^\top \mathbf H x = \alpha_i+\alpha_j,
        \qquad
        x^\top H^{(i,j)} x = 1.
    \end{equation*}
    Therefore,
    \begin{equation*}
        \alpha_i+\alpha_j \geq 1
        \qquad\text{for every }1\leq i<j\leq d.
    \end{equation*}

    Summing this inequality over all pairs $(i,j)$ gives
    \begin{equation*}
        \sum_{1\leq i<j\leq d}(\alpha_i+\alpha_j)
        \geq
        \binom{d}{2}.
    \end{equation*}
    Each coordinate $\alpha_k$ appears in exactly $d-1$ such pairs, so
    \begin{equation*}
        (d-1)\tr(\mathbf H)
        =
        \sum_{1\leq i<j\leq d}(\alpha_i+\alpha_j)
        \geq
        \frac{d(d-1)}{2}.
    \end{equation*}
    Hence,
    \begin{equation*}
        \tr(\mathbf H)\geq \frac{d}{2}.
    \end{equation*}
    Since this holds for every diagonal PSD matrix $\mathbf H$ dominating the Hessian of $f_{\mathrm{uni}}$, we conclude that
    \begin{equation*}
        \cL_{\mathrm{diag}}(f_{\mathrm{uni}})\geq \frac{d}{2}.
    \end{equation*}
\end{proof}

%% file: sections/experimental-details.tex
\section{Experimental Details for Section~3}
\label{app:sec3-experimental-details}

This appendix provides the model, data, optimization, and reproducibility
details for the experiments in Section~\ref{sec:minimal_model} and
Figures~\ref{main:figure_softmax}--\ref{main:figure_simplification}.
Unless stated otherwise, SGD and SignSGD use neither momentum nor weight
decay, and all models are trained using the cross-entropy objective.

\paragraph{Data and preprocessing.}
The stochastic language-modeling experiments use the Penn Treebank (PTB)
corpus and the basic-English word tokenizer. The vocabulary is constructed
from the PTB training split and contains $d=9{,}922$ tokens, including the
unknown-token symbol. To reduce computational cost, the experiments denoted
by \texttt{TinyPTB} use the PTB validation split for training and the PTB
test split for evaluation, while retaining the vocabulary constructed from
the full training split.

A mini-batch contains $1{,}024$ parallel token sequences. The transformer,
frozen-transformer, and bigram experiments use a sequence length of $35$,
corresponding to at most $35{,}840$ next-token targets per optimization
step. The collapsed unigram experiment uses a sequence length of $32$,
corresponding to at most $32{,}768$ targets per step. The final batch in an
epoch can contain fewer targets.

\paragraph{Models and training horizons.}
Table~\ref{tab:sec3-model-details} summarizes the sequence of model
simplifications. ``Trainable parameters'' counts only parameters updated by
the optimizer.

\begin{table}[ht]
    \centering
    \small
    \setlength{\tabcolsep}{4pt}
    \begin{tabular}{@{}p{0.12\linewidth}p{0.45\linewidth}
                    p{0.18\linewidth}p{0.17\linewidth}@{}}
        \toprule
        Experiment & Model & Mini-batch & Horizon and seeds \\
        \midrule
        Figure~\ref{fig:transformer-exp-final}
        &
        One-layer causal transformer with embedding dimension $512$,
        $8$ attention heads, feed-forward width $256$, and dropout $0.2$.
        The embedding, transformer layer, and output classifier are all
        trained. The model has approximately $11.49$M trainable parameters.
        &
        $1{,}024\times 35$
        &
        $100$ epochs; seeds $0,1,2$ for the optimizer comparison and seed
        $0$ for the per-frequency-group runs.
        \\[3pt]

        Figure~\ref{main:figure_simplification}(a)
        &
        One-layer transformer with embedding dimension $256$, one attention
        head, feed-forward width $128$, and dropout $0.2$. The token embedding
        and complete transformer encoder are frozen; only the softmax
        classification head is trained. The model has approximately $5.42$M
        parameters, of which $2.55$M are trainable.
        &
        $1{,}024\times 35$
        &
        $100$ epochs; seeds $0,1,2$.
        \\[3pt]

        Figure~\ref{main:figure_simplification}(b)
        &
        Bigram model obtained by removing attention. It consists of a fixed
        $9{,}922\times256$ token-embedding matrix followed by a trainable,
        bias-free $256\times9{,}922$ output classifier. It has approximately
        $5.08$M parameters, of which $2.54$M are trainable.
        &
        $1{,}024\times 35$
        &
        $100$ epochs; seeds $0,1,2$.
        \\[3pt]

        Figure~\ref{main:figure_simplification}(c,d)
        &
        Collapsed softmax unigram model. Every input token is replaced by the
        same index, leaving one trainable vector of $9{,}922$ logits. The
        logits are initialized to zero.
        &
        $1{,}024\times 32$
        &
        $200$ epochs; seeds $0,1,2$.
        \\[3pt]

        Figure~\ref{main:figure_softmax}
        &
        Deterministic softmax unigram model with $d=10{,}000$ trainable
        logits, initialized to zero, and target probabilities
        $p_i\propto i^{-1}$.
        &
        Full batch
        &
        $600$ gradient steps; seed $0$.
        \\
        \bottomrule
    \end{tabular}
    \caption{Model and training details for the experiments in Section~3.}
    \label{tab:sec3-model-details}
\end{table}

\paragraph{Initialization.}
For Figure~\ref{fig:transformer-exp-final}, the token embeddings are
initialized independently from $\mathcal{N}(0,0.02^2)$, while the output
classifier weights and biases are initialized to zero. In
Figure~\ref{main:figure_simplification}(a), the transformer weights and
embeddings are initialized with standard deviation $0.02$ before the
embedding and encoder are frozen. In
Figure~\ref{main:figure_simplification}(b), the fixed embedding and initial
classifier weights are sampled from $\mathcal{N}(0,0.02^2)$. The collapsed
unigram logits in Figure~\ref{main:figure_simplification}(c,d) and the
deterministic logits in Figure~\ref{main:figure_softmax} are initialized to
zero.

\paragraph{Learning-rate tuning.}
Learning rates are tuned independently for each optimizer using logarithmic
grids. Let
\[
    \mathcal{G}(a,b,\delta)
    := \left\{10^a,10^{a+\delta},\ldots,10^b\right\}.
\]
For each optimizer, we select the learning rate giving the lowest training
loss at the tuning horizon. No learning-rate schedule is used. The grids and
selected values are reported in
Table~\ref{tab:sec3-learning-rates}. Numerical approximations are included
for convenience.

\begin{table}[ht]
    \centering
    \small
    \setlength{\tabcolsep}{5pt}
    \begin{tabular}{@{}p{0.23\linewidth}lp{0.31\linewidth}
                    p{0.23\linewidth}@{}}
        \toprule
        Experiment & Optimizer & Search grid & Selected learning rate \\
        \midrule
        Figure~\ref{fig:transformer-exp-final}
        & SGD
        & $\mathcal{G}(-3,0,1/4)$
        & $10^{-1/2}\approx0.316$ \\
        & Adam
        & $\mathcal{G}(-4,-2,1/4)$
        & $10^{-11/4}\approx0.00178$ \\
        & normalized SGD
        & $\mathcal{G}(-2,1,1/4)$
        & $10^{1/4}\approx1.78$ \\
        & SignSGD
        & $\mathcal{G}(-5,-2,1/4)$
        & $10^{-5/2}\approx0.00316$ \\[2pt]

        Figure~\ref{main:figure_simplification}(a)
        & SGD
        & $\mathcal{G}(-3,0,1/2)$
        & $10^{-1/2}\approx0.316$ \\
        & SignSGD
        & $\mathcal{G}(-5,-2,1/2)$
        & $10^{-5/2}\approx0.00316$ \\[2pt]

        Figure~\ref{main:figure_simplification}(b)
        & SGD
        & $\mathcal{G}(0,3,1/2)$
        & $10^{3/2}\approx31.6$ \\
        & SignSGD
        & $\mathcal{G}(-3,-1,1/2)$
        & $10^{-2}=0.01$ \\[2pt]

        Figure~\ref{main:figure_simplification}(c,d)
        & SGD
        & $\mathcal{G}(-1,2,1/2)$
        & $10^{3/2}\approx31.6$ \\
        & SignSGD
        & $\mathcal{G}(-4,-1,1/2)$
        & $10^{-2}=0.01$ \\[2pt]

        Figure~\ref{main:figure_softmax}
        & GD
        & $\mathcal{G}(0,3,1/4)$
        & $10^{5/4}\approx17.8$ \\
        & Adam
        & $\mathcal{G}(-2,-1,1/4)$
        & $10^{-3/2}\approx0.0316$ \\
        & normalized GD
        & $\mathcal{G}(-1,1,1/4)$
        & $10^{-1/4}\approx0.562$ \\
        & SignGD
        & $\mathcal{G}(-3,-1,1/4)$
        & $10^{-2}=0.01$ \\
        \bottomrule
    \end{tabular}
    \caption{Per-optimizer learning-rate tuning for the Section~3
    experiments. For Adam, we use $\beta_1=0$, $\beta_2=0.999$, and
    $\epsilon=10^{-8}$.}
    \label{tab:sec3-learning-rates}
\end{table}

For the per-frequency-group panels of
Figure~\ref{fig:transformer-exp-final}, SGD uses
$10^{-3/4}\approx0.178$ and SignSGD uses
$10^{-11/4}\approx0.00178$. The unigram learning-rate grid was run for $500$ epochs; the selected
configurations were then re-run for the $200$ epochs displayed in
Figure~\ref{main:figure_simplification}(c,d). All other stochastic grids use
the same horizon as their corresponding plotted experiment.

\paragraph{Seeds and curve selection.}
All curves in Figure 2 use seed 0 and a training horizon of 100 epochs. The overall SGD and SignSGD curves in Figure 2(a) are taken from the same runs as the corresponding frequency-group curves in Figure 2(b,c), with identical learning rates. We display seed 0 at the learning rate selected by the lowest median final training cross-entropy across seeds 0, 1, and 2.
Figure~\ref{main:figure_simplification}, we run seeds $0$, $1$, and $2$ for
each optimizer. The plotting procedure selects, separately for each
optimizer, the seed with the lowest final training cross-entropy. The
displayed seeds are: seed $1$ for SGD and seed $0$ for SignSGD in
Figure~\ref{main:figure_simplification}(a), and seed $0$ for both optimizers
in Figure~\ref{main:figure_simplification}(b--d). The deterministic
softmax-unigram experiment uses seed $0$, although the zero initialization
and exact gradients make the optimization trajectory deterministic.

\paragraph{Per-frequency-group losses.}
For Figure~\ref{fig:transformer-exp-final}(b,c), token classes are ordered by
their empirical PTB frequency and partitioned into ten groups, each
containing approximately one tenth of the empirical target mass. We compute the mean conditional cross-entropy for each token class and take an unweighted arithmetic mean across classes within each group. The per-class statistics are obtained from SGD and SignSGD runs with
the same architecture and horizon as
Figure~\ref{fig:transformer-exp-final}(a), using the per-frequency-group
learning rates reported above and seed $0$.

For Figure~\ref{main:figure_softmax}(c,d), the $10{,}000$ classes are
similarly partitioned into ten groups of equal target-probability mass. If
$q=\softmax(\theta)$, the contribution assigned to token $i$ is
\[
    p_i\log\frac{p_i}{q_i}-p_i+q_i.
\]
These terms are non-negative and sum exactly to
$\KL(p\|q)$. Each curve reports the sum of these contributions over one
frequency group.

\subsection{Batch-Size Experiments}
\label{app:batch-size-experiments}

This subsection provides the model, data, optimization, and curve-selection details for Figures~5 and~6.

\paragraph{Shared optimizer settings.}
We compare SGD, Adam, normalized SGD, and SignSGD. SGD, normalized SGD, and SignSGD use zero momentum. No optimizer uses weight decay or a learning-rate schedule. Adam uses $\beta_1=0$, $\beta_2=0.999$, and $\epsilon=10^{-8}$. At each iteration, normalized SGD divides the complete gradient vector by its global $\ell_2$ norm, whereas SignSGD applies the sign operation coordinate-wise. All experiments use 32-bit floating-point arithmetic and PyTorch 2.5.1.

\paragraph{Stochastic softmax unigram model (Figure~5).}
We use a softmax unigram model with $d=10{,}000$ classes the Zipfian distribution in \Cref{main:zipf}. 

At every optimization step, a mini-batch
$\xi_1,\ldots,\xi_B\overset{\mathrm{i.i.d.}}{\sim}\operatorname{Cat}(p)$
is sampled with replacement. The stochastic training loss is
\[
    -\frac{1}{B}\sum_{b=1}^{B}\log [q_\theta]_{\xi_b}.
\]
The plotted metric is the exact population loss $\operatorname{KL}(p\Vert q_\theta)$, evaluated against the complete target distribution after every update.

We consider batch sizes $B\in\{64,1024,32768\}$. Every configuration is run for 300 optimization steps with seeds $0$, $1$, and $2$. Thus, the methods are compared at a fixed iteration horizon, and a run with batch size $B$ processes $300B$ sampled tokens.

For compactness, define the logarithmic learning-rate grid $\mathcal{G}(a,b,\delta)= \{10^a,10^{a+\delta},\ldots,10^b\}$.
The learning-rate grids are $\mathcal{G}(-1,3,1/2)$ for SGD,
$\mathcal{G}(-3,0,1/2)$ for Adam,
$\mathcal{G}(-2,2,1/2)$ for normalized SGD, and
$\mathcal{G}(-3,0,1/2)$ for SignSGD.

For each optimizer and batch size, the displayed curve is the individual learning-rate and seed combination with the lowest population KL divergence after 300 steps. The selected learning rates are reported in Table~\ref{tab:softmax-unigram-batch-lrs}, with the selected seed shown in parentheses.

\begin{table}[t]
    \centering
    \small
    \begin{tabular}{c|cccc}
        \hline
        Batch size
        & SGD
        & Adam
        & Normalized SGD
        & SignSGD \\
        \hline
        $64$
        & $10$ $(2)$
        & $10^{-3/2}$ $(2)$
        & $10^{1/2}$ $(2)$
        & $10^{-3/2}$ $(2)$ \\
        $1024$
        & $10^{3/2}$ $(0)$
        & $10^{-3/2}$ $(0)$
        & $10$ $(2)$
        & $10^{-2}$ $(0)$ \\
        $32768$
        & $10^{3/2}$ $(0)$
        & $10^{-3/2}$ $(0)$
        & $1$ $(0)$
        & $10^{-3/2}$ $(1)$ \\
        \hline
    \end{tabular}
    \caption{Selected learning rates for the stochastic softmax unigram experiment. Selected seeds are shown in parentheses.}
    \label{tab:softmax-unigram-batch-lrs}
\end{table}

\paragraph{Transformer language model (Figure~6).}
We use the TinyPTB setup described above. The vocabulary is constructed from the Penn Treebank training split using the basic-English word tokenizer and contains $9{,}922$ tokens, including the unknown-token symbol. The PTB validation split is used for training and the PTB test split is used for evaluation.

The model is a two-layer causal transformer with trainable token embeddings, sinusoidal positional encodings, eight attention heads per layer, embedding dimension 512, feed-forward width 256, ReLU activations, post-layer normalization, an untied output classifier, and zero dropout. The model has approximately $12.80$ million trainable parameters.

Embedding, feed-forward, and classifier weights are initialized independently from $\mathcal{N}(0,0.02^2)$, and the corresponding biases are initialized to zero. Layer-normalization weights are initialized from $\mathcal{N}(1,0.02^2)$, with zero biases. The attention projection weights use the default PyTorch initialization.

We use $B\in\{32,256,1024\}$ parallel token sequences with sequence length 35. Thus, an optimization step contains at most $35B$ next-token targets. The final sequence chunk of an epoch can be shorter. Every configuration is trained for 100 epochs with seeds $0$, $1$, and $2$ using the mean next-token cross-entropy objective. Because every run makes the same number of passes through the training corpus, larger sequence batches result in fewer optimizer updates per epoch.

For $B=32$, the learning-rate grids are
$\mathcal{G}(-2,2,1/2)$ for SGD and normalized SGD,
$\mathcal{G}(-4,-1,1/2)$ for SignSGD, and
$\mathcal{G}(-5,-1,1/2)$ for Adam.
For $B=256$, the grids are
$\mathcal{G}(-2,1,1/2)$ for SGD and normalized SGD and
$\mathcal{G}(-4,-1,1/2)$ for Adam and SignSGD.
For $B=1024$, the grids are
$\mathcal{G}(-3,-1,1/2)$ for SGD and normalized SGD and
$\mathcal{G}(-4,-1,1/2)$ for Adam and SignSGD.

For each optimizer and batch size, the learning rate is selected by the lowest median training cross-entropy after 100 epochs across seeds $0$, $1$, and $2$. The plotted curve is the pointwise median across the three seeds, and the shaded region spans their pointwise minimum and maximum. The selected learning rates are reported in Table~\ref{tab:transformer-batch-lrs}.

\begin{table}[t]
    \centering
    \small
    \begin{tabular}{c|cccc}
        \hline
        Batch size
        & SGD
        & Adam
        & Normalized SGD
        & SignSGD \\
        \hline
        $32$
        & $10^{-1}$
        & $10^{-4}$
        & $10^{-1/2}$
        & $10^{-7/2}$ \\
        $256$
        & $10^{-1}$
        & $10^{-7/2}$
        & $1$
        & $10^{-3}$ \\
        $1024$
        & $10^{-1}$
        & $10^{-3}$
        & $10^{-1}$
        & $10^{-3}$ \\
        \hline
    \end{tabular}
    \caption{Selected learning rates for the TinyPTB transformer experiment.}
    \label{tab:transformer-batch-lrs}
\end{table}